\documentclass{article}

\usepackage{microtype}
\usepackage{graphicx}
\usepackage{booktabs} 

\usepackage[accepted]{icml2021}

\usepackage[hidelinks]{hyperref}
\usepackage{bookmark}
\bookmarksetup{numbered}

\hypersetup{
    colorlinks=false,
}
\usepackage{graphicx}
\usepackage{amsmath,bm,bbm}
\usepackage{amsfonts}
\usepackage{xcolor}         
\usepackage{amsthm}        
\usepackage[algoruled,lined,linesnumbered]{algorithm2e} 
\usepackage[capitalize,nameinlink,noabbrev]{cleveref}
\usepackage[normalem]{ulem} 
\usepackage{tikz}
\usetikzlibrary{chains,arrows}
\usepackage{booktabs}
\usepackage[skip=0pt,belowskip=0pt]{caption}
\usepackage{subcaption}
\usepackage{wrapfig}
\usepackage{chngcntr}
\usepackage{dblfloatfix}    
\usepackage{ulem}

\newcommand{\mbf}[1]{\mathbf{#1}}

\DeclareMathOperator*{\argmax}{\arg\!\max}
\newtheorem{theorem}{Theorem}

\newtheorem{lemma}{Lemma}

\newtheorem{assumption}{Assumption}

\newcommand{\bbE}{\mathbb{E}}
\newcommand{\bbR}{\mathbb{R}}

\newcommand{\bbI}{\mathbb{I}}
\newcommand{\calS}{\mathcal{S}}
\newcommand{\calA}{\mathcal{A}}
\newcommand{\calR}{\mathcal{R}}

\newcommand{\calF}{\mathcal{F}}
\newcommand{\calM}{\mathcal{M}}

\newcommand{\norm}[1]{\left\lVert#1\right\rVert}
\newcommand{\abs}[1]{\left\lvert#1\right\rvert}

\usepackage{xargs}      
\usepackage[colorinlistoftodos,prependcaption,textsize=tiny]{todonotes}
\newcommandx{\unsure}[2][1=]{\todo[linecolor=red,backgroundcolor=red!25,bordercolor=red,#1]{#2}}

\usepackage{xspace}

\icmltitlerunning{Learning and Planning in Average-Reward Markov Decision Processes}

\begin{document}

\twocolumn[
\icmltitle{Learning and Planning in Average-Reward Markov Decision Processes}



\icmlsetsymbol{equal}{*}

\begin{icmlauthorlist}
\icmlauthor{Yi Wan}{equal,uofa}
\icmlauthor{Abhishek Naik}{equal,uofa}
\icmlauthor{Richard S.~Sutton}{uofa,dm}
\end{icmlauthorlist}

\icmlaffiliation{uofa}{University of Alberta and Alberta Machine Intelligence Institute (Amii), Edmonton, Canada.}
\icmlaffiliation{dm}{DeepMind}

\icmlcorrespondingauthor{Yi Wan}{wan6@ualberta.ca}
\icmlcorrespondingauthor{Abhishek Naik}{abhishek.naik@ualberta.ca}

\icmlkeywords{ICML, Artificial Intelligence, Machine Learning, Reinforcement Learning, Average Reward}

\vskip 0.3in
]



\printAffiliationsAndNotice{\icmlEqualContribution} 

\begin{abstract}
We introduce learning and planning algorithms for average-reward MDPs, including 1) the first general proven-convergent off-policy model-free control algorithm without reference states, 2) the first proven-convergent off-policy model-free prediction algorithm, and 3) the first off-policy learning algorithm that converges to the actual value function rather than to the value function plus an offset.
All of our algorithms are based on using the temporal-difference error rather than the conventional error when updating the estimate of the average reward. Our proof techniques are a slight generalization of those by Abounadi, Bertsekas, and Borkar (2001).
In experiments with an Access-Control Queuing Task, we show some of the difficulties that can arise when using methods that rely on reference states and argue that our new algorithms can be significantly easier to use. 
\end{abstract}



\section{Average-Reward Learning and Planning}

The average-reward formulation of Markov decision processes (MDPs) is arguably the most important for reinforcement learning and artificial intelligence (see, e.g., Sutton \& Barto 2018 Chapter 10, Naik et al.\ 2019) yet has received much less attention than the episodic and discounted formulations. In the average-reward setting, experience is continuing (not broken up into episodes) and the agent seeks to maximize the average reward per step, or \emph{reward rate}, with equal weight given to immediate and delayed rewards. 
In addition to this \emph{control} problem, there is also the \emph{prediction} problem of estimating the value function and the reward rate for a given \emph{target policy}. 
Solution methods for these problems can be divided into those that are driven by experiential data, called \emph{learning} algorithms, those that are driven by a model of the MDP, called \emph{planning} algorithms, and combined methods that first learn a model and then plan with it. 
For learning and combined methods, both control and prediction problems can be further subdivided into \emph{on-policy} versions, in which data is gathered using the target policy, and \emph{off-policy} versions, in which data is gathered using a second policy, called the \emph{behavior policy}. 
In general, both policies may be non-stationary. For example, in the control problem, the target policy should converge to a policy that maximizes the reward rate. 
Useful surveys of average-reward learning are given by Mahadevan (1996) and Dewanto et al.\ (2020).

On-policy problems are generally easier than off-policy problems and permit more capable algorithms with convergence guarantees. For example, on-policy \emph{prediction} algorithms with function approximation and convergence guarantees include average-cost TD$(\lambda)$ (Tsitsiklis \& Van Roy 1999), LSTD$(\lambda)$ (Konda 2002), and LSPE$(\lambda)$ (Yu \& Bertsekas 2009). On-policy \emph{control} algorithms that have been proved to converge asymptotically or to achieve sub-linear regret or to be probably approximately correct under various conditions include tabular learning algorithms (e.g., Wheeler \& Narendra 1986, Abbasi-Yadkori et al.\ 2019a,b
), tabular combined algorithms (e.g., Kearns \& Singh 2002, Brafman \& Tennenholtz 2002, Auer \& Ortner 2006, 
Jaksch et al.\ 2010
), and policy gradient algorithms (e.g., Sutton et al.\ 1999, Marbach \& Tsitsiklis 2001, Kakade 2001, Konda 2002
). 

The \textit{off-policy learning control} problem is particularly challenging, and theoretical results are available only for the tabular, discrete-state setting without function approximation. 
The most important prior algorithm is RVI Q-learning, introduced by Abounadi, Bertsekas, and Borkar (1998, 2001). 
The same paper also introduced \emph{SSP Q-learning}, but SSP Q-learning was limited to MDPs with a special state that is recurrent under all stationary policies, whereas RVI Q-learning is convergent for more general 
MDPs.
Ren and Krogh (2001) presented a tabular algorithm and proved its convergence, but their algorithm required knowledge of properties of the MDP which are not in general known. 
Gosavi (2004) also introduced an algorithm and proved its convergence, but it was limited in the same way as SSP Q-learning. 
Yang et al.~(2016) presented an algorithm and claimed to prove its convergence, but their proof is not correct (as we detail in \Cref{app:additional}). 
The earliest tabular average-reward off-policy learning control algorithms that we know of were those introduced (without convergence proofs) by Schwartz (1993) and Singh (1994). 
Bertsekas and Tsitsiklis (1996) and Das et al.~(1999) introduced off-policy learning control algorithms with function approximation, but did not provide convergence proofs.

Abounadi et al.'s RVI Q-learning is actually a family of off-policy algorithms, 
a particular member of which is determined by specifying a function that references the estimated values of specific state--action pairs and produces an estimate of the reward rate.
We call this function the \textit{reference function}.
Examples include a weighted average of the value estimates of all state--action pairs, or in the simplest case, the estimate of a single state--action pair's value. 
For best results, the referenced state--action pairs should be frequently visited; otherwise convergence can be unduly slow (as we illustrate in Section \ref{sec:control-exps}). 
However, if the behavior policy is linked to the target policy (as in $\epsilon$-greedy behavior policies), then knowing which state--action pairs will be frequently visited may be to know a substantial part of the problem's solution.
For example, in learning an optimal path through a maze from diverse starting points, the frequently visited state--action pairs are likely to be those on the shortest paths to the goal state. To know these would be tantamount to knowing a priori the best paths to the goal.
This observation motivates the search for a general learning algorithm that does not require a reference function.

Our first contribution is to introduce such a learning control algorithm without a reference function.
Our \textit{Differential Q-learning} algorithm is convergent for general MDPs, which we prove by slightly generalizing the theory of RVI Q-learning (Abounadi et al. 2001). Unlike RVI Q-learning, Differential Q-learning does not involve reference states. Instead, it maintains an explicit estimate of the reward rate 
(as in Schwartz 1993, Singh 1994).

Our second contribution is \emph{Differential TD-learning}, the first off-policy model-free prediction learning algorithm proved convergent to the reward rate and differential value function of the target policy. 
There are a number of algorithms that estimate the reward rate (e.g., Wen et al.~2020, Liu et al.~2018, Tang et al.~2019, Mousavi et al.~2020, Zhang et al.~ 2020a,b), but none that estimate the value function. 
These algorithms also differ from Differential TD-learning in that are not online algorithms; they operate on a fixed batch of data.
Finally, they differ in that they estimate the ratio of the steady-state occupancy distributions under the target and behavior policies, whereas Differential TD-learning does not.

\textit{Planning} algorithms for average-reward MDPs have been known at least since the setting was introduced by Howard in 1960. However, most of these, including value iteration (Bellman 1957), policy iteration (Howard 1960), and relative value iteration (RVI, White 1963), are ill-suited for use in reinforcement learning because they involve sub-steps whose complexity is order the number of states or more. Jalali and Ferguson (1989, 1990) were among the first to explore more incremental methods, though their algorithms are limited to special-case MDPs and require a reference state--action pair. In planning, as in learning, the state of the art appears to be RVI Q-learning, now applied as a planning algorithm to a stream of experience generated by the model. When our Differential Q-learning algorithm is applied in the same way, we call it \emph{Differential Q-planning}; it improves over the RVI Q-learning's planner in that it omits reference states, with concomitant efficiencies just as in the learning case. In the prediction case we have \textit{Differential TD-planning}. Both of these algorithms are fully incremental and well suited for use in reinforcement learning architectures (e.g., Dyna (Sutton 1990)).

All the aforementioned average-reward algorithms converge not to the actual value function, but to the value function plus an offset that depends on initial conditions or on a reference state or state--action pair. The offset is not necessarily a problem because only the relative values of states (or of state--action pairs) are used to determine policies. However, the actual value function of any policy is \emph{centered}, meaning that the mean value of states encountered under the policy is zero.
Although it is easy to center an estimated value function in the on-policy case, in the off-policy case it is not.
Our final contribution is to extend our off-policy algorithms to centered versions that converge to the actual value function without an offset. 


\section{Learning and Planning for Control}
\label{sec:control_background_algorithms}

We formalize an agent's interaction with its environment by a finite Markov decision process, defined by the tuple $\calM \doteq (\calS, \calA, \calR, p)$, where $\calS$ is a set of states, $\calA$ is a set of actions, $\calR$ is a set of rewards, and $p : \calS \times \calR \times \calS \times \calA \to [0, 1]$ is the dynamics of the environment. 
At each of a sequence of discrete time steps $t = 0, 1, 2, \ldots$ , the agent receives an indication of a state of the MDP $S_t\in\calS$ and selects, using behavior policy $b:\calA\times\calS\rightarrow[0, 1]$, an action $A_t\in\calA$, then receives from the environment a reward $R_{t+1}\in\calR$ and the next state $S_{t+1}\in\calS$, and so on. The transition dynamics are such that $p(s', r \mid s, a) \doteq \Pr(S_{t+1} = s', R_{t+1} = r \mid S_t = s, A_t = a)$ for all $s, s' \in \calS, a \in \calA$, and $r \in \calR$. All policies we consider in the paper are in the set of stationary Markov policies $\Pi$.

Technically, for an unconstrained MDP, the best reward rate depends on the start state. For example, the MDP may have two disjoint sets of states with no policy that passes from one to the other; in this case there are effectively two MDPs, with unrelated rates of reward. A learning algorithm would have no difficulty with such cases---it would optimize for whichever sub-MDP it found itself in---but it is complex to state formally what is meant by an optimal policy. 
To remove this complexity, it is commonplace to rule out such cases by assuming that the MDP is communicating, which just means that there are no states from which it is impossible to get back to the others.

\textbf{Communicating Assumption}: For every pair of states, there exists a policy that transitions from one to the other in a finite number of steps with non-zero probability.

Under the communicating assumption, there exists a unique optimal reward rate $r^*$ that does not depend on the start state. 
To define $r^*$, we will need the reward rate for an arbitrary policy $\pi$ and a given start state $s$:
\begin{align}
\label{eq:avg_rew_definition}
    r(\pi,s) \doteq \lim_{n \to \infty} \frac{1}{n} \sum_{t = 1}^n \mathbb{E}[R_t \mid S_0=s, A_{0:t-1} \sim \pi].
\end{align}
It turns out that the best reward rate from $s$, $\sup_\pi r(\pi, s)$, does not depend on $s$ (see, e.g., Puterman 1994), and we define it as $r^*$. We seek a learning algorithm which achieves $r^*$.

Our \emph{Differential Q-learning} algorithm updates a table of estimates $Q_t: \calS \times \calA \to \bbR$ as follows:
\begin{align}
    Q_{t+1}(S_t, A_t) &\doteq Q_t (S_t, A_t) + \alpha_t \delta_t,  \label{Q-learning: Q update} \\
    Q_{t+1}(s, a) &\doteq Q_t (s, a),\ \forall s,a \neq S_t, A_t, \nonumber
\end{align}
where $\alpha_t$ is a step-size sequence, and $\delta_t$, the temporal-difference (TD) error, is: 
\begin{align}
    \delta_t \doteq R_{t+1} - \bar R_t + \max_{a} Q_t(S_{t+1}, a) - Q_t (S_t, A_t),
\end{align}
where $\bar R_t$ is a scalar estimate of $r^*$, updated by:
\begin{align}
    \bar R_{t+1} \doteq \bar R_t + \eta \alpha_t \delta_t \label{Q-learning: bar R update},
\end{align}
and $\eta$ is a positive constant. 

The following theorem shows that $\bar R_t$ converges to $r^*$ and $Q_t$ converges to a solution of $q$ in the Bellman equation:
\begin{align}
    q(s, a) = \sum_{s', r} p(s', r \mid s, a) (r - \bar r + \max_{a'} q(s', a')), \label{action-value Bellman equation}
\end{align}
for all $s \in \calS$ and $a \in \calA$. The unique solution for $\bar r$ is $r^*$. 
To guarantee that $Q_t$ converges to a unique point, we need to assume that the solution of $q$ is unique up to a constant.

\begin{theorem}[Informal] \label{Differential Q-learning}
If 1) the MDP is communicating, 2) the solution of $q$ in \eqref{action-value Bellman equation} is unique up to a constant, 3) the step sizes, specific to each state--action pair, are decreased appropriately, 4) all the state--action pairs are updated an infinite number of times, and 5) the ratio of the update frequency of the most-updated state--action pair to the update frequency of the least-updated state--action pair is finite, then the Differential Q-learning algorithm \eqref{Q-learning: Q update}--\eqref{Q-learning: bar R update} converges, almost surely, $\bar R_t$ to $r^*$, $Q_t$ to a solution of $q$ in \eqref{action-value Bellman equation}, 
and $r(\pi_t, s)$ to $r^*$, for all $s \in \calS$, where $\pi_t$ is  any greedy policy w.r.t. $Q_t$.
\end{theorem} 

\begin{proof} (Sketch; complete proof in Appendix \ref{app:convergence-proofs}) 
\def\X{\Sigma}
Our proof comprises two steps. First, we combine our algorithm's two updates to obtain a single update that is similar to the RVI Q-learning's update. Second,
we extend the family of RVI-learning algorithms so that the aforementioned single update is a member of the extended family and show convergence for the extended family.

Define $\X_t \doteq \sum_{s,a} Q_{t}(s, a)$. At each time step, the increment to $\bar R_t$ is $\eta$ times the increment to $Q_t$ and hence to $\X_t$. Therefore, the cumulative increment can be written as:
\vspace{-2mm}
\begin{align}
    \bar R_t - \bar R_0 &= \sum_{i = 0}^{t-1} \eta \alpha_i \delta_i = \eta\left(\X_t - \X_0 \right) \nonumber \\
    \implies \bar R_t &= \eta \X_t - c \text{,\ \ where\ \  } c \doteq \eta \X_0 - \bar R_0. \label{eq:proof-rviq-barr}
\end{align}
Next, substitute $\bar R_t$ in \eqref{Q-learning: Q update} with \eqref{eq:proof-rviq-barr}:
\begin{align}
    & Q_{t+1}(S_t, A_t) = Q_{t}(S_t, A_t)\ + \nonumber \\ 
    & \alpha_t \big(R_{t+1} + \max_{a} Q_t(S_{t+1}, a) - Q_t(S_t, A_t) - \eta \X_t + c \big) \nonumber \\
    &\hspace{1.9cm} = Q_{t}(S_t, A_t)\ + \nonumber \\ 
    & \alpha_t \big(\tilde{R}_{t+1} + \max_{a} Q_t(S_{t+1}, a) - Q_t(S_t, A_t) - \eta \X_t\big), \label{eq:proof-diffq-update}
\end{align}
where $\tilde{R}_{t+1} \doteq R_{t+1} + c$. Now \eqref{eq:proof-diffq-update} is in the same form as RVI Q-learning's update:
\begin{align}
\label{RVI Q-learning}
    &Q_{t+1}(S_t, A_t) = Q_{t} (S_t, A_t)\ +  \nonumber\\ 
    &\alpha_t \big(R_{t+1} + \max_a Q_t(S_{t+1}, a) - Q_t(S_t, A_t) - f(Q_t)\big), 
\end{align}
with $f(Q_t) = \eta \X_t$ for a slightly different MDP $\tilde \calM$ whose rewards are all shifted by $c$. 

Note that the convergence of $Q_t$ in \eqref{eq:proof-diffq-update} cannot be obtained using the convergence theorem of RVI Q-learning because $\eta \X_t = \eta \sum_{s,a} Q_t(s,a)$ in general does not satisfy conditions on $f$ allowed by Assumption 2.2 of Abounadi et al.~(2001).
However, by extending the family of RVI Q-learning algorithms to cover the case of $f(Q_t) = \eta \sum_{s, a} Q_t(s, a)\ \forall \eta\in\mathbb{R}$,
we show that the convergence of $Q_t$ in \eqref{eq:proof-diffq-update} holds. In particular, we show that $Q_t$ converges almost surely to a solution, denoted as $q_\infty$, which is the unique solution for $q$ in \eqref{action-value Bellman equation} under MDP $\tilde{\calM}$ and $\eta \sum_{s, a} q(s, a) = r_* + c$. It can be shown that $q_\infty$ is also a solution for $q$ in \eqref{action-value Bellman equation} in $\calM$. Additionally, because $\eta \X_t = \eta \sum_{s, a} Q_t(s, a)$ converges to $\eta \sum_{s, a} q_\infty(s, a) = r_* + c$, we have $\bar R_t = \eta \X_t - c$ 
converges to $r_*$ almost surely. The almost-sure convergence of $r(\pi_{t}, s)$ to $r_*,\ \forall s$ then follows from a variant of Theorem 8.5.5 by Puterman (1994), the continuous mapping theorem, and the convergence of $Q_t$.
\end{proof}

\textbf{Remark}: Interestingly, RVI Q-learning and Differential Q-learning make the same updates to $Q_t$ in special cases. For RVI Q-learning, the special case is when the reference function is the mean of all state--action pairs' values. For Differential Q-learning, the special case is when $\eta = \frac{1}{|\calS||\calA|}$. These special cases are not particularly good for either algorithm, and therefore their special-case equivalence tells us little about the relationship between the algorithms in practice. In RVI Q-learning, it is generally better for the reference function to emphasize state--action pairs that are frequently visited rather than to weight all state--action pairs equally (an example of this is shown and discussed in Section~\ref{sec:control-exps}). In Differential Q-learning, the special-case setting of $\eta = \frac{1}{|\calS||\calA|}$ would often be much too small on problems with large state and action spaces. 

If Differential Q-learning is applied to simulated experience generated from a model of the environment, then it becomes a planning algorithm, which we call \emph{Differential Q-planning}. Formally, the model is a function $\hat p: \calS \times \calR \times \calS \times \calA \to [0, 1]$, analogous to $p$, that, like $p$, sums to 1: $\sum_{s', r} \hat p(s', r \mid s, a) = 1$ for all $s, a$. A model MDP can be thus constructed using $\hat p$ and $\calS, \calA, \calR$. If the model MDP is communicating, then there is a unique optimal reward rate $\hat{r}_*$. 
The simulated transitions are generated as follows: at each planning step $n$, the agent arbitrarily chooses a state $S_n$ and an action $A_n$, and applies $\hat p$  to generate a simulated resulting state and reward $S_n', R_n \sim \hat p(\cdot, \cdot \mid S_n, A_n)$. 

Like Differential Q-learning, Differential Q-planning maintains a table of action-value estimates $Q_n: \calS \times \calA \to \bbR$ 
and a reward-rate estimate $\bar R_n$. At each planning step $n$, these estimates are updated by \eqref{Q-learning: Q update}–\eqref{Q-learning: bar R update}, just as in Differential Q-learning, except now using $S_n, A_n, R_n, S_n'$ instead of $S_t, A_t, R_{t+1}, S_{t+1}$. 
\smallskip
\begin{theorem}[Informal]\label{Differential Q-planning} Under the same assumptions made in Theorem \ref{Differential Q-learning} (except now for the model MDP corresponding to $\hat p$ rather than $p$) the Differential Q-planning algorithm converges, almost surely, $\bar R_n$ to $\hat r_*$ and $Q_n$ to a solution of $q$ in the Bellman equation (cf. \eqref{action-value Bellman equation}) for the model MDP. 
\end{theorem}
\vspace{-2mm}
\begin{proof}
Essentially as in Theorem 1. Full proof in Appendix \ref{app:convergence-proofs}.
\end{proof}

\begin{figure}[b!]
\centering
    \begin{subfigure}{.5\textwidth}
    \includegraphics[width=0.95\textwidth]{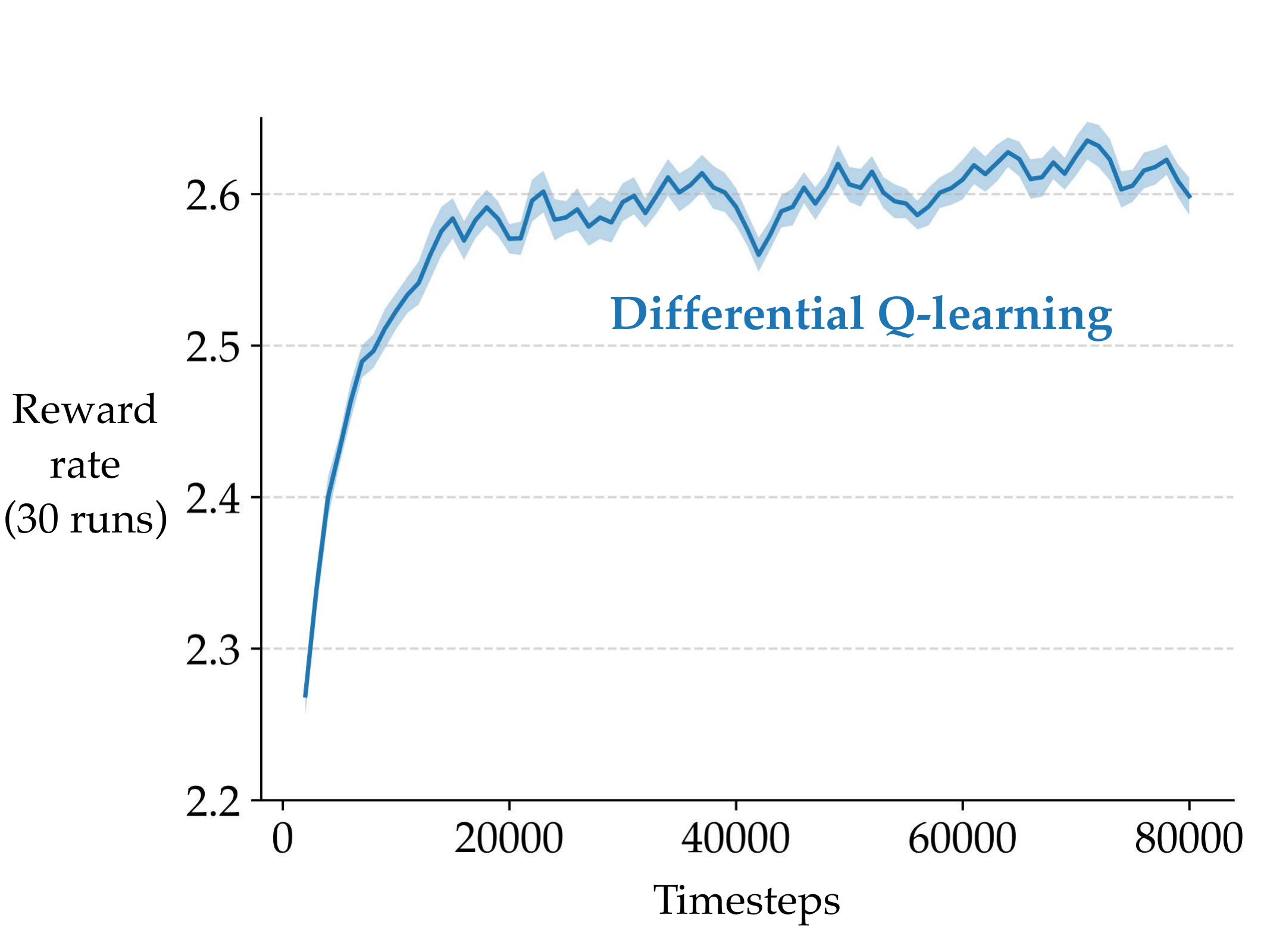}
    \end{subfigure}
    \caption{A typical learning curve for the Access-Control Queuing task. A point on the solid line denotes reward rate over the last 2000 timesteps, and the shaded region indicates one standard error.}
    \label{fig:results-accesscontrol-learningcurve}
\end{figure}


\section{Empirical Results for Control}

\begin{figure*}[t!]
\centering
\begin{subfigure}{.5\textwidth}
    \centering
    \includegraphics[width=0.94\textwidth]{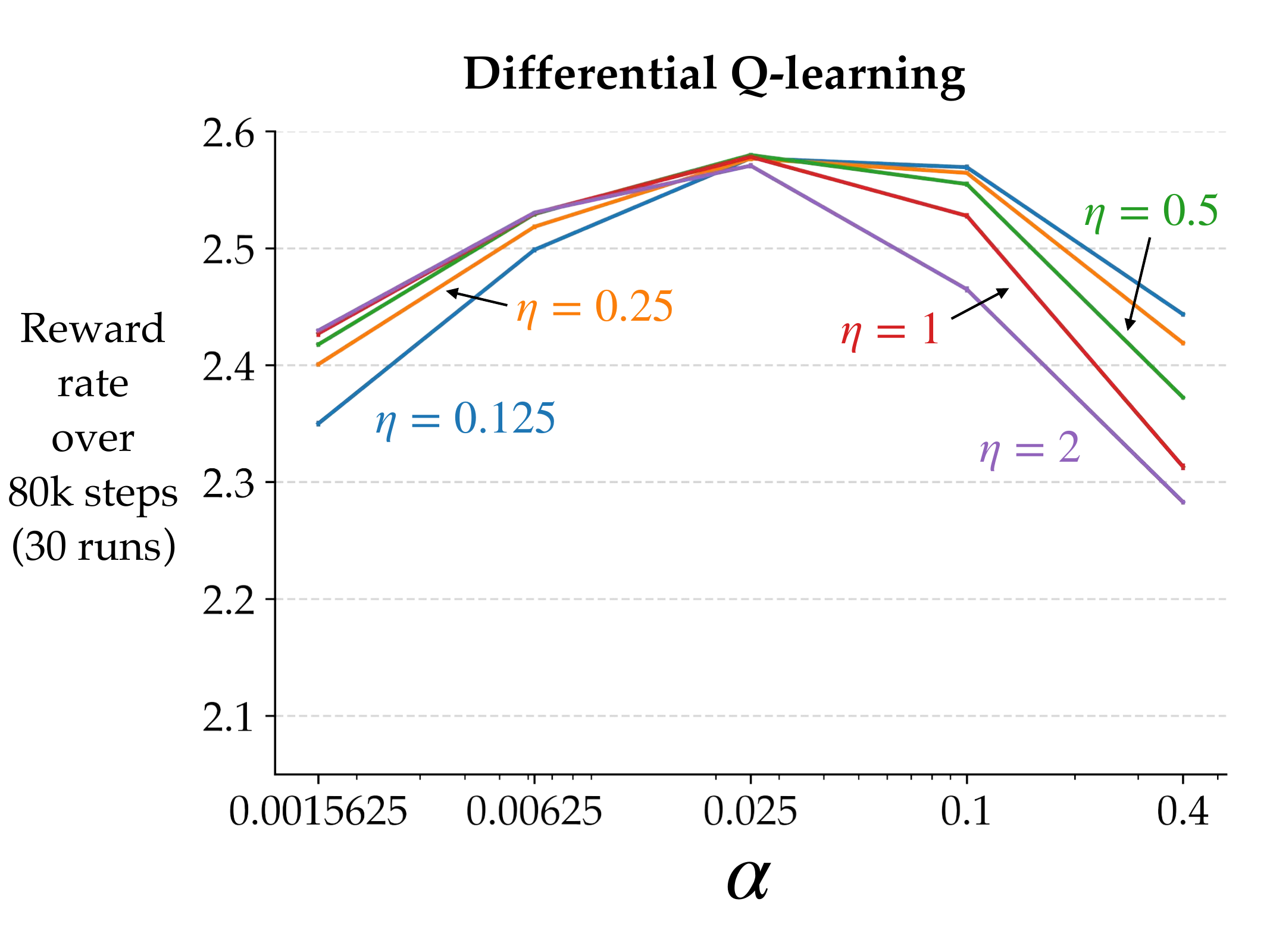}
\end{subfigure}%
\begin{subfigure}{.5\textwidth}
    \centering
    \includegraphics[width=0.94\textwidth]{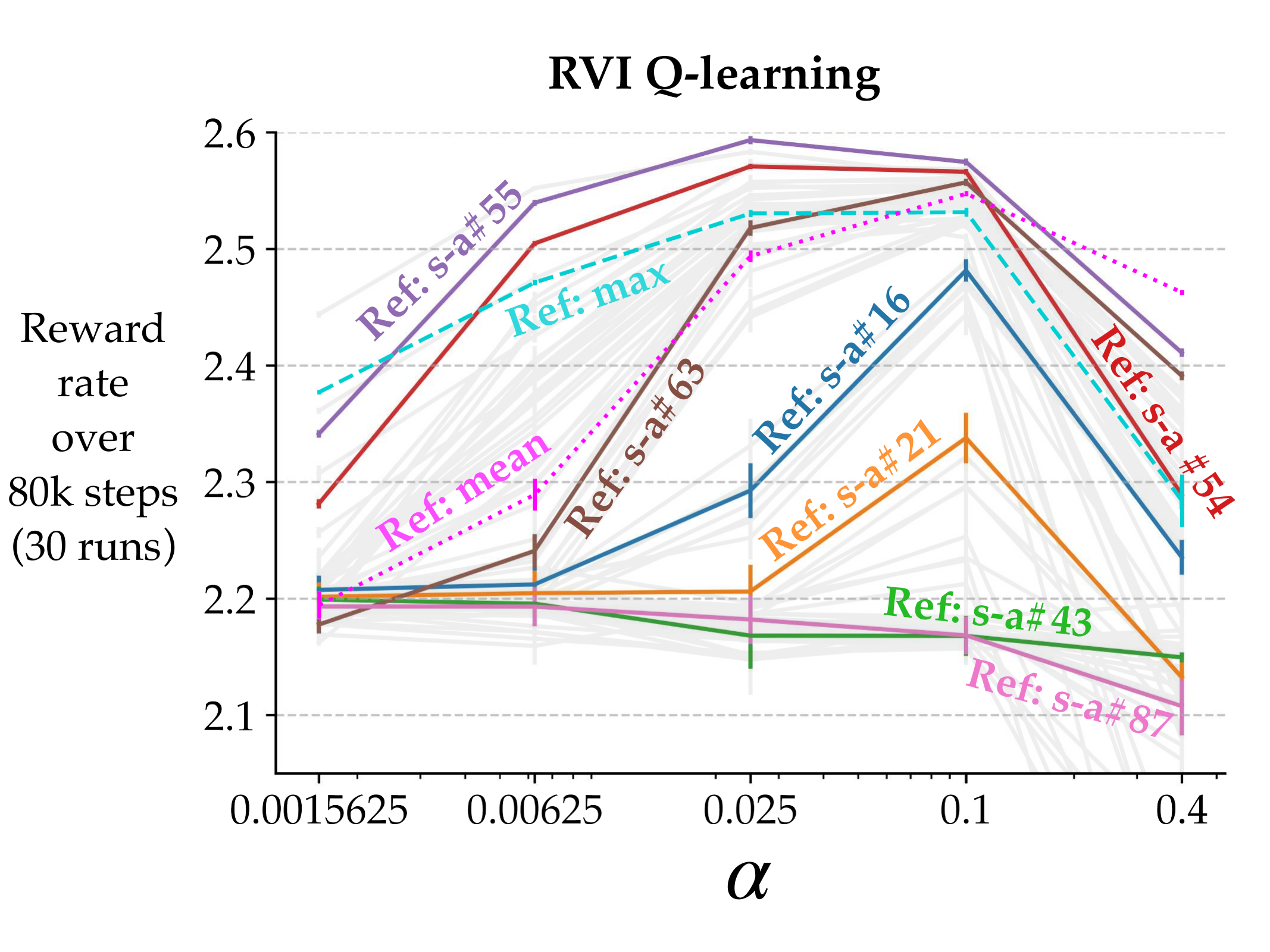}
\end{subfigure}
    \caption{Parameter studies showing the sensitivity of the two algorithms' performance to their parameters. The error bars indicate one standard error, which at times is less than the width of the solid lines. \textit{Left:} Differential Q-learning's rate of learning varied little over a broad range of its parameters. \textit{Right:} RVI Q-learning's rate of learning depended strongly on the choice of the reference function. The solid greyed-out lines mark the performance for each of the 88 state--action pairs considered individually as the single reference pair, with a few representative ones highlighted (labelled as `Ref: s-a'). The dotted lines correspond to the reference function being the mean or the max of all the action-value estimates.}
    \label{fig:results-accesscontrol-sensitivity}
\end{figure*}

\label{sec:control-exps}
In this section we present empirical results with both Differential Q-learning and RVI Q-learning algorithms on the Access-Control Queuing task (Sutton \& Barto 2018).
This task involves customers queuing up to access to one of 10 servers. The customers have differing priorities (1, 2, 4, or 8), which are also the rewards received if and when their service is complete. At each step, the customer at the head of queue is either accepted and allocated a free server (if any) or is rejected (in which case a reward of 0 is received). This decision is made based on the priority of the customer and the number of currently free servers, which together constitute the state of this average-reward MDP. The rest of the details of this test problem are exactly as described by Sutton and Barto (2018).

We applied RVI Q-learning and Differential Q-learning (pseudocodes for both algorithms are in Appendix \ref{app:pseudocodes}) to this task, each for 30 runs of 80,000 steps, and each for a range of step sizes $\alpha$. Differential Q-learning was run with a range of $\eta$ values, and RVI Q-learning was run with
three kinds of reference functions suggested by Abounadi et al.~(2001): (1) the value of a single reference state--action pair, 
for which we considered all possible 88 state--action pairs, 
(2) the maximum value of the action-value estimates,
and (3) the mean of the action-value estimates.
Both algorithms used an $\epsilon$-greedy behavior policy with $\epsilon = 0.1$. 
The rest of the experimental details are in \Cref{app:experiments}.

A typical learning curve is shown in Figure \ref{fig:results-accesscontrol-learningcurve}. While this learning curve is for Differential Q-learning, the learning curves for both algorithms typically started at around 2.2 and plateaued at around 2.6, with different parameter settings leading to different rates of learning. A reward rate of 2.2 corresponds to a policy that accepts every customer irrespective of their priority or the number of free servers---with positive rewards for every accept action, such a policy is learned rapidly in the first few timesteps starting from a zero initialization of value estimates (i.e., a random policy). 
The optimal performance was close to 2.7 (note both algorithms use an $\epsilon$-greedy policy without annealing $\epsilon$).

Figure \ref{fig:results-accesscontrol-sensitivity} shows parameter studies for each algorithm. Plotted is the reward rate averaged over all 80,000 steps, reflecting their rates of learning. The error bars denote one standard error. 

We saw that Differential Q-learning performed well on this task for a wide range of parameter values (left panel). Its two parameters did not interact strongly; the best value of $\alpha$ was independent of the choice of $\eta$. Moreover, the best performance for different $\eta$ values was roughly the same.

RVI Q-learning also performed well on this task for the best choice of the reference state--action pair, but its performance varied significantly for the various choices of the reference function and state--action pairs (right panel).

A closer look at the data revealed a correlation between the performance of a particular reference state--action pair and how frequently it occurs under 
an optimal policy. For example, state--action pairs 55 and 54 occurred frequently and also resulted in good performance. They correspond to states when only two servers are free and the customer at the front of the queue has priority 8 and 4 respectively, and the action is to accept. This is the optimal action in this state. On the other hand, the performance was poor with state--action pairs 43 and 87, which occurred rarely. They correspond to states when all 10 servers are free, a condition that rarely occurs in this problem. 
Finally, the mean of value estimates of all state--action pairs performs moderately well as a reference function. These observations lead us to a conjecture: \textit{an important factor determining the performance of RVI Q-learning with a single reference state--action pair is how often that pair occurs under an optimal policy}. This is problematic because knowing which state--action pairs will occur frequently under an optimal policy is tantamount to knowing the solution of the problem we set out to solve.

The conjecture might lead us to think that the reference function that is the max over all action-value estimates would always lead to good performance because the corresponding state--action pair would occur most frequently under an optimal policy, but this is not true in general. For example, consider an MDP with a state that rarely occurs under any policy. Let all rewards in the MDP be zero except a positive reward from that state. Then the highest action value among all state--action pairs is corresponding to this rarely-occurring state. 

To conclude, our experiments with the Access-Control Queuing task show that the performance of RVI Q-learning can vary significantly over the range of reference functions and state--action pairs. On the other hand, Differential Q-learning does not use a reference function and can be significantly easier to use.


\section{Learning and Planning for Prediction}
\label{sec:prediction_background_algorithms}

In this section, we define the problem setting for the prediction problem and then present our new algorithms for learning and planning.

In the prediction problem, we deal with Markov chains induced by the target and the behavior policies when applied to an MDP. The MDP interactions are the same as described earlier (\Cref{sec:control_background_algorithms}). 

As before, it is convenient to rule out the possibility of the reward rate of the target policy depending on the start state.
In particular, we assume that under the target policy there is only one possible limiting distribution for the resulting Markov chain, independent of the start state.
This is known as the Markov chain being \textit{unichain}.
The reward rate of the target policy then does not depend on the start state. We denote it as $r(\pi)$, where $\pi$ is the target policy:
\vspace{-2mm}
\begin{align}
\label{eq:avg_rew_definition_unichain}
    r(\pi) \doteq \lim_{n \to \infty} \frac{1}{n} \sum_{t = 1}^n \mathbb{E}[R_t \mid A_{0:t-1} \sim \pi].
\end{align}
The \textit{differential state-value function} (also called bias; see, e.g., Puterman 1994) $v_\pi: \mathcal{S} \to \mathbb{R}$ for a policy $\pi$ is:
\begin{align*}
    &v_\pi(s) \doteq \nonumber \\
    &\lim_{n \to \infty} \frac{1}{n} \sum_{k = 1}^n \sum_{t = 1}^k \bbE \left[R_t - r(\pi) \mid S_0 = s, A_{0:t-1} \sim \pi \right],
\end{align*}
for all $s\in\mathcal{S}$.
As usual, the differential state-value function satisfies a recursive Bellman equation:
\begin{align}
    \label{state value Bellman equation}
    v(s) = \sum_a \pi(a | s) \sum_{s', r} p(s', r \,|\, s, a) \big(r - \bar r + v(s')\big),
\end{align}
for all $s\in\mathcal{S}$.
The unique solution for $\bar{r}$ is $r(\pi)$ and the solutions for $v:\mathcal{S}\to\mathbb{R}$ are unique up to an additive constant. 

As usual in off-policy prediction learning, we need an assumption of \emph{coverage}. In this case we assume that every state--action pair for which $\pi(a | s) > 0$ occurs an infinite number of times under the behavior policy.

Our \emph{Differential TD-learning} algorithm updates a table of estimates $V_t: \calS \to \bbR$ as follows: 
\begin{align}
    V_{t+1}(S_t) & \doteq V_t(S_t) + \alpha_t \rho_t \delta_t, \label{TD Learning: V} \\
    V_{t+1}(s) & \doteq V_t(s),\ \forall s \neq S_t, \nonumber
\end{align}
where $\alpha_t$ is a step-size sequence, $\rho_t \doteq \pi(A_t | S_t) / b(A_t | S_t)$ is the importance-sampling ratio, and $\delta_t$ is the TD error:
\begin{equation} \label{TD Learning: TD error}
    \delta_t \doteq R_{t+1} - \bar R_t + V_t(S_{t+1}) - V_t(S_t),
\end{equation} 
where $\bar{R}_t$ is a scalar estimate of $r(\pi)$, updated by:
\begin{align}
\bar R_{t+1} & \doteq \bar R_t + \eta \alpha_t \rho_t \delta_t, \label{TD Learning: bar R}
\end{align}
and $\eta$ is a positive constant.

The following theorem shows that $\bar{R}_t$ converges to $r(\pi)$ and $V_t$ converges to a solution of $v$ in  \eqref{state value Bellman equation}.

\smallskip

\begin{theorem}[Informal] \label{Differential TD-learning}
If 1) the Markov chain induced by the target policy $\pi$ is unichain, 2) every state--action pair for which $\pi(a | s) > 0$ occurs an infinite number of times under the behavior policy, 3) the step sizes, specific to each state, are decreased appropriately, and 4) the ratio of the update frequency of the most-updated state to the update frequency of the least-updated state is finite, then the Differential TD-learning algorithm \eqref{TD Learning: V}–\eqref{TD Learning: bar R} converges, almost surely, $\bar R_t$ to $r(\pi)$ and $V_t$ to a solution of $v$ in the Bellman equation \eqref{state value Bellman equation}.
\end{theorem} 
\vspace{-2mm}
\begin{proof}
Essentially as in Theorem 1. Full proof in Appendix \ref{app:convergence-proofs}.
\end{proof}

Note that this result applies to both on-policy and off-policy problems. In off-policy problems, Differential TD-learning is the first model-free average-reward algorithm proved to converge to the true reward rate.

The planning version of Differential TD-learning, called \emph{Differential TD-planning}, 
uses simulated transitions generated just as in Differential Q-planning, except that Differential TD-planning chooses actions according to policy $b$ and not arbitrarily. Differential TD-planning maintains a table of value estimate $V_n: \calS \to \bbR$ 
and a reward rate estimate $\bar R_n$ and updates them just as in Differential TD-learning \eqref{TD Learning: V}–\eqref{TD Learning: bar R} using $S_n, A_n, R_n, S_n'$ instead of $S_t, A_{t}, R_{t+1}, S_{t+1}$. \smallskip

\begin{theorem}[Informal]
Under the same assumptions made in Theorem \ref{Differential TD-learning} (except now for the model MDP corresponding to $\hat p$ rather than $p$) the Differential TD-planning algorithm converges, almost surely, $\bar R_n$ to $\hat r(\pi)$ and $V_n$ to a solution of $v$ in the state-value Bellman equation (cf. \eqref{state value Bellman equation}) for the model MDP.
\end{theorem} 
\vspace{-2mm}
\begin{proof}
Essentially as in Theorem 1. Full proof in Appendix \ref{app:convergence-proofs}.
\end{proof}


\section{Empirical Results for Prediction}
\label{sec:prediction-experiments}

\begin{figure*}[b]
\centering
\begin{subfigure}{.5\textwidth}
    \centering
    \includegraphics[width=0.9\textwidth]{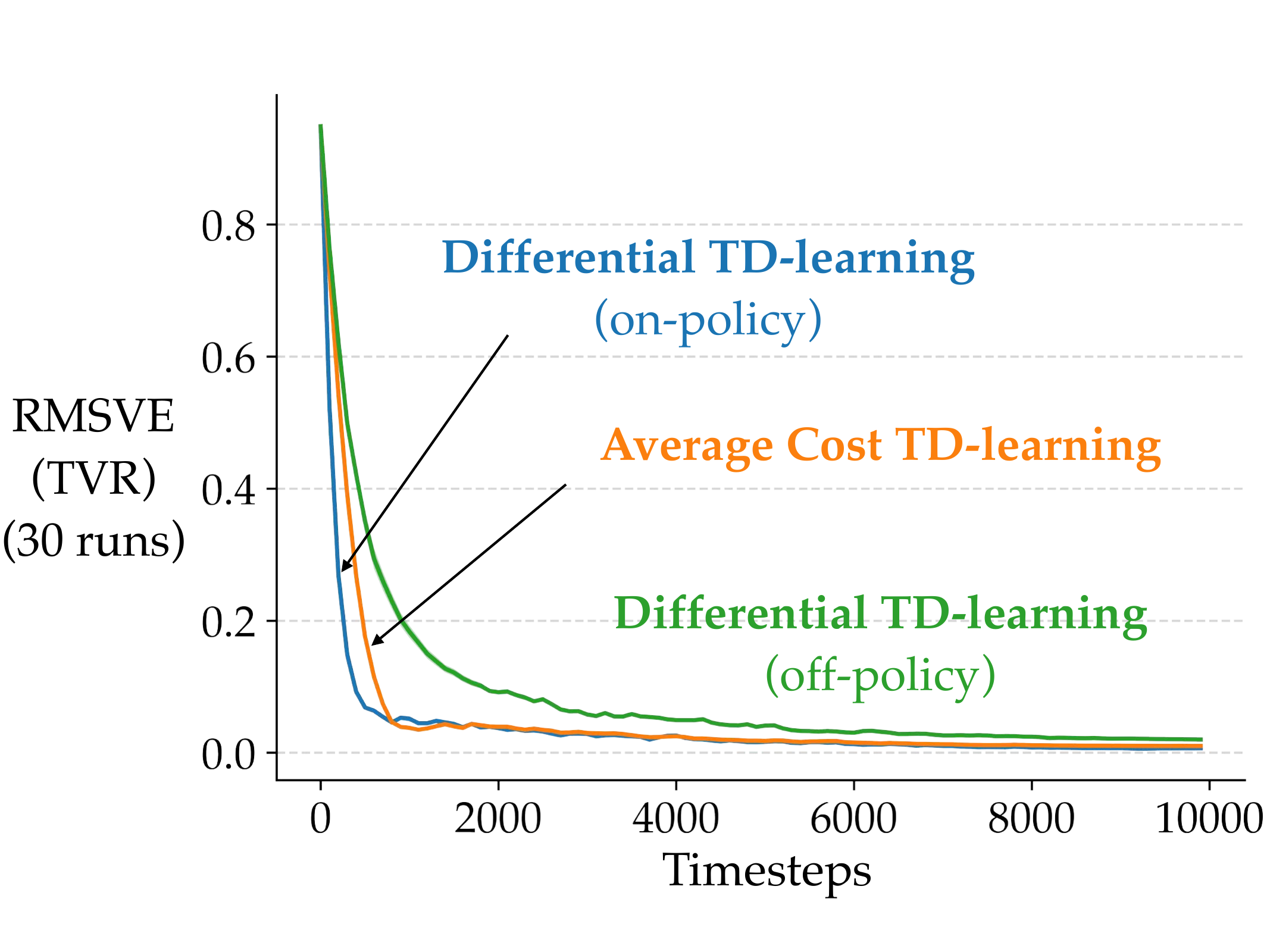}
\end{subfigure}%
\begin{subfigure}{.5\textwidth}
    \centering
    \includegraphics[width=0.9\textwidth]{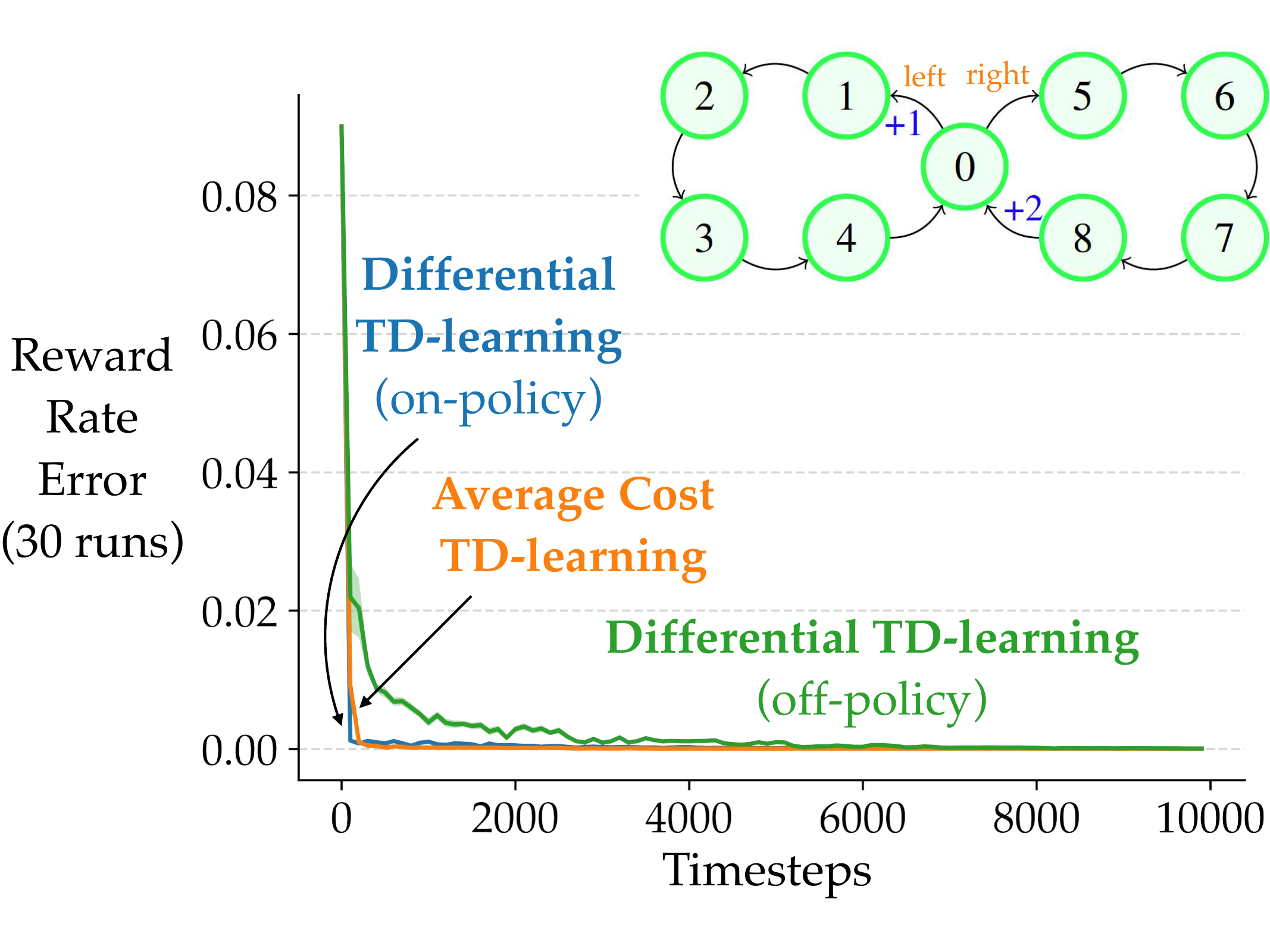}
\end{subfigure}
\begin{subfigure}{.5\textwidth}
    \centering
    \includegraphics[width=0.9\textwidth]{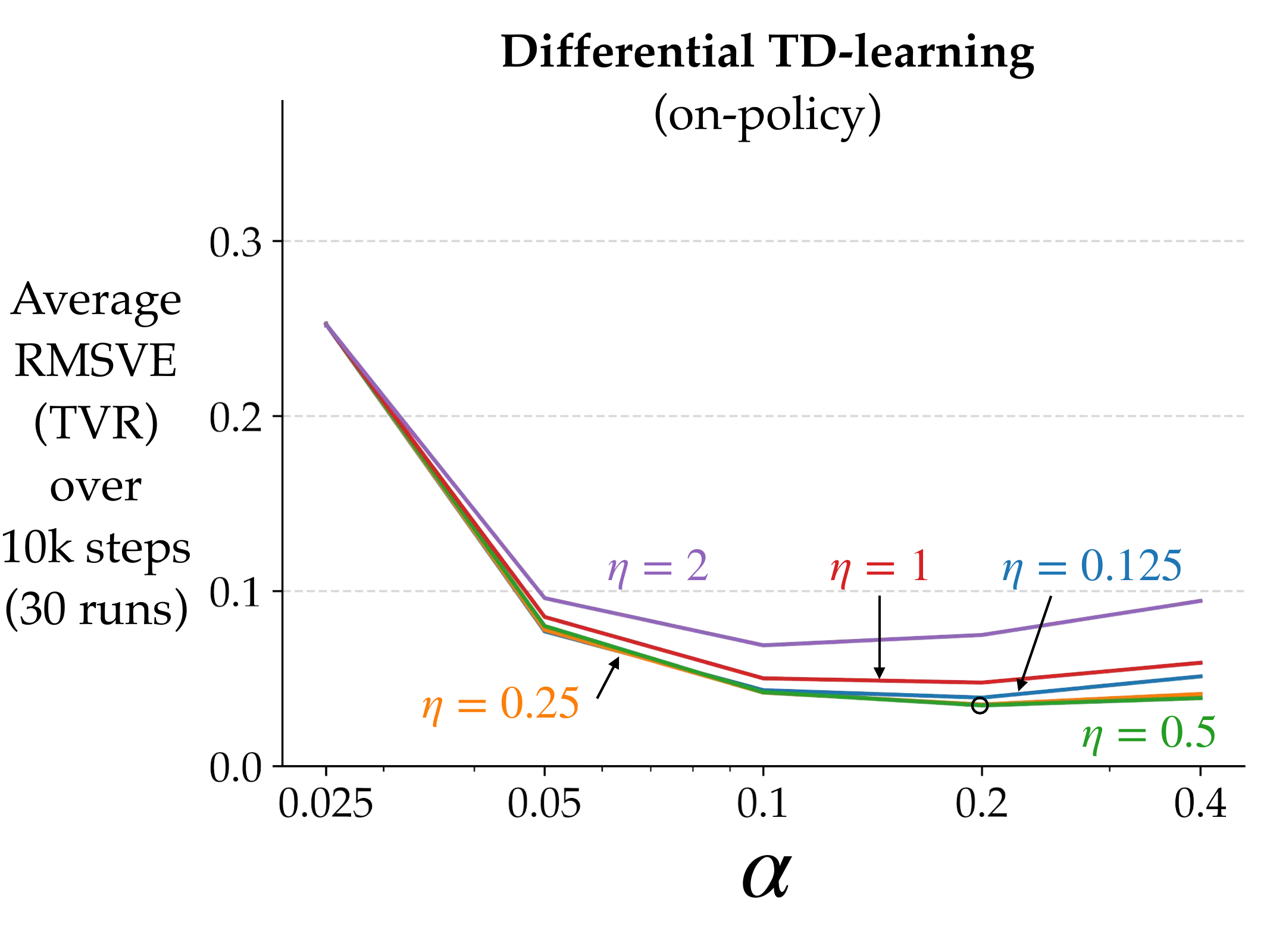}
\end{subfigure}%
\begin{subfigure}{.5\textwidth}
    \centering
    \includegraphics[width=0.9\textwidth]{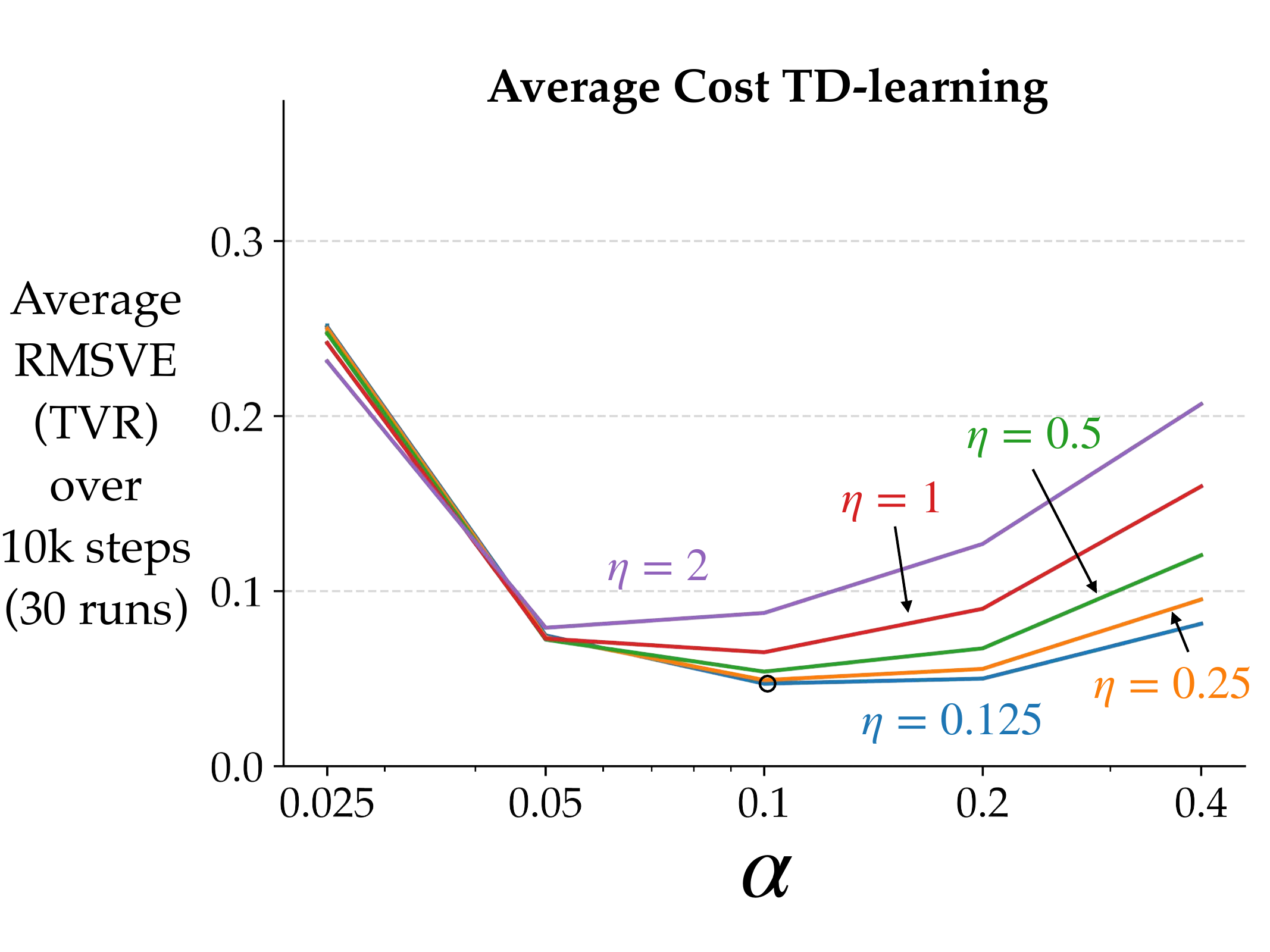}
\end{subfigure}
    \caption{Learning curves and parameter studies for Differential TD-learning and Average Cost TD-learning on the Two Loop task (inset top-right). The standard errors are thinner than width of the solid lines. \textit{Top:} Exemplary learning curves showing all three algorithms tend to zero errors in terms of RMSVE (TVR) and RRE. 
    \textit{Bottom:} Parameter studies showing the performance of Differential TD-learning in terms of average RMSVE (TVR) is less sensitive to the choice of parameters $\alpha$ and $\eta$ than Average Cost TD-learning. The black circles in the bottom row denote the parameter configurations for which the learning curves in the top row are shown. 
    }
\label{fig:results-prediction}
\end{figure*}

In this section we present empirical results with average-reward prediction learning algorithms using the Two Loop task shown in the upper right of Figure \ref{fig:results-prediction} (cf.\ Mahadevan 1996, Naik et al.\ 2019). This is a continuing MDP with only one action in every state except state $0$. Action \texttt{left} in state $0$ gives an immediate reward of +1 and action \texttt{right} leads to a delayed reward of +2 after five steps. The optimal policy is to take the action \texttt{right} in state $0$ to obtain a reward rate of 0.4 per step. The easier-to-find sub-optimal policy of going \texttt{left} results in a reward rate of 0.2.

We performed two prediction experiments: on-policy and off-policy. For the first on-policy experiment, the policy $\pi$ to be evaluated was the one that randomly picks \texttt{left} or \texttt{right} in state $0$ with probability 0.5. The reward rate corresponding to this policy is 0.3. In addition to the on-policy version of Differential TD-learning, we ran Tsitsiklis and Van Roy's (1999) Average Cost TD-learning. It is an on-policy algorithm with the following updates:
\begin{align}
    V_{t+1}(S_t) &\doteq V_t(S_t) + \alpha_t \delta_t, \nonumber \\
    \bar{R}_{t+1} &\doteq \bar{R}_t + \eta\alpha_t \big( R_{t+1} - \bar{R}_{t} \big), \label{eq:AvgCostTD_update_rbar}
\end{align}
where $\delta_t$ is the TD error as in \eqref{TD Learning: TD error}. Both algorithms have the same two step-size parameters. For each parameter setting, 30 runs of 10,000 steps each were performed. 

We evaluated the accuracy of the estimated value function as well as the estimated reward rate of the target policy. The top-left panel in Figure \ref{fig:results-prediction} shows the learning curves of the two algorithms (blue and orange) in terms of root-mean-squared value error (RMSVE) w.r.t.\ timesteps. We used Tsitsiklis and Van Roy's (1999) variant of RMSVE which measures the distance of the estimated values to the nearest solution that satisfies the state-value Bellman equation \eqref{state value Bellman equation}. We denote this metric by `RMSVE (TVR)'. Details on how it was computed are provided in Appendix \ref{app:experiments} along with the complete experimental details. We saw that the RMSVE (TVR) went to zero in a few thousand steps for both on-policy Differential TD-learning and Average Cost TD-learning. The top-right panel shows the learning curves of the two algorithms (blue and orange) in terms of squared error in the estimate of the reward rate w.r.t.\ the true reward rate of the target policy (($r(\pi)-\bar{R}_t)^2$, denoted as reward rate error or `RRE'), which also went to zero for both algorithms. 

The plots in the bottom row indicate the sensitivity of the performance of these two algorithms to the two step-size parameters $\alpha$ and $\eta$. The average RMSVE (TVR) over all the 10k timesteps was equal or lower for Differential TD-learning than Average Cost TD-learning across the range of parameters tested. In addition, on-policy Differential TD-learning was less sensitive to the values of both $\alpha$ and $\eta$ than Average Cost TD-learning. This was also the case with RRE, the plots for which are reported in Appendix \ref{app:experiments}. 

The green learning curves in the top row of Figure \ref{fig:results-prediction} correspond to the off-policy version of Differential TD-learning. This was used in the second off-policy experiment: the same policy as in the on-policy experiment was evaluated (i.e., target policy takes either action in state $0$ with probability 0.5), now using data collected with a behavior policy that picks the \texttt{left} and \texttt{right} actions with probabilities 0.9 and 0.1 respectively. Both RMSVE (TVR) and RRE went to zero for off-policy Differential TD-learning within a reasonable amount of time. Its parameter studies for both RMSVE (TVR) and RRE are presented in Appendix \ref{app:experiments} along with additional experimental details.\footnote{Average Cost TD-learning cannot be extended to the off-policy setting due to the use of a sample average of the observed rewards to estimate the reward rate \eqref{eq:AvgCostTD_update_rbar}. For more details, please refer to Appendix \ref{app:additional}.}

Our experiments 
show that our on- and off-policy Differential TD-learning algorithms can accurately estimate the value function and the reward rate of a given target policy, as expected from Theorem \ref{Differential TD-learning}. In addition, on-policy Differential TD-learning can be easier to use than Average Cost TD-learning.


\section{Estimating the Actual Differential Value Function}
\label{sec:centering}

All average-reward algorithms, including the ones we proposed, converge to an \emph{uncentered} differential value function, in other words, the actual differential value function plus some unknown offset that depends on the algorithm itself and design choices such as initial values or reference states. 

We now introduce a simple technique to compute the offset in the value estimates for both on- and off-policy learning and planning. Once the offset is computed, it can simply be subtracted from the value estimates to obtain the estimate of the actual (centered) differential value function. 

We demonstrate how the offset can be eliminated in Differential TD-learning; the other cases (Differential TD-planning, Differential Q-learning and Differential Q-planning) are shown in Appendix \ref{app:convergence-proofs}. For this purpose, we introduce, in addition to the first estimator \eqref{TD Learning: V}–\eqref{TD Learning: bar R}, a second estimator for which the rewards are the value estimates of the first estimator. The second estimator maintains an estimate of the scalar offset $\bar V_t$, an auxiliary table of estimates $F_t(s), \forall s \in \calS$, and uses the following update rules:
\vspace{-1mm}
\begin{align}
    F_{t+1}(S_t) & \doteq F_t(S_t) + \beta_t \rho_t \Delta_t,  \label{Centered TD Learning: F} \\
    F_{t+1}(s) & \doteq F_t(s),\ \forall s \neq S_t,
    \nonumber
\end{align}
where $\beta_t$ is a step-size sequence, $\Delta_t$ is the TD error of the second estimator:
\vspace{-1mm}
\begin{align} \label{Centered TD Learning: TD error}
    \Delta_t \doteq V_{t}(S_{t}) - \bar V_t + F_t(S_{t+1}) - F_t(S_t), 
\end{align}
where:
\begin{align}
    \bar V_{t+1} & \doteq \bar V_t + \kappa \beta_t \rho_t \Delta_t, \label{Centered TD Learning: bar V}
\end{align}
and $\kappa$ is a positive constant. 
We call \eqref{TD Learning: V}–\eqref{TD Learning: bar R} with \eqref{Centered TD Learning: F}–\eqref{Centered TD Learning: bar V} \emph{Centered Differential TD-learning}. 
Before presenting the convergence theorem, we briefly give some intuition on why this technique can successfully compute the offset. By Theorem \ref{Differential TD-learning}, $\bar R_t$ converges to $r(\pi)$ and $V_t$ converges to some $v_\infty$ almost surely, where $v_\infty(s) = v_\pi(s) + c, \forall s \in \calS$ for some offset $c \in \bbR$. In Appendix \ref{app:convergence-proofs}, we show $\sum_s d_\pi(s) v_\pi(s) = 0$, where $d_\pi$ is the limiting state distribution following policy $\pi$, which implies $\sum_s d_\pi(s) v_\infty(s) = c$. As $V_t$ converges to $v_\infty$, $\sum_s d_\pi(s) V_t(s)$ converges to $c$. Now note that $\sum_s d_\pi(s) V_t(s)$ and $r(\pi) = \sum_s d_\pi(s) r_\pi(s)$ are of the same form. Therefore $\sum_s d_\pi(s) V_t(s)$ can be estimated similar to how $r(\pi)$ is estimated, using $V_t$ as the reward. This leads to the second estimator: \eqref{Centered TD Learning: F}–\eqref{Centered TD Learning: bar V}.\smallskip

\begin{theorem}[Informal]
If the assumptions in Theorem \ref{Differential TD-learning} hold, and the step sizes, specific to each state, are decreased appropriately, then Centered Differential TD-learning \eqref{Centered TD Learning: F}–\eqref{Centered TD Learning: bar V} converges, almost surely, $V_t(s) - \bar V_t$ to $v_\pi(s)$ for all $s$ and $\bar R_t$ to $r(\pi)$.
\end{theorem}

The proof is presented in Appendix \ref{app:convergence-proofs}. We also demonstrate how this technique can be used to learn the actual differential value function with an experiment in Appendix \ref{app:experiments} (with full pseudocode in Appendix \ref{app:pseudocodes}). 


\section{Discussion and Future Work}

We have presented several new learning and planning algorithms for average-reward MDPs. 
Our algorithms 
differ from previous work in that they do not involve reference functions, they apply in off-policy settings for both prediction and control, and they find centered value functions. In our opinion, these changes make the average-reward formulation more appealing for use in reinforcement learning.

The most important way in which our work is limited is that it treats only the tabular case, whereas some form of function approximation is necessary for large-scale applications and the larger ambitions of artificial intelligence. Indeed, the need for function approximation is a large part of the motivation for studying the average-reward setting. We present some ideas for extending our algorithms to linear function approximation in Appendix \ref{app:extensions_lfa}. However, the theory and practice are both more challenging in the function approximation setting, and much future research is needed. 

Our work is also limited in ways that are unrelated to function approximation. 
One is that we treat only one-step methods and not $n$-step, $\lambda$-return, or sophisticated eligibility-trace methods (van Seijen et al.~2016, Sutton \& Barto 2018).
Another important direction for future work is to extend these algorithms to semi-Markov decision processes so that they can be used with temporal abstractions like options (Sutton, Precup, \& Singh 1999).


\section*{Acknowledgements}
The authors were supported by DeepMind, Amii, NSERC, and CIFAR. 
The authors wish to thank Vivek Borkar, Huizhen Yu, Martha White, Csaba Szepesv\'{a}ri, Dale Schuurmans, and Benjamin Van Roy for their valuable feedback during various stages of the work. Computing resources were provided by Compute Canada.



\section*{References}
\parskip=5pt
\parindent=0pt
\def\hangin{\hangindent=0.2in}
\def\bibitem[#1]#2{\hangin}

\hangin
Abounadi, J., Bertsekas, D., Borkar, V. S. (1998). \textit{Stochastic Approximation for Nonexpansive Maps: Application to Q-Learning}, Report LIDS-P-2433, Laboratory for Information and Decision Systems, MIT.

\hangin
Abounadi, J., Bertsekas, D., Borkar, V. S. (2001).  Learning algorithms for Markov decision processes with average cost. \emph{SIAM Journal on Control and Optimization, 40}(3):681--698.

\hangin
Abbasi-Yadkori, Y., Bartlett, P., Bhatia, K., Lazic, N., Szepesvari, C., Weisz, G. (2019a). {POLITEX}: Regret bounds for policy iteration using expert prediction. In \emph{Proceedings of the International Conference on Machine Learning}, pp.~3692--3702.

\hangin
Abbasi-Yadkori, Y., Lazic, N., Szepesvari, C., Weisz, G. (2019b). Exploration-enhanced {POLITEX}. \emph{ArXiv:1908.10479.}

\hangin
Auer, R., Ortner, P. (2006). Logarithmic online regret bounds for undiscounted reinforcement learning. In \emph{Advances in Neural Information Processing Systems}, pp.~49--56.

\hangin
Bellman, R. E. (1957). \textit{Dynamic Programming}. Princeton University Press.

\hangin
Bertsekas, D. P., Tsitsiklis, J. N. (1996). \emph{Neuro-dynamic Programming}. Athena Scientific.


\hangin
Borkar, V. S. (1998). Asynchronous stochastic approximations. \emph{SIAM Journal on Control and Optimization, 36}(3):840--851.

\hangin
Borkar, V. S. (2009). \emph{Stochastic Approximation: A Dynamical Systems Viewpoint}. Springer.


\hangin
Brafman, R. I., Tennenholtz, M. (2002). R-{MAX} — a general polynomial time algorithm for near-optimal reinforcement learning. \emph{Journal of Machine Learning Research, 3}(10):213--231.

\hangin
Das, T. K., Gosavi, A., Mahadevan, S. Marchalleck, N. (1999). Solving semi-Markov decision problems using average reward reinforcement learning. \emph{Management Science, 45}(4):560--574.

\hangin
Dewanto, V., Dunn, G., Eshragh, A., Gallagher, M., Roosta, F. (2020). Average-reward model-free reinforcement learning: a systematic review and literature mapping. \emph{ArXiv:2010.08920.}

\hangin
Gosavi, A. (2004). Reinforcement learning for long-run average cost. \textit{European Journal of Operational Research, 155}(3):654--674.

\hangin
Howard, R. A. (1960). \textit{Dynamic Programming and Markov Processes}. MIT Press.

\hangin
Jalali, A., Ferguson, M. J. (1989). Computationally efficient adaptive control algorithms for Markov chains. In \emph{Proceedings of the IEEE Conference on Decision and Control}, pp.~1283--1288. 

\hangin
Jalali, A., Ferguson, M. J. (1990). A distributed asynchronous algorithm for expected average cost dynamic programming. In \emph{Proceedings of the IEEE Conference on Decision and Control}, pp.~1394--1395.

\hangin
Jaksch, T., Ortner, R., Auer, P. (2010). Near-optimal Regret Bounds for Reinforcement Learning. \emph{Journal of Machine Learning Research, 11}(4):1563--1600.

\hangin
Kakade, S. M. (2001). A natural policy gradient. In \emph{Advances in Neural Information Processing Systems}, pp.~1531--1538.

\hangin
Kearns, M., Singh, S. (2002). Near-optimal reinforcement learning in polynomial time. \emph{Machine Learning, 49}(2):209--232.


\hangin
Konda, V. R., (2002). \textit{Actor-critic algorithms}. Ph.D. dissertation, MIT.

\hangin
Liu, Q., Li, L., Tang, Z., Zhou, D. (2018). Breaking the curse of horizon: Infinite-horizon off-policy estimation. In \emph{Advances in Neural Information Processing Systems}, pp.~5356--5366.

\hangin
Mahadevan, S. (1996). Average reward reinforcement learning: Foundations, algorithms, and empirical results. \emph{Machine Learning, 22}(1--3):159--195.

\hangin
Marbach, P., Tsitsiklis, J. N. (2001). Simulation-based optimization of Markov reward processes. \emph{IEEE Transactions on Automatic Control, 46}(2):191--209.

\hangin
Mousavi, A., Li, L., Liu, Q., Zhou, D. (2020). Black-box off-policy estimation for infinite-horizon reinforcement learning. \emph{ArXiv:2003.11126.}

\hangin
Naik, A., Shariff, R., Yasui, N., Sutton, R. S. (2019). Discounted reinforcement learning is not an optimization problem. \emph{Optimization Foundations for Reinforcement Learning Workshop at the Conference on Neural Information Processing Systems}. Also arXiv:1910.02140.

\hangin
Puterman, M. L. (1994). \emph{Markov Decision Processes: Discrete Stochastic Dynamic Programming.} John Wiley \& Sons.

\hangin
Ren, Z., Krogh, B. H. (2001). Adaptive control of Markov chains with average cost. \emph{IEEE Transactions on Automatic Control, 46}(4):613--617.

\hangin
Schwartz, A. (1993). A reinforcement learning method for maximizing undiscounted rewards. In \emph{Proceedings of the International Conference on Machine Learning,} pp.~298--305. 

\hangin
Schweitzer, P. J., \& Federgruen, A. (1978). The Functional Equations of Undiscounted Markov Renewal Programming. \emph{Mathematics of Operations Research, 3}(4), pp.~308--321.


\hangin
Singh, S. P. (1994). Reinforcement learning algorithms for average-payoff Markovian decision processes. In \textit{Proceedings of the AAAI Conference on Artificial Intelligence}, pp.~700--705.

\hangin
Sutton, R. S. (1990). Integrated architectures for learning, planning, and reacting based on approximating dynamic programming. In \emph{Proceedings of the International Conference on Machine Learning}, pp.~216--224. 

\hangin
Sutton, R. S., McAllester, D. A., Singh, S. P., Mansour, Y. (1999). Policy gradient methods for reinforcement learning with function approximation. In \emph{Advances in Neural Information Processing Systems}, pp.~1057--1063.

\hangin
Sutton, R. S., Barto, A. G. (2018). \emph{Reinforcement Learning: An Introduction.} MIT Press.



\hangin
Tang, Z., Feng, Y., Li, L., Zhou, D., Liu, Q. (2019). Doubly robust bias reduction in infinite horizon off-policy estimation. \emph{ArXiv:1910.07186}.

\hangin
Tsitsiklis, J. N., Van Roy, B. (1999). Average cost temporal-difference learning. \emph{Automatica, 35}(11):1799--1808.

\hangin
van Seijen, H., Mahmood, A. R., Pilarski, P. M., Machado, M. C., Sutton, R. S. (2016). True online temporal-difference learning. \textit{Journal of Machine Learning Research, 17}(145):1--40.

\hangin
Wen, J., Dai, B., Li, L., Schuurmans, D. (2020). Batch Stationary Distribution Estimation. In \textit{Proceedings of the International Conference on Machine Learning}, pp.~10203--10213.

\hangin
Wheeler, R., Narendra, K. (1986). Decentralized learning in finite Markov chains. \emph{IEEE Transactions on Automatic Control, 31}(6):519--526.


\hangin
White, D. J. (1963). Dynamic programming, Markov chains, and the method of successive approximations. \emph{Journal of Mathematical Analysis and Applications, 6}(3):373--376.

\hangin
Yang, S., Gao, Y., An, B., Wang, H., Chen, X. (2016). Efficient average reward reinforcement learning using constant shifting values. In \emph{Proceedings of the AAAI Conference on Artificial Intelligence}, pp.~2258--2264.

\hangin
Yu, H., \& Bertsekas, D. P. (2009). Convergence results for some temporal difference methods based on least squares. \emph{IEEE Transactions on Automatic Control, 54}(7):1515--1531.

\hangin
Zhang, R., Dai, B., Li, L., Schuurmans, D. (2020a). Gen{DICE}: Generalized offline estimation of stationary values. \emph{ArXiv:2002.09072.}

\hangin
Zhang, S., Liu, B., Whiteson, S. (2020b). Gradient{DICE}: Rethinking generalized offline estimation of stationary values. In \textit{Proceedings of the International Conference on Machine Learning}, pp.~11194--11203.




\onecolumn
\appendix

\counterwithin{figure}{section} 
\counterwithin{table}{section}
\counterwithin{theorem}{section}
\counterwithin{lemma}{section}
\counterwithin{assumption}{section}
\counterwithin{equation}{section}

\section{Algorithm Pseudocodes}
\label{app:pseudocodes}

This section contains the pseudocodes for the algorithms used in the experiments in this paper:
\begin{itemize}
    \item Section \ref{sec:control-exps} - Empirical Results for Control: \\Differential Q-learning and RVI Q-learning
    \item Section \ref{sec:prediction-experiments} - Empirical Results for Prediction: \\Differential TD-learning and Average Cost TD learning
    \item Section \ref{sec:centering} - Estimating the Actual Differential Value Function: \\Centered Differential Q-learning
\end{itemize}

\begin{algorithm}[h]
\SetAlgoLined
\KwIn{The policy $b$ to be used (e.g., $\epsilon$-greedy)}
\SetKwInput{AP}{Algorithm parameters}
\AP{step-size parameters $\alpha, \eta$}
Initialize $Q(s,a)\ \forall s,a; \bar{R}$ arbitrarily (e.g., to zero) \\
Obtain initial $S$ \\
 \While{still time to train}
 {
    $A \leftarrow$ action given by $b$ for $S$ \\
    Take action $A$, observe $R, S'$ \\
    $\delta = R - \bar{R} + \max_{a}Q(S',a) - Q(S,A)$ \\
    $Q(S,A) = Q(S,A) + \alpha \delta$ \\
    $\bar{R} = \bar{R} + \eta \alpha \delta$ \\
    $S = S'$ \\
 }
 return $Q$
 \caption{Differential Q-learning (one-step off-policy control)}
 \label{algo:diffQ-uncentered}
\end{algorithm}

\begin{algorithm}[h]
\SetAlgoLined
\KwIn{The policy $b$ to be used (e.g., $\epsilon$-greedy)}
\SetKwInput{AP}{Algorithm parameters}
\AP{step-size parameter $\alpha$}
Initialize $Q(s,a)\ \forall s,a$ arbitrarily (e.g., to zero) \\
Choose function $f(Q) \big($e.g., a single reference state--action pair — $f(Q) = Q(s_0,a_0)\big)$ \\
Obtain initial $S$ \\
 \While{still time to train}
 {
    $A \leftarrow$ action given by $b$ for $S$ \\
    Take action $A$, observe $R, S'$ \\
    $\delta = R - f(Q) + \max_{a}Q(S',a) - Q(S,A)$ \\
    $Q(S,A) = Q(S,A) + \alpha \delta$ \\
    $S = S'$ \\
 }
 return $Q$
 \caption{RVI Q-learning (one-step off-policy control)}
 \label{algo:rviQ}
\end{algorithm}

\begin{algorithm}[h]
\DontPrintSemicolon
\SetAlgoLined
\KwIn{The policy $\pi$ to be evaluated, and $b$ to be used}
\SetKwInput{AP}{Algorithm parameters}
\AP{step-size parameters $\alpha, \eta$}
Initialize $V(s)\; \forall s,\; \bar{R}$ arbitrarily (e.g., to zero) \;
 \While{still time to train}
 {
    $A \leftarrow$ action given by $b$ for $S$\;
    Take action $A$, observe $R, S'$\;
    $\delta = R - \bar{R} + V(S') - V(S)$\;
    $\rho = \pi(A | S)\, /\, b(A | S) $\;
    $V(S) = V(S) + \alpha \rho \delta$\;
    $\bar{R} = \bar{R} + \eta \alpha \rho \delta$\;
    $S = S'$ \;
 }
 return $V$\;
 \caption{Differential TD-learning (one-step off-policy prediction)}
 \label{algo:diffTD0-off}
\end{algorithm}

\begin{algorithm}[h]
\DontPrintSemicolon
\SetAlgoLined
\KwIn{The policy $\pi$ to be evaluated}
\SetKwInput{AP}{Algorithm parameters}
\AP{step-size parameters $\alpha, \eta$}
Initialize $V(s)\; \forall s,\; \bar{R}$ arbitrarily (e.g., to zero) \;
 \While{still time to train}
 {
    $A \leftarrow$ action given by $\pi$ for $S$\;
    Take action $A$, observe $R, S'$\;
    $\delta = R - \bar{R} + V(S') - V(S)$\;
    $V(S) = V(S) + \alpha \delta$\;
    $\bar{R} = \bar{R} + \eta \alpha (R - \bar{R})$\;
    $S = S'$ \\
 }
 return $V$\;
 \caption{Average Cost TD-learning (one-step on-policy prediction)}
 \label{algo:avgcostTD0}
\end{algorithm}

\begin{algorithm}[h]
\SetAlgoLined
\KwIn{The policy $b$ to be used (e.g., $\epsilon$-greedy)}
\SetKwInput{AP}{Algorithm parameters}
\AP{step-size parameters $\alpha, \eta, \beta, \kappa$}
Initialize $Q(s,a), F(s,a)\ \forall s,a; \bar{R}, \bar{Q}$ arbitrarily (e.g., to zero) \\
Obtain initial $S$ \\
 \While{still time to train}
 {
    $A \leftarrow$ action given by $b$ for $S$ \\
    Take action $A$, observe $R, S'$ \\
    $\delta = R - \bar{R} + \max_{a}Q(S',a) - Q(S,A)$ \\
    $Q(S,A) = Q(S,A) + \alpha \delta$ \\
    $\bar{R} = \bar{R} + \eta \alpha \delta$ \\
    $\Delta = Q(S,A) - \bar{Q} + F(S',\argmax_{a}Q(S',a)) - F(S,A)$ \\
    $F(S,A) = F(S,A) + \beta \Delta$ \\
    $\bar{Q} = \bar{Q} + \kappa \beta \Delta$ \\
    $S = S'$ \\
 }
 return $Q - \bar{Q} e$, where $e$ is a $\abs{\calS \times \calA}$ vector of all ones.
 \caption{Centered Differential Q-learning}
 \label{algo:diffQ-centered}
\end{algorithm}

\clearpage
\section{Convergence Proofs}
\label{app:convergence-proofs}

In this section, we present the convergence proofs of Differential Q-learning and Differential Q-planning in subsection \ref{app:convergence-proof-diffq}, of Differential TD-learning and Differential TD-planning in subsection \ref{app:convergence-proof-difftd}, and that of the centered version of these algorithms in subsection \ref{app:convergence-proof-centering}.

For convenience, the following notations are used for all the proofs:
\begin{itemize}\itemsep0mm
    \item Given any vector $x$, $\sum x$ denotes sum of all elements in $x$. Formally, $ \sum x \doteq \sum_{i} x(i)$.
    \item $e$ denotes an all-ones vector, whose length may be $\vert \calS \times \calA \vert$ or $\vert \calS \vert$ depending on the context. 
    \item Finally, $\exp(\cdot)$ is used instead of $e^{(\cdot)}$ to denote the exponential function.
\end{itemize}

\subsection{Proof of Differential Q-learning and Differential Q-planning}
\label{app:convergence-proof-diffq}

In this section, we analyze a general algorithm that includes both Differential Q-learning and Differential Q-planning cases. We call it \emph{General Differential Q}. We first formally define it and then explain why both Differential Q-learning and Differential Q-planning are special cases of General Differential Q. We then provide assumptions and the convergence theorem of General Differential Q. The theorem would lead to convergence of Differential Q-learning and Differential Q-planning. Finally, we provide a proof for the theorem.

Given a MDP $\calM \doteq (\calS, \calA, \calR, p)$, for each state $s \in \calS $ action $a \in \calA$ and discrete step $n \geq 0$, let $R_n(s, a), S'_n(s, a) \sim p(\cdot, \cdot \mid s, a)$ denote a sample of resulting state and reward. We hypothesize a set-valued process $\{Y_n\}$ taking
values in the set of nonempty subsets of $\calS \times \calA$ with the interpretation: $Y_n = \{(s, a): (s, a)$ component of $Q$ was updated at time $n\}$. Let $\nu(n, s, a) \doteq \sum_{k=0}^n I\{(s, a) \in Y_k\}$, where $I$ is the indicator function. Thus $\nu(n, s, a) =$ the number of times the $(s, a)$ component was updated up to step $n$. The update rules of General Differential Q are 
\begin{align}
    Q_{n+1}(s, a) & \doteq Q_n(s, a) + \alpha_{\nu(n, s, a)} \delta_n(s, a) I\{(s, a) \in Y_n\}, \quad \forall s \in \calS, a \in \calA, \label{eq: async Q update}\\
    \bar R_{n+1} & \doteq \bar R_n + \eta \sum_{s, a} \alpha_{\nu(n, s, a)} \delta_n(s, a) I\{(s, a) \in Y_n\}, \label{Q: async bar R update}
\end{align}
where
\begin{align}
    \delta_n(s, a) \doteq R_n(s, a) - \bar R_n + \max_{a'} Q_n(S_n'(s, a), a') - Q_n(s, a) . \label{eq: Q: async TD error}
\end{align}
Here $\alpha_{\nu(n, s, a)}$ is the stepsize at step $n$ for state--action pair $(s, a)$. The quantity $\alpha_{\nu(n, s, a)}$ depends on the sequence $\{\alpha_n\}$, which is an algorithmic design choice, and also depends on the visitation of state--action pairs $\nu(n, s, a)$. To obtain the stepsize, the algorithm could maintain a $\vert \calS \times \calA \vert$-size table counting the number of visitations to each state--action pair, which is exactly $\nu(\cdot, \cdot, \cdot)$. Then the stepsize $\alpha_{\nu(n, s, a)}$ can be obtained as long as the sequence $\{\alpha_n\}$ is specified.

Now we show Differential Q-learning is a special case of General Differential Q. Consider a sequence of real experience $\ldots, S_t, A_t, R_{t+1}, S_{t+1}, \ldots$. By choosing step $n$ = time step $t$, 
\begin{align*}
    Y_t(s, a) &= 1 \text{, if } s = S_t, a = A_t,\\
    Y_t(s, a) &= 0 \text{ otherwise,}
\end{align*}
and $S'_n(S_t, A_t) = S_{t+1}, R_n(S_t, A_t) = R_{t+1}$, update rules \eqref{eq: async Q update}, \eqref{Q: async bar R update}, and \eqref{eq: Q: async TD error} become
\begin{align}
    Q_{t+1}(S_t, A_t) & \doteq Q_t (S_t, A_t) + \alpha_{\nu(t, S_t, A_t)} \delta_t \text{\ , and\ } Q_{t+1}(s, a) \doteq Q_t (s, a), \forall s \neq S_t, a \neq A_t, \\
    \bar R_{t+1} & \doteq \bar R_t + \eta \alpha_{\nu(t, S_t, A_t)} \delta_t,\\
    \delta_t & \doteq R_{t+1} - \bar R_t + \max_{a'} Q_t (S_{t+1}, a') - Q_t (S_t, A_t),
\end{align}
which are Differential Q-learning's update rules with stepsize at time $t$ being $\alpha_{\nu(t, S_t, A_t)}$.

Similarly we can show Differential Q-planning is a special case of General Differential Q. Consider a sequence of simulated experience $\ldots, \hat S_n, \hat A_n, \hat R_n, \hat S_n', \ldots$. By choosing step $n$ to be the planning step,
\begin{align*}
    Y_n(s, a) &= 1, \text{ if } s = \hat S_n, a = \hat A_n,\\
    Y_n(s, a) &= 0, \text{ otherwise,}
\end{align*}
and $S'_n(\hat S_n, \hat A_n) = \hat S_{n}', R_n(\hat S_n, \hat A_n) = \hat R_{n}$, update rules \eqref{eq: async Q update}, \eqref{Q: async bar R update}, and \eqref{eq: Q: async TD error} become
\begin{align}
    Q_{n+1}(\hat S_n, \hat A_n) & \doteq Q_n (\hat S_n, \hat A_n) + \alpha_{\nu(n, \hat S_n, \hat A_n)} \delta_n \text{,\ and\ } Q_{n+1}(s, a) \doteq Q_n (s, a), \forall s \neq \hat S_n, a \neq \hat A_n, \\
    \bar R_{n+1} & \doteq \bar R_n + \eta \alpha_{\nu(n, \hat S_n, \hat A_n)} \delta_n,\\
    \delta_n & \doteq \hat R_{n+1} - \bar R_n + \max_{a'} Q_n (\hat S_{n+1}, a') - Q_n (\hat S_n, \hat A_n),
\end{align}
which are Differential Q-planning's update rules with stepsize $\alpha_{\nu(n, \hat S_n, \hat A_n)}$.

We now specify assumptions on General Differential Q, which are required by our convergence theorem. 


\begin{assumption}[Communicating Assumption] \label{assu: communicating}
The MDP $\calM$ has a single communicating class, that is, each state in $\calM$ is accessible from every other state under some deterministic stationary policy.
\end{assumption}

\begin{assumption}[Action-Value Function Uniqueness] \label{assu: action-value function uniqueness}
There exists a unique solution of $q$ only up to a constant in \eqref{action-value Bellman equation}.
\end{assumption}

\begin{assumption}[Stepsize Assumption] \label{assu: stepsize} $\alpha_n > 0$, $\sum_{n = 0}^\infty \alpha_n = \infty$, $\sum_{n = 0}^\infty \alpha_n^2 < \infty$.
\end{assumption}

\begin{assumption}[Asynchronous Stepsize Assumption A] \label{assu: asynchronous stepsize 1}
Let $[\cdot]$ denote the integer part of $(\cdot)$, for $x \in (0, 1)$, 
\begin{align*}
    \sup_i \frac{\alpha_{[xi]}}{\alpha_i} < \infty
\end{align*}
and 
\begin{align*}
    \frac{\sum_{j=0}^{[yi]} \alpha_j}{\sum_{j=0}^i \alpha_j} \to 1
\end{align*} 
uniformly in $y \in [x, 1]$.
\end{assumption}

\begin{assumption}[Asynchronous Stepsize Assumption B] \label{assu: asynchronous stepsize 2}
There exists $\Delta > 0$ such that 
\begin{align*}
    \liminf_{n \to \infty} \frac{\nu(n, s, a)}{n+1} \geq \Delta,
\end{align*}
a.s., for all $s \in \calS, a \in \calA$.
Furthermore, for all $x > 0$, let 
\begin{align*}
    N(n, x) = \min \Bigg \{m > n: \sum_{i = n+1}^m \alpha_i \geq x \Bigg \},
\end{align*}
the limit 
\begin{align*}
    \lim_{n \to \infty} \frac{\sum_{i = \nu(n, s, a)}^{\nu(N(n, x), s, a)} \alpha_i}{\sum_{i = \nu(n, s', a')}^{\nu(N(n, x), s', a')} \alpha_i}
\end{align*}
exists a.s. for all $s, s', a, a'$.
\end{assumption}

We now explain the meanings of these assumptions. 

Assumption~\ref{assu: communicating} is the standard communicating assumption for the MDP. If this is not satisfied (i.e., there exist states from which it is impossible to get back to the others), no learning algorithm can be guaranteed to learn differential value function up to an additive constant for any policy in that MDP using a single stream of experience. The reward rate of a learned policy can still converge to the optimal reward rate under a slightly weaker \emph{weakly communicating} assumption, which assumes that the MDP has a single communicating class and some additional \emph{transient} states. Whenever a weakly communicating MDP starts from a transient state, eventually it will never visit that state again under any policy. 
The differential values of states in the communicating class can be learned well using some algorithms but that of transient states can not. Both the theorem and proof of convergence under the weakly communicating assumption would require distinguishing between these two class of states. For a simpler analysis, we use the communicating assumption here. 
In case of our control planning problem, transient states can appear in the simulated experience for an infinite number of times, and thus differential values of transient states can be learned accurately. Therefore the communicating assumption for the planning algorithm can be replaced by the more general weakly communicating assumption. However, to give a simple theorem and proof which cover both our learning and planning algorithms, we choose to present our result using the communicating MDP assumption.

\cref{assu: action-value function uniqueness} is required by average-reward learning and planning algorithms to guarantee convergence of estimates of $Q$ to a unique solution (up to a constant). 
A necessary and sufficient condition for \cref{assu: action-value function uniqueness} is provided by Schweitzer \& Federgruen (1978). The condition is that there exists a randomized stationary optimal policy that induces a single recurrent class of states $\mathcal{C}$ such that recurrent states induced by any randomized stationary optimal policy are members of $\mathcal{C}$.

Assumptions~\ref{assu: stepsize}, \ref{assu: asynchronous stepsize 1}, and \ref{assu: asynchronous stepsize 2} originate from the another result showing convergence of stochastic approximation algorithms (Borkar 1998) and were also required by the convergence theorem of RVI Q-learning.
Assumptions~\ref{assu: stepsize} and \ref{assu: asynchronous stepsize 1} can be satisfied if the sequence $\{\alpha_n\}$ decreases to 0 appropriately. The sequence $\{\alpha_n\}$ could be, for example, $1/ n$, $1 / (n \log n)$, or $\log n / n$ (Abounadi, Bertsekas, and Borkar 2001). 
The first part of Assumption~\ref{assu: asynchronous stepsize 2} requires that, for each state--action pair, the limiting ratio of the number of visitations to the pair and the number of visitations to all pairs is greater than or equal to any fixed positive number. The second part of the assumption requires that the relative update frequency between any two elements is finite. For example, Borkar (personal communication; see also page 403 by Bertsekas and Tsitsiklis 1996) showed that with a common $\alpha_n = 1/n$, Assumption \ref{assu: asynchronous stepsize 2} can be satisfied (Assumption \ref{assu: stepsize} and \ref{assu: asynchronous stepsize 1} can also be satisfied with this stepsize). 

It is easy to verify that under the communicating assumption the following system of equations:
\begin{align}
    q(s, a) & = \sum_{s', r} p(s', r \mid s, a) (r - \bar r + \max_{a'} q(s', a')) \text{\ , for all } s \in \calS, a \in \calA, \quad \label{eq: action-value Bellman equation 2}\\
    r_* - \bar R_0 & = \eta \left(\sum q - \sum Q_0 \right),  \label{eq: determination equation for q infty}
\end{align}
has a unique solution for $q$. Denote the solution as $q_\infty$.

\begin{theorem}[Convergence of General Differential Q]\label{eq: convergence of uncentered differential q update}
If \Cref{assu: communicating}-\Cref{assu: asynchronous stepsize 2} hold, then the General Differential Q algorithm (Equations~\ref{eq: async Q update}-\ref{eq: Q: async TD error}) converges a.s. $\bar R_n$ to $r_*$, $Q_n$ to $q_\infty$, and $r(\pi_t)$ to $r_*$ where $\pi_t$ is any greedy policy w.r.t. $Q_t$.
\end{theorem}

We now prove this theorem.

\subsubsection{Proof of Theorem \ref{eq: convergence of uncentered differential q update}}

At each step, the increment to $\bar R_n$ is $\eta$ times the increment to $Q_n$ and $\sum Q_n$. Therefore, the cumulative increment can be written

\begin{align}
    \bar R_n - \bar R_0 &= \eta  \sum_{i = 0}^{n-1} \sum_{s, a} \alpha_{\nu(i, s, a)} \delta_i(s, a) I\{(s, a) \in Y_i\} \nonumber  \\
    & = \eta \left(\sum Q_n - \sum Q_0 \right) \nonumber \\
    \implies \bar R_n &= \eta \sum Q_n - \eta \sum Q_0 + \bar R_0 = \eta \sum Q_n - c \label{Q: relation between bar R and Q} ,\\
    \text{ where } c &\doteq \eta \sum Q_0 - \bar R_0. \label{eq: definition of c}
\end{align}

Now substituting $\bar R_n$ in \eqref{eq: async Q update} with \eqref{Q: relation between bar R and Q}, we have $\forall s \in \calS, a \in \calA$: 

\begin{align}
    & Q_{n+1}(s, a) = Q_{n}(s, a) + \alpha_{\nu(n, s, a)} \left(R_n(s, a) + \max_{a'} Q_n(S_n'(s, a), a') - Q_n(s, a) - \eta \sum Q_n + c \right) I\{(s, a) \in Y_n\} \nonumber\\
    & = Q_{n}(s, a) + \alpha_{\nu(n, s, a)} \left(\tilde R_n(s, a) + \max_{a'} Q_n(S'_n(s, a), a') - Q_n(s, a) - \eta \sum Q_n \right) I\{(s, a) \in Y_n\}, \label{Q: transformed async single update}
\end{align}

where $\tilde R_n(s, a) \doteq R_n(s, a) + c$. 

\eqref{Q: transformed async single update} is in the same form as the RVI Q-learning's update (Equation 2.7 by Abounadi, Bertsekas, and Borkar (2001), see also \eqref{RVI Q-learning}) with $f(Q_n) = \eta \sum Q_n$, for a MDP $\tilde \calM$ whose rewards are all shifted by $c$ from the original MDP $\calM$. 

This transformed MDP has the same state and action space as the original MDP and has the transition probability defined as
\begin{align}
    \tilde p(s', r + c \mid s, a) \doteq p(s', r \mid s, a). \label{Q: relation between p and tilde p}
\end{align}
In other words, $\tilde \calM \doteq (\calS, \calA, \calR, \tilde p)$.

Note that the communicating assumption we made for the original MDP is still valid for the transformed MDP. For this transformed MDP, denote the best possible average reward rate as $\tilde r_*$.  Then
\begin{align}
    \tilde r_* = r_* + c \label{eq: relation between r star and tilde r star}
\end{align}
because the reward in the transformed MDP is shifted by $c$ compared with the original MDP. Combining \eqref{eq: relation between r star and tilde r star}, \eqref{eq: determination equation for q infty}, and \eqref{eq: definition of c}, we have
\begin{align}
    \tilde r_* = \eta \sum q_\infty \; . \label{eq: relation between tilde r and eta sum q infty}
\end{align}

Furthermore, because
\begin{align}
    q_\infty(s, a)  & = \sum_{s', r} p(s', r \mid s, a) (r + \max_{a'} q_\infty (s', a') - r_*) \quad \text{(from \eqref{eq: action-value Bellman equation 2})} \nonumber\\
    & = \sum_{s', r} p(s', r \mid s, a) (r + c + \max_{a'} q_\infty (s', a') - \tilde r_*) \quad \text{(from \eqref{eq: relation between r star and tilde r star})} \nonumber\\
    & = \sum_{s', r} \tilde p(s', r \mid s, a) (r + \max_{a'} q_\infty (s', a') - \tilde r_*) \quad \text{(from \eqref{Q: relation between p and tilde p})} \label{Q: action-value optimality equations for tilde p},
\end{align}
$q_\infty$ is a solution of $q$ in the action-value Bellman equations for not only the original MDP $\calM$ but also the transformed MDP $\tilde \calM$.

If the convergence theorem of the RVI Q-learning applies, then $Q_n \to q_\infty$ and $\eta \sum Q_n \to \tilde r_*$. However, in general, $f(x) \doteq \eta \sum x$ does not satisfy some requirements on $f$ by Abounadi, Bertsekas, and Borkar (2001). In particular,
\begin{align}
    f(e) = 1 \text{\ , and \ } f(x + c e) = f(x) + c, \forall x \in \mathbb{R}^{\abs{\calS \times \calA}} \label{RVI Q-learning: old conditions on f}
\end{align}
in Assumption 2.2 (Abounadi, Bertsekas, and Borkar 2001) are violated. In the next section, \Cref{Extension of RVI Q-learning}, we extend the RVI Q-learning family of algorithms by replacing \eqref{RVI Q-learning: old conditions on f} with the following weaker assumptions:
\begin{align}
    \exists\ u > 0 \text{\ s.t. \ } f(e) = u \text{\ , and \ } f(x + ce) = f(x) + cu, \forall x \in \mathbb{R}^{\abs{\calS \times \calA}} \label{RVI Q-learning: new conditions on f}.
\end{align}
It can be seen that \eqref{RVI Q-learning: old conditions on f} is a special case of \eqref{RVI Q-learning: new conditions on f} when $u = 1$. Therefore RVI Q-learning family is a subset of the extended RVI Q-learning family.


Because $f(x) = \eta \sum x$ satisfies assumptions on $f$ required by \Cref{Extension of RVI Q-learning} and \Cref{assu: communicating}-\Cref{assu: asynchronous stepsize 2} also hold for the transformed MDP $\tilde \calM$, \eqref{Q: transformed async single update} converges a.s. $Q_n$ to $q_\infty$, which is the solution of
\begin{align*}
    q(s, a) & = \sum_{s', r} \tilde p(s', r \mid s, a) (r - \bar r + \max_{a'} q(s', a')) \text{\ , for all } s \in \calS, a \in \calA\\
    \eta \sum q &= \tilde r_*
\end{align*}
by \eqref{Q: action-value optimality equations for tilde p} and \eqref{eq: relation between tilde r and eta sum q infty}.

Now consider $\bar R_n$. Combining \eqref{Q: relation between bar R and Q} and $Q_n \to q_\infty$, we have $\bar R_n \to \eta \sum q_\infty - c$. In addition, because $\eta \sum q_\infty = \tilde r_*$ (Equation \ref{eq: relation between tilde r and eta sum q infty}), we have $\bar R_n \to \tilde r_* - c$. Because $\tilde r_* = r_* + c$ (Equation \ref{eq: relation between r star and tilde r star}), we have 
\begin{align} \label{eq: Q: R_n converges to r_*}
    \bar R_n \to r_* \text{\ \  a.s. as \ \  } n \to \infty.
\end{align}

Finally consider $r(\pi_t)$ where $\pi_t$ is a greedy policy w.r.t. $Q_t$. 
From Theorem 8.5.5 by Puterman (1994), we have,
\begin{align}
    \min_{s, a} (TQ_t(s, a) - Q_t(s, a)) \leq r(\pi_t) \leq r_* \leq \max_{s, a} (TQ_t(s, a) - Q_t(s, a)) \\
    \implies
    \abs{r_* - r(\pi_t)} \leq sp(TQ_t - Q_t)
\end{align}
where $TQ(s, a) \doteq \sum_{s', r} \tilde p(s', r \mid s, a) (r + \max_{a'} Q(s', a'))$.
Because $Q_t \to q_\infty$ a.s., and $sp(TQ_t - Q_t)$ is a continuous function of $Q_t$, by continuous mapping theorem, $sp(TQ_t - Q_t) \to sp(Tq_\infty - q_\infty) = 0$ a.s. Therefore we conclude that $r(\pi_t) \to r_*$.

\Cref{eq: convergence of uncentered differential q update} is proved.

\begin{theorem}[Convergence of the Extended RVI Q-learning] \label{Extension of RVI Q-learning}
For any $Q_0 \in \mathbb{R}^{\abs{\calS \times \calA}}$, let $R_n, Y_n, \alpha_{\nu(n, s, a)}$ be defined as aforementioned, consider an update rule
\begin{align}
    Q_{n+1}(s, a) = Q_{n} (s, a) + \alpha_{\nu(n, s, a)} \left(R_n(s, a) + \max_{a'} Q_n(S_{n}'(s,a), \cdot) - Q_n(s, a) - f(Q_n)\right) I\{(s, a) \in Y_n\}, \label{eq: async RVI Q-learning}
\end{align}
if
\begin{enumerate}
    \item \Cref{assu: communicating}-\Cref{assu: asynchronous stepsize 2} hold,
    \item $f: \mathbb{R}^{\abs{\calS \times \calA}} \to \mathbb{R}$ is Lipschitz and there exists some $u > 0$ such that $\forall c \in \mathbb{R}$ and $x \in \mathbb{R}^{\abs{\calS \times \calA}}$, $f(e) = u$, $f(x + ce) = f(x) + cu$ and $f(cx) = cf(x)$,
\end{enumerate}
then $Q_n$ converges a.s. to $q_*$, where $q_*$ is the solution to action-value optimality equation (\Cref{eq: action-value Bellman equation 2}) satisfying $f(q_*) = r_*$.
\end{theorem}

If we set $u = 1$ in the above theorem, then we recover the convergence result of RVI Q-learning. 

The rest part of this section proves the above theorem. We use arguments similar to those of RVI Q-learning.

First, note that \eqref{eq: async RVI Q-learning} is in the same form as the asynchronous update (Equation 7.1.2) by Borkar (2009). We apply the result in Section 7.4 of the same text (Borkar 2009) (see also Theorem 3.2 by Borkar (1998)), which shows convergence for Equation 7.1.2, to show convergence of \eqref{Q: transformed async single update}. This result, given Assumption~\ref{assu: asynchronous stepsize 1}, \ref{assu: asynchronous stepsize 2}, only requires showing the convergence of the following \emph{synchronous} version of \eqref{eq: async RVI Q-learning}:
\begin{align}
    & Q_{n+1}(s, a) = Q_{n}(s, a) + \alpha_n \left( R_n(s, a) + g( Q_n(S'_n(s, a), \cdot)) - Q_n(s, a) - f( Q_n) \right) \text{\ , for all } s \in \calS, a \in \calA, \label{sync RVI Q-learning }
\end{align}

Like the proof of RVI Q-learning, first define operators $T, T_1, T_2$:
\begin{align*}
    T (Q) (s, a) & \doteq \sum_{s', r} p (s', r \mid s, a) (r + g( Q(s', \cdot))), \\
    T_1 (Q) & \doteq T (Q) - r_* e, \\
    T_2 (Q) & \doteq T (Q) - f(Q) e = T_1(Q) + \left (r_* - f(Q) \right) e.
\end{align*}
Consider two ordinary differential equations (ODEs):
\begin{align}
    \dot y_t & = T_1 (y_t) - y_t, \label{aux ode}\\
    \dot x_t & = T_2 (x_t) - x_t. \label{original ode}
\end{align}
Note that by the properties of $T_1$ and $T_2$, both \eqref{aux ode} and \eqref{original ode} have Lipschitz r.h.s.'s and thus are well-posed.

The next two lemmas are the same as Lemma 3.1 and Lemma 3.2 by Abounadi, Bertsekas, and Borkar (2001). Their proofs do not rely on properties of $f$ and therefore they hold with our more general $f$ function.

\begin{lemma}\label{lemma: aux ode convergence}
Let $\bar y$ be an equilibrium point of the ODE defined in \eqref{aux ode}. Then $\norm{y_t - \bar y}_\infty$ is nonincreasing, and $y_t \to y_*$ for some equilibrium point $y_*$ of \eqref{aux ode} that may depend on $y_0$.
\end{lemma}

\begin{lemma} \label{lemma: unique equilibrium}
\eqref{original ode} has a unique equilibrium at $q_*$.
\end{lemma}

We then show the relation between $x_t$ and $y_t$ using the following lemma. It shows that the difference between $x_t$ and $y_t$ is a vector with identical elements and this vector satisfies a new ODE.

\begin{lemma}\label{lemma: connection between original and aux ode}
Let $x_0 = y_0$, then $x_t= y_t + z_t e$, where $z_t$ satisfies the ODE $\dot z_t= - u z_t + (r_* - f(y_t))$.
\end{lemma}

\begin{proof}
The proof of $x_t = y_t + z_t e$ is the same with the Lemma 3.3 by Abounadi, Bertsekas, and Borkar (2001). 

Now we show $\dot z_t= - u z_t + (r_* - f(y_t) )$. Note that $f(x_t) = f(y_t + z_t e) = f(y_t) + u z_t$. In addition, $T_1(x_t) - T_1(y_t) = T_1(y_t + z_t e) - T_1(y_t) = T_1(y_t) + z_t e - T_1(y_t) = z_t e$, therefore we have, for $z_t \in \bbR$:
\begin{align*}
    \dot z_t e & = \dot x_t - \dot y_t\\
    & = \left(T_1 (x_t) - x_t + \left(r_* - f(x_t) \right) e \right) - ( T_1 (y_t) - y_t) \quad \text{(from \eqref{aux ode} and \eqref{original ode})}\\
    & = - (x_t - y_t) + (T_1 (x_t) - T_1(y_t)) + \left(r_* - f(x_t) \right) e \quad \\
    & = - z_t e + z_t e + \left(r_* - f(x_t) \right) e\\
    & = - u z_t e + u z_t e + \left (r_* - f(x_t) \right) e\\
    & = - u z_t e + \left (r_* - f(y_t) \right) e\\
    \implies \dot z_t &= - u z_t + \left(r_* - f(y_t) \right).
\end{align*}
\end{proof}

With the above lemmas, we have:

\begin{lemma} \label{lemma: globally asymptotically stable equilibrium lemma}
$q_*$ is the globally asymptotically stable equilibrium for \eqref{original ode}.
\end{lemma}

\begin{proof}
We have shown that $q_*$ is the unique equilibrium in Lemma \ref{lemma: unique equilibrium}.

With that result, we first prove Lyapunov stability. That is, we need to show that given any $\epsilon > 0$, we can find a $\delta > 0$ such that $\norm{q_* - x_0}_\infty \leq \delta$ implies $\norm{q_* - x_t}_\infty \leq \epsilon$ for $t \geq 0$. 

First, from Lemma \ref{lemma: connection between original and aux ode} we have $\dot z_t= - u z_t + (r_* - f(y_t))$. By variation of parameters and $z_0 = 0$, we have
\begin{align*}
    z_t = \int_0^t \exp(u (\tau - t)) \left (r_* - f(y_\tau) \right) d\tau .
\end{align*}
Then
\begin{align}
    \norm{q_* - x_t}_\infty & = \norm{q_* - y_t - z_t ue}_\infty \nonumber\\ 
    & \leq \norm{q_* - y_t}_\infty + u \abs{z_t} \nonumber \\
    & \leq \norm{q_* - y_0}_\infty + u \int_0^t \exp(u (\tau - t)) \abs{r_* - f(y_\tau)} d\tau \nonumber\\
    & = \norm{q_* - x_0}_\infty + u \int_0^t \exp(u (\tau - t)) \abs{f(q_*) - f(y_\tau)} d\tau \quad \text{(from \eqref{eq: relation between tilde r and eta sum q infty})} \label{Q: lemma: globally asymptotically stable equilibrium lemma: eq 1}.
\end{align}

Because $f$ is $L$-lipschitz, we have
\begin{align*}
    \abs{ f(q_\infty) - f( y_\tau)} & \leq L \norm{q_* - y_\tau}_\infty \\
    & \leq L \norm{q_* - y_0}_\infty \quad \text{(from Lemma \ref{lemma: aux ode convergence})} \\
    & = L \norm{q_* - x_0}_\infty,
\end{align*}

\begin{align*}
    \int_0^t \exp(u (\tau - t)) \abs{ f(q_*) - f(y_\tau) } d\tau & \leq \int_0^t \exp(u (\tau - t)) L \norm{q_* - x_0 }_\infty d\tau \\
    & = L \norm{q_* - x_0 }_\infty \int_0^t \exp(u (\tau - t)) d\tau \\
    & = L \norm{q_* - x_0 }_\infty \frac{1}{u }(1 - \exp(-u t)) \\
    & = \frac{L}{u }\norm{q_* - x_0 }_\infty (1 - \exp(-u t))
\end{align*}

Substituting the above equation in \eqref{Q: lemma: globally asymptotically stable equilibrium lemma: eq 1}, we have
\begin{align*}
    \norm{q_* - x_t}_\infty \leq (1 + L) \norm{q_* - x_0}_\infty.
\end{align*}

Lyapunov stability follows.

Now in order to prove the asymptotic stability, in addition to Lyapunov stability, we need to show that there exists $\delta >0$ such that if $\norm{x_0 - q_*}_\infty < \delta$ , then $\lim_{t \to \infty } \norm{x_t - q_*}_\infty=0$. Note that

\begin{align*}
    \lim_{t \to \infty} z_t & = \lim_{t \to \infty} \int_0^t \exp(u (\tau - t)) \left(r_* - f(y_\tau) \right) d\tau\\
    & = \lim_{t \to \infty} \frac{\int_0^t \exp(u \tau) ( r_* - f(y_\tau) ) d\tau}{\exp(ut)} \\
    & = \lim_{t \to \infty} \frac{\exp(u t) (r_* - f(y_t))}{u \exp(ut)} \quad \text{(by L'Hospital's rule)}\\
    & = \frac{r_* - f(y_*)}{u} \quad \text{(by Lemma \ref{lemma: aux ode convergence})} .
\end{align*}

Because $x_t = y_t + z_t e$ (Lemma \ref{lemma: connection between original and aux ode}) and $y_t \to y_*$ (Lemma \ref{lemma: aux ode convergence}), we have $x_t \to y_* + (r_* - f(y_*)) e / u$, which must coincide with $q_*$ because that is the only equilibrium point for \eqref{original ode} (Lemma \ref{lemma: unique equilibrium}). Therefore $\lim_{t \to \infty} \norm{x_t - q_*}_\infty = 0$ for any $x_0$. Asymptotic stability is shown and the proof is complete.
\end{proof}

\begin{lemma}\label{lemma: Synchronous General Differential Q}
Equation \ref{sync RVI Q-learning } converges a.s. $Q_{n}$ to $q_*$ 
as $n \to \infty$.
\end{lemma}

\begin{proof}
The proof uses Theorem 2 in Section 2 of Borkar (2009) and is essentially the same as Lemma 3.8 by Abounadi, Bertsekas and Borkar (2001). For completeness, we repeat the proof (with more details) here.

First write the synchronous update \eqref{sync RVI Q-learning } as 
\begin{align*}
    Q_{n+1} = Q_{n} + \alpha_n (h(Q_n) + M_{n+1})
\end{align*}
where 
\begin{align*}
    h(Q_n)(s, a) & \doteq \sum_{s', r} p (s', r \mid s, a) (r + \max_{a'} Q_n(s', a')) - Q_n(s, a) - f(Q_n) \\
    & = T(Q_n)(s, a) - Q_n(s, a) - f(Q_n) \\
    & = T_2(Q_n)(s, a) - Q_n(s, a) ,\\
    M_{n + 1}(s, a) & \doteq R_n(s, a) + \max_{a'} Q_n(S_n'(s, a), a') - T(Q_n)(s, a) .
\end{align*}
Theorem 2 requires verifying the following conditions and concludes that $Q_n$ converges
to a (possibly sample path dependent) compact connected internally chain transitive invariant set of ODE $\dot x_t = h(x_t)$. This is exactly the ODE defined in \eqref{original ode}. Lemma \ref{lemma: unique equilibrium} and \ref{lemma: globally asymptotically stable equilibrium lemma} conclude that this ODE has $q_\infty$ as the unique globally asymptotically stable equilibrium. Therefore the (possibly sample path dependent) compact connected internally chain transitive invariant set is a singleton set containing only the unique globally asymptotically stable equilibrium. Thus Theorem 2 concludes that $Q_n \to q_\infty$ a.s. as $n \to \infty$. We now list conditions required by Theorem 2:

\begin{itemize}
    \item \textbf{(A1)} The function $h$ is Lipschitz: $\norm{h(x) - h(y)} \leq L \norm{x - y}$ for some $0 < L < \infty$.
    \item \textbf{(A2)} The sequence $\{ \alpha_n\}$ satisfies $ \alpha_n > 0$, and $\sum \alpha_n = \infty$, $\sum \alpha_n^2 < \infty$.
    \item \textbf{(A3)} $\{M_n\}$ is a martingale difference sequence with respect to the increasing family of $\sigma$-fields 
        \begin{align*}
            \calF_n \doteq \sigma(Q_i, M_i,i \leq n), n \geq 0
        \end{align*}
        That is
        \begin{align*}
            \bbE[M_{n+1} \mid \calF_n] = 0 \text{ \ \ a.s., } n \geq 0.
        \end{align*}
        Furthermore, $\{M_n\}$ are square-integrable 
        \begin{align*}
            \bbE[\norm{M_{n+1}}^2 \mid \calF_n] \leq K(1 + \norm{Q_n}^2) \text{ \ \ a.s., \ \ } n \geq 0,
        \end{align*}
        for some constant $K > 0$.
    \item \textbf{(A4)} $\sup_n \norm{Q_n} \leq \infty$ a.s..
\end{itemize}

Let us verify these conditions now.

(A1) is satisfied as both $T$ and $\sum$ operators are Lipschitz.

(A2) is satisfied by \Cref{assu: stepsize}.

(A3) is also satisfied because for any $s \in \calS, a \in \calA$
\begin{align*}
    \bbE[M_{n+1}(s, a) \mid \calF_n] & = \bbE \left[R_n(s, a) + \max_{a'}Q_n(S_n'(s, a), a') - T (Q_n)(s, a) \mid \calF_n \right] \\
    & = \bbE \left [R_n(s, a) + \max_{a'}Q_n(S_n'(s, a), a') \mid \calF_n \right] - T (Q_n)(s, a) \\
    & = 0
\end{align*}
and $\bbE [\norm{M_{n+1}}^2 \mid \calF_n] \leq K (1 + \norm{Q_n}^2)$ for a suitable constant $K > 0$ can be verified by a simple application of triangle inequality.

To verify (A4), we apply Theorem 7 in Section 3 by Borkar (2009), which shows $\sup_n \norm{Q_n} \leq \infty$ a.s., if (A1), (A2), and (A3) are all satisfied and in addition we have the following condition satisfied:

\textbf{(A5)} The functions $h_d(x) \doteq h(dx)/d$, $d \geq 1, x \in \bbR^{k}$, satisfy $h_d(x) \to h_\infty(x)$ as
$d \to \infty$, uniformly on compacts for some $h_\infty \in C(\bbR^k)$. Furthermore, the ODE $\dot x_t = h_\infty(x_t)$ has the origin as its unique globally asymptotically stable equilibrium.

Note that
\begin{align*}
    & h_\infty(x) = \lim_{d \to \infty} h_d(x) = \lim_{d \to \infty} \left (T(dx) - dx - f( dx) e \right) / d \\
    & = T_0 (x) - x - f(x)  e
\end{align*}
where 
\begin{align*}
    T_0 (x) \doteq \sum_{s', r} p (s', r \mid  s, a) \max_{a'} x(s', a').
\end{align*}
The function $h_\infty$ is clearly continuous in every $x \in \bbR^k$ and therefore $h_\infty \in C(\bbR^k)$.

Now consider the ODE $\dot x_t = h_\infty(x_t) = T_0 (x_t) - x_t - f(x_t) e$. Clearly the origin is an equilibrium. This ODE is a special case of \eqref{original ode}, corresponding to the reward being always zero, therefore Lemma \ref{lemma: unique equilibrium} and \ref{lemma: globally asymptotically stable equilibrium  lemma} also apply to this ODE and the origin is the unique globally asymptotically stable equilibrium.

(A1), (A2), (A3), (A4) are all verified and therefore 
\begin{align}
    Q_n \to q_* \text{\ \  a.s. as \ \  } n \to \infty. \label{Q: convergence of Q_n to q_infty}
\end{align}
\end{proof}

\subsection{Proof of Differential TD-learning and Differential TD-planning} \label{app:convergence-proof-difftd}
The proof is similar to that of Differential Q-learning and Differential Q-planning. We consider a General Differential TD algorithm which includes both Differential TD-learning and Differential TD-planning. 

Given a MDP $\calM \doteq (\calS, \calA, \calR, p)$, a behavior policy $b$, and a target policy $\pi$, for any state $s \in \calS$ and discrete step $n \geq 0$, let $A_n(s) \sim b(\cdot \mid s)$, $R_n(s, A_n(s)), S'_n(s, A_n(s)) \sim p(\cdot, \cdot \mid s, A_n(s))$. We hypothesize a set-valued process $\{Y_n\}$ taking
values in the set of nonempty subsets of $\calS$ with the interpretation: $Y_n = \{s: s$ component of $V$ was updated at time $n\}$. Define $\nu(n, s) = \sum_{i=0}^n I\{s \in Y_i\}$ where $I$ is the indicator function. Thus $\nu(n, s) =$ the number of times $V(s)$ was updated up to time $n$. Then the update rules of General Differential TD are, for $n \geq 0$:
\begin{align}
    V_{n+1}(s) & \doteq V_n(s) + \alpha_{\nu(n, s)} \rho_n(s) \delta_n(s) I\{s \in Y_n\} \quad \forall s \in \calS \label{eq: TD: async V update}\\
    \bar R_{n+1} & \doteq \bar R_n + \sum_s \alpha_{\nu(n, s)} \rho_n(s) \delta_n(s) I\{s \in Y_n\} \label{TD: async bar R update},
\end{align}
where
\begin{align}
    \delta_n(s) & \doteq R_n(s, A_n(s)) + V_n(S_n'(s, A_n(s))) - V_n(s) - \bar R_n \label{eq: TD: async TD error},
\end{align}
and $\rho_n(s) \doteq \pi(A_n(s) \mid s) / b(A_n(s) \mid s)$ is the importance sampling ratio (this is always well-defined given Assumption~\ref{assu: coverage}).

The quantity $\alpha_{\nu(n, s)}$ is the stepsize at step $n$ for state $s$ and can be obtained the same way as introduced in \ref{app:convergence-proof-diffq}. It can be shown, using similar arguments as those in \ref{app:convergence-proof-diffq}, that Differential TD-learning and Differential TD-planning are special cases of General Differential TD. And therefore we only need to prove the convergence of General Differential TD. We now specify required assumptions for the convergence proof.

\begin{assumption}\label{assu: unichain}
The Markov chain induced by the target policy is unichain.
\end{assumption}

\begin{assumption}[Coverage Assumption] \label{assu: coverage}
$b(a \mid s) > 0$ if $\pi(a \mid s) > 0$ for all $s \in \calS$, $a \in \calA$. 
\end{assumption}

The above assumption requires that the behavior policy covers all possible state--action pairs the target policy may incur. To guarantee the full coverage, we will need that the behavior policy visit all states for an infinite number of times.

\begin{assumption}[Asynchronous Stepsize Assumption B] \label{assu: asynchronous stepsize td 2}
There exists $\Delta > 0$ such that 
\begin{align*}
    \liminf_{n \to \infty} \frac{\nu(n, s)}{n+1} \geq \Delta ,
\end{align*}
a.s., for all $s \in \calS$.
Furthermore, for all $x > 0$, and
\begin{align*}
    N(n, x) = \min \Bigg \{m \geq n: \sum_{i = n+1}^m \alpha_i \geq x \Bigg\} ,
\end{align*}
the limit 
\begin{align*}
    \lim_{n \to \infty} \frac{\sum_{i = \nu(n, s)}^{\nu(N(n, x), s)} \alpha_i}{\sum_{i = \nu(n, s')}^{\nu(N(n, x), s')} \alpha_i}
\end{align*} 
exists a.s. for all $s, s'$.
\end{assumption}

It can be easily verified that
\begin{align}
    v(s) & = \sum_a \pi(a \mid s) \sum_{s', r} p(s', r \mid s, a) (r - \bar r + v(s')), \text{    for all } s\in\calS, \label{state value Bellman equation 2} \\
    r(\pi) - \bar R_0 & = \eta \left(\sum v - \sum V_0 \right) \label{eq: TD: determination equation for v infty}
\end{align}
has a unique solution of $v$. Denote the solution as $v_\infty$.

\begin{theorem}[Convergence of General Differential TD]\label{thm: convergence of uncentered differential td update}

If Assumptions~\ref{assu: unichain}, \ref{assu: stepsize}, \ref{assu: asynchronous stepsize 1}, \ref{assu: asynchronous stepsize td 2}, and \ref{assu: coverage} hold, then General Differential TD (Equations \ref{eq: TD: async V update}-\ref{eq: TD: async TD error}) converges a.s., $\bar R_n$ to $r(\pi)$ and $V_n$ to $v_\infty$.
\end{theorem}

We now prove this theorem.

\subsubsection{Proof of Theorem \ref{thm: convergence of uncentered differential td update}}
Similar as what we did in the proof of General Differential Q, we can combine update rules \eqref{eq: TD: async V update}-\eqref{eq: TD: async TD error} to obtain a single update rule.
\begin{align}
    & \bar R_n - \bar R_0 \nonumber\\
    & = \eta \sum_{i = 0}^{n-1} \sum_{s} \alpha_{\nu(i, s)} \rho_i(s) \delta_i (s) I\{s \in Y_k\} \nonumber\\
    & = \eta \left (\sum V_{n} - \sum V_0 \right) \nonumber\\
    & \implies \nonumber \\
    & \bar R_n = \eta \sum V_n - \eta \sum V_0 + \bar R_0 = \eta \sum V_n - c, \label{TD: relation between bar R and V} \\
    & \text{ where } c \doteq \eta \sum V_0 - \bar R_0. \label{eq: TD: definition of c}
\end{align}
Substituting $\bar R_n$ in \eqref{eq: TD: async V update} with \eqref{TD: relation between bar R and V} we have, $\forall s \in \calS$: 
\begin{align}
    V_{n+1}(s) & = V_{n}(s) + \alpha_{\nu(n, s)} \rho_n(s) \left(R_n(s, A_n(s)) + V_n(S_n'(s, A_n(s))) - V_n(s) - \eta \sum V_n + c \right) I\{s \in Y_n\} \nonumber\\
    & = V_{n}(s) + \alpha_{\nu(n, s)} \rho_n(s) \left(\tilde R_n(s, A_n(s)) + V_n(S_n'(s, A_n(s))) - V_n(s) - \eta \sum V_n \right) I\{s \in Y_n\}  \label{eq: TD: transformed async single update} ,
\end{align}
where $\tilde R_n(s, A_n(s)) \doteq R_n(s, A_n(s)) + c$. Now \eqref{eq: TD: transformed async single update} is in the same form with the asynchronous update (Equation 7.1.2) studied by Borkar (2009). Again we can apply the result in Section 7.4 by Borkar (2009) to show convergence of \eqref{eq: TD: transformed async single update}. This result, given Assumption \ref{assu: asynchronous stepsize 1} and \ref{assu: asynchronous stepsize td 2}, only requires showing the convergence of the following \emph{synchronous} version of General Differential TD:
\begin{align}
    V_{n+1}(s) = V_{n}(s) + \alpha_n \rho_n(s) \left(\tilde R_n(s, A_n(s)) + V_n(S'_n(s, A_n(s))) - V_n(s) - \eta \sum V_n \right), \quad \forall s \in \calS . \label{eq: TD: transformed sync single update}
\end{align}
This transformed MDP has the same state and action space as the original MDP and has the transition probability defined as 
\begin{align}
    \tilde p(s', r + c \mid s, a) \doteq p(s', r \mid s, a). \label{TD: relation between p and tilde p}
\end{align}
Note that the unichain assumption (Assumption \ref{assu: communicating}) and the coverage assumption (Assumption \ref{assu: coverage}) we made for the original MDP is still valid for the transformed MDP. For this transformed MDP, denote the average reward rate following policy $\pi$ as $\tilde r(\pi)$.  Then 
\begin{align}
    \tilde r(\pi) = r(\pi) + c \label{eq: TD: relation between r pi and tilde r pi}
\end{align}
because the reward in the transformed MDP is shifted by $c$ compared with the original MDP.

Combining \eqref{eq: TD: relation between r pi and tilde r pi}, \eqref{eq: TD: determination equation for v infty} and \eqref{eq: TD: definition of c}, we have
\begin{align}
    \tilde r(\pi) = \eta \sum v_\infty \label{TD: relation between tilde r and eta sum v infty}.
\end{align}
Furthermore, 
\begin{align*}
    v_\infty(s) & = \sum_a \pi(a \mid s) \sum_{s', r} p(s', r \mid s, a) (r + v_\infty (s') - r(\pi)) \quad \text{(from \eqref{state value Bellman equation 2})}\\
    & = \sum_a \pi(a \mid s) \sum_{s', r} p(s', r \mid s, a) (r + c + v_\infty (s') - \tilde r(\pi)) \quad \text{(from \eqref{eq: TD: relation between r pi and tilde r pi})}\\
    & = \sum_a \pi(a \mid s) \sum_{s', r} \tilde p(s', r \mid s, a) (r + v_\infty (s') - \tilde r(\pi)),
\end{align*}
therefore $v_\infty$ is a solution of $v$ in the state-value Bellman equations for not only the original MDP $\calM$ but also the transformed MDP $\tilde \calM$. 

We now show $V_n \to v_\infty$ and $\eta \sum V_n \to \tilde r(\pi)$. First, define operators $T, T_1, T_2$:
\begin{align*}
    T(V) (s) & \doteq \sum_a \pi(a \mid s) \sum_{s', r} \tilde p (s', r \mid s, a) (r + V(s')) , \\
    T_1 (V) & \doteq T (V) - \tilde r(\pi) e, \\
    T_2 (V) & \doteq T (V) - \left( \eta \sum V \right) e = T_1 (V) + \left(\tilde r(\pi) - \eta \sum V \right) e.
\end{align*}

Consider two ODEs:
\begin{align}
    \dot y_t & = T_1 (y_t) - y_t, \label{eq: TD: aux ode}\\
    \dot x_t & = T_2 (x_t) - x_t. \label{eq: TD: original ode}
\end{align}
Note that by the properties of $T, T_1, T_2$, both \eqref{eq: TD: aux ode} and \eqref{eq: TD: original ode} have Lipschitz R.H.S.'s and thus are well-posed.

The next lemma is similar to Lemma 3.1 by Abounadi, Bertsekas, and Borkar (2001) and is a special case of Theorem 3.1 and Lemma 3.2 by Borkar and Soumyanath (1997).

\begin{lemma}
Let $\bar y$ be an equilibrium point of \eqref{eq: TD: aux ode}. Then $\norm{y_t - \bar y}_\infty$ is nonincreasing, and $y_t \to y_*$ for some equilibrium point $y_*$ of \eqref{eq: TD: aux ode} that may depend on $y_0$.
\end{lemma}
The next lemma is similar to Lemma \ref{lemma: unique equilibrium} and the proof of it is almost the same as the proof of Lemma \ref{lemma: unique equilibrium}. The only changes are to replace $\tilde r_*$, $q_\infty$ and $q$ with $\tilde r(\pi)$, $v_\infty$ and $v$ respectively.
\begin{lemma} \label{lemma: TD: unique equilibrium}
\eqref{eq: TD: original ode} has a unique equilibrium at $v_\infty$.
\end{lemma}
The next two lemmas are almost the same as Lemma \ref{lemma: connection between original and aux ode} and \ref{lemma: globally asymptotically stable equilibrium lemma}. Their proofs can be easily obtained from the proofs of Lemma \ref{lemma: connection between original and aux ode} and \ref{lemma: globally asymptotically stable equilibrium lemma} by replacing $\tilde r_*$ with $\tilde r(\pi)$.
\begin{lemma}\label{TD: lemma: connection between original and aux ode}
Let $x_0 = y_0$, then $x_t= y_t + z_t e$, where $z_t$ satisfies the ODE $\dot z_t= - k z_t + (\tilde r(\pi) - k \sum y_t)$, and $k \doteq \vert \calS \vert$.
\end{lemma}
\begin{lemma} \label{TD: lemma: globally asymptotically stable equilibrium lemma}
$v_\infty$ is the unique globally asymptotically stable equilibrium for \eqref{eq: TD: original ode}.
\end{lemma}
\begin{lemma}\label{lemma: Synchronous General Differential TD}
Synchronous General Differential TD (Equation \ref{eq: TD: transformed sync single update}) converges a.s., $V_{n}$ to $v_\infty$ as $n \to \infty$.
\end{lemma}

\begin{proof}
Similar as what we did in the proof of Lemma \ref{lemma: Synchronous General Differential Q}, we use Theorem 2 in Section 2 by Borkar (2009) to show the convergence of this lemma.

We first write the synchronous update rule \eqref{eq: TD: transformed sync single update} as 
\begin{align}
    & V_{n+1} = V_{n} + \alpha_n (h(V_n) + M_{n+1}), \label{eq: TD: sync h + M update form}
\end{align}
where 
\begin{align}
    h(V_n)(s) & \doteq \sum_a \pi(a \mid s) \sum_{s', r} \tilde p (s', r \mid s, a) (r + V_n(s')) - V_n(s) - \eta \sum V_n \label{TD: sync h def} \\
    & = T(V_n)(s) - V_n(s) - \eta \sum V_n \nonumber \\
    & = T_2(V_n)(s) - V_n(s) , \nonumber \\
    M_{n + 1}(s) & \doteq \rho_n(s) \left(\tilde R_n(s, A_n(s)) + V_n(S_n'(s, A_n(s))) - V_n(s) - \eta \sum V_n \right) - h(V_n)(s). \label{TD: sync M def}
\end{align}

Similar as the proof of \Cref{lemma: Synchronous General Differential Q}, we only need to verify conditions (A1) - (A4) in order to conclude that $V_n$ converges $v_\infty$ a.s. as $n \to \infty$. 

(A1) is satisfied as both $T$ and $\sum$ operators are Lipschitz.

(A2) is satisfied by \Cref{assu: stepsize}.

(A3) is also satisfied because for any $s \in \calS$
\begin{align*}
    \bbE[M_{n+1}(s) \mid \calF_n]
    & = \bbE \left[\rho_n(s) \left(\tilde R_n(s, A_n(s)) + V_n(S_n'(s, A_n(s))) - V_n(s) - \eta \sum V_n \right) - h (V_n)(s) \mid \calF_n \right]\\
    & = \bbE\left[\rho_n(s)\left(\tilde R_n(s, A_n(s)) + V_n(S_n'(s, A_n(s))) - V_n(s) - \eta \sum V_n \right) \mid \calF_n \right] - h (V_n)(s)\\
    & = \bbE\left[\rho_n(s)\left(\tilde R_n(s, A_n(s)) + V_n(S_n'(s, A_n(s))) \right) \mid \calF_n \right]  - V_n(s) - \eta \sum V_n - h (V_n)(s)\\
    & = \bbE[\rho_n(s)(\tilde R_n(s, A_n(s)) + V_n(S_n'(s, A_n(s)))) \mid \calF_n] - T (V_n)(s)\\
    & = 0
\end{align*}

and $\bbE[\norm{M_{n+1}}^2 \mid \calF_n] \leq K(1 + \norm{V_n}^2)$ for a suitable constant $K > 0$ can be verified by applying triangle inequality given the boundedness of the second moment of the importance sampling ratio, reward and $V_n$.

To verify (A4), again we only need to verify (A5). Note that
\begin{align*}
    h_\infty(x) = \lim_{a \to \infty} h_a(x) = \lim_{a \to \infty} \frac{T(ax) - ax - \eta \left(\sum ax \right) e}{a} = T_0 (x) - x - \eta \left( \sum x \right) e,
\end{align*}
where 
\begin{align*}
    T_0 (x) \doteq \sum_a \pi(a \mid s) \sum_{s', r} \tilde p (s', r \mid s, a) x(s') .
\end{align*}
The function $h_\infty$ is clearly continuous in every $x \in \bbR^k$ and therefore $h_\infty \in C(\bbR^k)$.

Now consider the ODE $\dot x_t = h_\infty(x_t) = T_0 (x_t) - x_t - \eta (\sum x_t) e$, clearly the origin is an equilibrium. This ODE is a special case of \eqref{eq: TD: original ode}, corresponding to the reward being always zero, therefore Lemma \ref{lemma: TD: unique equilibrium} and Lemma \ref{TD: lemma: globally asymptotically stable equilibrium  lemma} also apply to this ODE and the origin is the unique globally asymptotically stable equilibrium.

(A1), (A2), (A3), (A4) are all verified and therefore $V_n \to v_\infty$ a.s. as $n \to \infty$.

\end{proof}

Given the convergence of $V_n$ in the synchronous update rule \eqref{eq: TD: transformed sync single update}, the convergence of $V_n$ in the original update rule \eqref{eq: TD: transformed async single update} follows immediately using results introduced in Chapter 7 of Borkar (2009) under Assumption \ref{assu: asynchronous stepsize 1}, \ref{assu: asynchronous stepsize td 2}. 

Finally consider $\bar R_n$. Because $\bar R_n = \eta \sum V_n - c$ (Equation \ref{TD: relation between bar R and V}) and $V_n \to v_\infty$, we have $\bar R_n \to \eta \sum v_\infty - c$. In addition, because $\tilde r(\pi) = \eta \sum v_\infty$, we have $\bar R_n \to \tilde r(\pi) - c$. Finally, because $\tilde r(\pi) = r(\pi) + c$, we have 
\begin{align}
    \bar R_n \to r(\pi)  \label{TD: convergence of R_n to r(pi)}.
\end{align}
a.s. as $n \to \infty$.

Theorem \ref{thm: convergence of uncentered differential td update} is proved.

\subsection{Centered Algorithms} \label{app:convergence-proof-centering}
This section serves as a supplement of Section \ref{sec:centering} of the main text. We introduce 4 algorithms: Centered Differential TD-learning, Centered Differential TD-planning, Centered Differential Q-learning, and Centered Differential Q-planning. All these algorithms are shown to converge to the centered (actual) differential value function rather than the differential value function plus some offset.

The next lemma is useful in the convergence proofs for the centered algorithms.

\begin{lemma}\label{lemma: equations that v_pi satisfies} Let $\pi$ be a stationary Markov policy. Assume that the induced Markov chain under $\pi$ is unichain. Let $d_\pi$ be the stationary distribution following policy $\pi$. Then

1) $(v, \bar r) = (v_\pi, r(\pi))$ is the unique solution of \eqref{state value Bellman equation 2} and 
\begin{align}
    \sum_{s} d_\pi(s) v(s) = 0, \label{d pi v pi equal 0}
\end{align}
and

2) if $v = v_\pi + ce$ then $c = \sum_s d_\pi(s) v(s)$.
\end{lemma}

\begin{proof}
Let $P_\pi$ denote the $\abs{\calS} \times \abs{\calS}$ transition probability matrix under policy $\pi$, i.e., $P_\pi(s, s') \doteq \sum_{a, r} \pi(a \mid s) p(s', r \mid s, a)$ and let $P^*_\pi \doteq \lim_{N \to \infty} \frac{1}{N} \sum_{t = 1}^N P_\pi^{t-1}$. Because $\calS$ is finite, the limit exists and $P^*_\pi$ is a stochastic matrix (has row sums equal to 1). Because the Markov chain induced by $\pi$ is unichain, all rows of $P^*_\pi$ are identical and are all equal to $d_\pi^\top$. Let $r_\pi(s) \doteq \sum_{a, r, s'} \pi(a \mid s) p(s', r \mid s, a) r$ denote the expected one-step reward under $\pi$. Then the average reward rate following $\pi$ can be written as
\begin{align}
    r(\pi) = d_\pi^\top r_\pi, \label{reward rate = d pi r pi}
\end{align}
and the differential value function following policy $\pi$ can be written as
\begin{align*}
    v_\pi(s) = \lim_{N \to \infty} \frac{1}{N} \sum_{k = 0}^{N - 1} \sum_{t = 0}^{k} P_\pi^t(r_\pi - r(\pi))(s),
\end{align*}
or $v_\pi = H_{P_\pi} r_\pi$ in vector form, where $H_{P_\pi} \doteq \lim_{N \to \infty} \frac{1}{N} \sum_{k = 0}^{N-1} \sum_{t = 0}^k (P_\pi^t - P^*_\pi)$.

The differential value function $v_\pi$ satisfies \eqref{state value Bellman equation 2} due to Theorem 8.2.6 (a) by Puterman (1994).

To see that $v_\pi$ satisfies the equation \eqref{d pi v pi equal 0}, we apply Equation A.18 in Appendix A by Puterman (1994), which is $P^*_\pi H_{P_\pi} = 0$. Therefore we have $d_\pi^\top H_{P_\pi} = 0$ because all rows of $P^*_\pi$ are $d_\pi^\top$. Because $v_\pi = H_{P_\pi} r_\pi$, we have $d_\pi^\top v_\pi = d_\pi^\top H_{P_\pi} r_\pi = 0$.

To verify that $v_\pi$ is the unique solution of \eqref{state value Bellman equation 2} and \eqref{d pi v pi equal 0}, suppose there exists another vector $v' \neq v_\pi$ satisfying \eqref{state value Bellman equation 2} and \eqref{d pi v pi equal 0}, then $v' = v_\pi + ce$ for some $c \neq 0$ (any two solutions of \eqref{state value Bellman equation 2} differ by a constant). Substituting this into \eqref{d pi v pi equal 0}, we have

\begin{align*}
    d_\pi^\top v' = d_\pi^\top (v_\pi + ce) = d_\pi^\top v_\pi + c d_\pi^\top e = c
\end{align*}

To satisfy \eqref{d pi v pi equal 0}, we must have $c = 0$. Therefore, $v_\pi$ is the unique solution of \eqref{state value Bellman equation 2} and \eqref{d pi v pi equal 0}.

To prove the second part, consider $v = v_\pi + ce$, then we have $\sum_s d_\pi(s) v(s) = \sum_s d_\pi(s) (v_\pi + ce)(s) = c$.
\end{proof}

\subsubsection{Centered Differential TD-learning and Differential TD-planning }
Centered Differential TD-learning is already presented in Section \ref{sec:centering} of the main text. The planning version of Centered Differential TD-learning is called \emph{Centered Differential TD-planning}. It uses simulated experience just as in Differential TD-planning. In addition, just like Differential TD-planning, Centered Differential TD-planning maintains $V_n$ and $\bar R_n$. Centered Differential TD-planning also maintains an auxiliary table of estimates $F_n(s, a), \forall s \in \calS, a \in \calA$ and a offset estimate $\bar V_n$, and updates them just as in Centered Differential TD-learning, using $S_n, A_n, R_n, S_n'$ instead of $S_t, A_t, R_{t+1}, S_{t+1}$.

Just as we did in Section \ref{app:convergence-proof-diffq} and \ref{app:convergence-proof-difftd}, we now present a general algorithm that includes both Centered Differential TD-learning and Centered Differential TD-planning. We call it \emph{General Centered Differential TD}. Using arguments that are similar as those in Section \ref{app:convergence-proof-diffq}, it can be shown that both Centered Differential TD-learning and Centered Differential TD-planning are special cases of General Centered Differential TD.

The data is generated the same way as it in \ref{app:convergence-proof-difftd}. Also, we use same notations introduced in \ref{app:convergence-proof-difftd}. In addition to update rules of General Differential TD (Equation \ref{eq: TD: async V update}-\ref{eq: TD: async TD error}), General Centered Differential TD has two more update rules:
\begin{align}
    F_{n+1}(s) & \doteq F_{n}(s) + \beta_{\nu(n, s)} \rho_n(s) \Delta_n(s)I\{s \in Y_n\} \quad \forall s \in \calS \label{eq: TD: async F update},\\
    \bar V_{n+1} & \doteq \bar V_n + \kappa \beta_{\nu(n, s)} \sum_{s}\rho_n(s) \Delta_n(s)I\{s \in Y_n\}, \label{TD: async bar V update}
\end{align}
where
\begin{align}
    \Delta_n(s) \doteq V_n(s) + F_n(S_n'(s, A_n(s))) - F_n(s) - \bar V_n \label{eq: TD: second TD error}.
\end{align}
Here $\beta_{\nu(n, s)}$ is the stepsize and $\kappa$ is a positive number. $\beta_{\nu(n, s)}$ and $\kappa$ doesn't need to be equal to $\alpha_{\nu(s, a)}$ and $\eta$.

\begin{theorem}[Convergence of Centered Differential TD]\label{thm: convergence of balanced differential td update}
If Assumption \ref{assu: communicating} holds, Assumption \ref{assu: stepsize}, \ref{assu: asynchronous stepsize 1}, \ref{assu: asynchronous stepsize td 2}, and \ref{assu: coverage} hold for both $\alpha_n$ and $\beta_n$, then General Centered Differential TD (Equations \ref{eq: TD: async V update}-\ref{eq: TD: async TD error}, \ref{eq: TD: async F update}-\ref{eq: TD: second TD error}) converges a.s., $\bar R_n$ to $r(\pi)$ and $V_n - \bar V_n e$ to $v_\pi$.
\end{theorem}

\begin{proof}

To show this theorem, we use the last extension of Section 2.2 by Borkar (2009), which states that a deterministic or random bounded $o(1)$ noise will not influence the convergence.

Because \eqref{eq: TD: async V update}-\eqref{eq: TD: async TD error} will not be influenced by \eqref{eq: TD: async F update}-\eqref{eq: TD: second TD error}, we have $V_n \to v_\infty$ and $\bar R_n \to r(\pi)$ according to Theorem \ref{thm: convergence of uncentered differential td update}.

Now consider \eqref{eq: TD: async F update}-\eqref{eq: TD: second TD error}. Similar as the proof of Theorem \ref{thm: convergence of uncentered differential td update}, we can combine \eqref{eq: TD: async F update}-\eqref{eq: TD: second TD error} and obtain a single update rule:
\begin{align}
    F_{n+1}(s) = F_{n}(s) + \beta_{\nu(n, a)} \rho_n(s) \left(\tilde V_n(s) + F_n(S'_n(s, A_n(s))) - F_n(s) - \kappa \sum F_n \right) I\{s \in Y_n\}, \quad \forall s \in \calS , \label{Centered TD: async single update}
\end{align}
where $\tilde V_n(s) \doteq V_n(s) + c$ and $c \doteq \kappa \sum F_0 - \bar V_0$. As we discussed above, given Assumption \ref{assu: asynchronous stepsize 1} and \ref{assu: asynchronous stepsize td 2}, to obtain $V_n - \bar V_n e \to v_\pi$, it only remains to show the convergence of the following synchronous update rule:
\begin{align}
    F_{n+1}(s) = F_{n}(s) + \beta_n \rho_n(s) \left (\tilde V_n(s) + F_n(S'_n(s, A_n(s))) - F_n(s) - \kappa \sum F_n \right), \quad \forall s \in \calS . \label{eq: centered td: sync single update}
\end{align}
Now we rewrite the above equation:
\begin{align*}
    F_{n+1}(s) = F_{n}(s) + \beta_n \rho_n(s) \left(v_\infty(s) + c + (V_n(s) - v_\infty(s)) + F_n(S'_n(s, A_n(s))) - F_n(s) - \kappa \sum F_n \right), \quad \forall s \in \calS .
\end{align*}

Let
\begin{align}
    F_{n+1} = F_{n} + \beta_n (h(F_n) + M_{n+1} + \epsilon_n), \label{Centered TD: F = h m epsilon}
\end{align}
where 
\begin{align*}
    h(F_n)(s) & = \sum_a \pi(a \mid s) \sum_{s'} p (s' \mid s, a) (v_\infty(s) + c + F_n(S'_n(s, A_n(s)))) - F_n(s) - \kappa \sum F_n, \\
    M_{n + 1}(s) & = \rho_n(s) \left(v_\infty(s) + c + F_n(S_n'(s, A_n(s))) - F_n(s) - \kappa \sum F_n \right) - h(F_n)(s), \\
    \epsilon_n(s) & = \rho_n(s)(V_n(s) - v_\infty(s)) .
\end{align*}

We first show that without $\epsilon_n$, \eqref{Centered TD: F = h m epsilon} converges. Then we need to show that $\epsilon_n$ is bounded and is $o(1)$ so that the last extension of the Section 2.2 of Borkar (2009) can be applied to conclude the convergence of \eqref{Centered TD: F = h m epsilon} with $\epsilon_n$.

Given a new table of estimates $F'_n(s, a), \forall s \in \calS, a \in \calA$, and the following update rule
\begin{align}
    F'_{n+1} & \doteq F'_{n} + \alpha_n (h(F'_n) + M_{n+1}'),
\end{align}
where $M'_{n + 1}(s) \doteq \rho_n(s) \left(v_\infty(s) + c + F'_n(S_n'(s, A_n(s))) - F'_n(s) - \kappa \sum F'_n \right) - h(F'_n)(s)$, and $F'_0 \doteq F_0$.

Lemma \ref{lemma: Synchronous General Differential TD} shows that $F_n'$ converges to some point a.s. and \eqref{TD: convergence of R_n to r(pi)} shows that $\bar V'_n \doteq \kappa \sum F_n' - c$ converges to the reward rate following policy $\pi$ in a new MDP whose transition dynamics is the same as it of the original MDP but the reward from state $s$ is $v_\infty(s)$ instead of $R_n(s, A_n(s))$. From \eqref{reward rate = d pi r pi}, the reward rate in the new MDP is $d_\pi^\top v_\infty$. Therefore $\bar V'_n = \kappa \sum F_n' - c$ converges to $d_\pi^\top v_\infty$ a.s.. 

We now show that $\epsilon_n$ is bounded and is $o(1)$. $\epsilon_n$ is bounded because $\rho_n$ is bounded due to the finite state and action space and Assumption \ref{assu: coverage}, and $V_n$ is bounded as shown in the proof of Theorem \ref{thm: convergence of uncentered differential td update}. In addition, because $V_n \to v_\infty$ and $\rho_n$ is bounded, $\epsilon_n$ converges to 0 and thus $\epsilon_n$ is $o(1)$. 

Given the above results, the last extension of the Section 2.2 of Borkar (2009) applies. In other words, the noise $\epsilon_n$ does not change the convergence of $F'_n$ (i.e., $\lim_{n \to \infty} F_n = \lim_{n \to \infty} F'_n$). Therefore we conclude that almost surely, $F_n$ converges to some point and $\bar V_n = \kappa \sum F_n - c$ converges to $d_\pi^\top v_\infty$.

Because $V_n \to v_\infty$ and $\bar V_n \to d_\pi^\top v_\infty$, $V_n - \bar V_n e \to v_\infty - d_\pi^\top v_\infty e $. Because $d_\pi^\top v_\infty e$ is a vector with all equal elements, $v_\infty - d_\pi^\top v_\infty e$ satisfies the state-value Bellman equation \eqref{state value Bellman equation 2}. In addition, because $\sum_{s'} (d_\pi(s') (v_\infty(s') - d_\pi^\top v_\infty)) = 0$, from Lemma \ref{lemma: equations that v_pi satisfies} we have $v_\infty - d_\pi^\top v_\infty e = v_\pi$. Therefore $V_n - \bar V_n e \to v_\pi$ a.s., as $n \to \infty$.

The original update rule \eqref{Centered TD: async single update} and the synchronous update rule \eqref{eq: centered td: sync single update} converge to the same point. Therefore using the original update rule, $V_n - \bar V_n e \to v_\pi$ a.s..
\end{proof}

\subsubsection{Centered Differential Q-learning and Differential Q-planning}
Our \emph{Centered Differential Q-learning} maintains, in addition to the first estimator (Equations 5-7), a second estimator in which the reward is the value estimate of the first estimator. The second estimator maintains a scalar offset estimate $\bar Q_t$, an auxiliary table of estimates $F_t(s, a), \forall s \in \calS, a \in \calS$, and uses the following update rules:
\begin{align}
    F_{t+1}(S_t, A_t) & \doteq F_t(S_t, A_t) + \beta_t \Delta_t, \text{\ \ and \ \ } F_{t+1}(s, a) \doteq F_t(s, a) , \forall s \neq S_t, a \neq A_t, \label{Centered Q Learning: F}\\
    \bar V_{t+1} & \doteq \bar V_t + \kappa \beta_t \Delta_t, \label{Centered Q Learning: bar Q}
\end{align}
where 
\begin{align} \label{Centered Q Learning: TD error}
    \Delta_t & \doteq Q_{t}(S_t, A_t) - \bar Q_t + F_t(S_{t+1}, \argmax_{a'} Q_t(S_{t+1}, a')) - F_t(S_t, A_t),
\end{align}
is the TD error of the second estimator, $\{\beta_t\}$ is a step size sequence, and $\kappa$ is a positive constant. $\beta_t$ and $\kappa$ can be different from $\alpha_t$ and $\eta$. We call (5)-(7) plus \eqref{Centered Q Learning: F}-\eqref{Centered Q Learning: TD error} Centered Differential Q-learning. 

The planning version of Centered Differential Q-learning is called \emph{Centered Differential Q-planning}. It uses simulated experience just as in Differential Q-planning. Just like Differential Q-planning, Centered Differential Q-planning maintains $Q_n$ and $\bar R_n$. In addition, Centered Differential Q-planning maintains an auxiliary table of estimates $F_n(s, a)$, for all $s \in \calS, a \in \calA$, and an offset estimate $\bar Q_n$, and updates them just as in Centered Differential Q-learning, using $S_n, A_n, R_n, S_n'$ instead of $S_t, A_t, R_{t+1}, S_{t+1}$.

Just as we did in Section \ref{app:convergence-proof-diffq} and \ref{app:convergence-proof-difftd}, we now present a general algorithm that includes both Centered Differential Q-learning and Centered Differential Q-planning cases. We call it \emph{General Centered Differential Q}. Using arguments that are similar as those in Section \ref{app:convergence-proof-diffq}, it can be shown that both Centered Differential Q-learning and Centered Differential Q-planning are special cases of General Centered Differential Q.

For General Centered Differential Q, let the data be generated the same way as it in \ref{app:convergence-proof-diffq}. Also, we use same notations introduced in \ref{app:convergence-proof-diffq}. In addition to update rules of General Differential Q (Equation \ref{eq: async Q update}-\ref{eq: Q: async TD error}), General Centered Differential Q has two more update rules:
\begin{align}
    F_{n+1}(s, a) & = F_n(s, a) + \alpha_{\nu(n, s, a)} \Delta_n(s, a) I\{(s, a) \in Y_n\} \quad \forall s \in \calS, a \in \calA, \label{eq: async F update diff q diff q} \\
    \bar Q_{n+1} & = \bar Q_n + \sum_{s, a} \alpha_{\nu(n, s, a)} \Delta_n(s, a) I\{(s, a) \in Y_n\}, \label{Q: async bar Q update} 
\end{align}
where
\begin{align}
    \Delta_n(s, a) = Q_n(s, a) + F_n(S_n'(s, a), \argmax_{a'} Q_n(S_n'(s, a), a')) - F_n(s, a) - \bar Q_n \label{eq: Q: second TD error} .
\end{align}
We now present a convergence theorem for General Centered Differential Q. Unlike the previous theorems, this theorem requires that the optimal policy is unique. The reason is, if there are multiple optimal policies all achieving the optimal average reward, the greedy policy w.r.t. $Q_n$ will jump between these optimal policies even in the limit so the second estimator can not evaluate any particular optimal policy. In addition, unlike the discounted case, where different optimal policies all correspond to the same unique optimal value function, in the average reward case, optimal policies correspond to different differential value functions. Therefore, in order to use the second estimator to evaluate some policy derived from $Q_n$, that policy must converge as $n \to \infty$.

In practice, our algorithms can still deal with problems with multiple optimal policies. This can be achieved by choosing a small threshold $\epsilon > 0$, and then replace the $\argmax_{a} Q(s, a)$ in our algorithms with the first action $\tilde a$ satisfying $Q(s, \tilde a) >= \max_a Q(s, a) - \epsilon$. The resulting policy will converge to an optimal policy if $\epsilon$ is sufficiently small. 

\begin{theorem}[Convergence of General Centered Differential Q]
If Assumption \ref{assu: communicating} holds, Assumption \ref{assu: stepsize}, \ref{assu: asynchronous stepsize 1}, and \ref{assu: asynchronous stepsize 2} hold for both $\alpha_n$ and $\beta_n$, and the optimal policy is unique, denote the differential value function for the optimal policy as $q_*$, then General Centered Differential Q (Equations \ref{eq: async Q update}-\ref{eq: Q: async TD error}, \ref{eq: async F update diff q diff q}-\ref{eq: Q: second TD error}) converges, almost surely, $\bar R_n$ to $r_*$ and $Q_n - \bar Q_n e$ to $q_*$.
\end{theorem}
\begin{proof}
Similar as what we did to show Theorem \ref{thm: convergence of balanced differential td update} we use the last extension of section 2.2 of (Borkar 2009) to show this theorem.

Because \eqref{eq: async Q update}-\eqref{eq: Q: async TD error} will not be influenced by \eqref{eq: async F update diff q diff q}-\eqref{eq: Q: second TD error}, we have $Q_n \to q_\infty$ and $\bar R_n \to r_*$ a.s., according to Theorem \ref{eq: convergence of uncentered differential q update}.

Now consider \eqref{eq: async F update diff q diff q}-\eqref{eq: Q: second TD error}. Similar as the proof of Theorem \ref{eq: convergence of uncentered differential q update}, we can combine \eqref{eq: async F update diff q diff q}-\eqref{eq: Q: second TD error} and obtain a single update rule:

\begin{align}
    & F_{n+1}(s, a) = F_{n}(s, a) + \beta_{\nu(n, s, a)} \nonumber \\
    & \left (\tilde Q_n(s, a) + F_n(S'_n(s, a), \argmax_{a'} Q_n(S'_n(s, a), a')) - F_n(s, a) - \kappa \sum F_n \right)  I\{(s, a) \in Y_n\}, \nonumber\\
    & \forall s \in \calS, a \in \calA . \label{eq: Centered Q: async single update}
\end{align}
where $\tilde Q_n(s, a) \doteq Q_n(s, a) + c$ and $c \doteq \sum F_0 - \bar Q_0$. As we discussed above, given Assumption \ref{assu: asynchronous stepsize 1} and \ref{assu: asynchronous stepsize 2}, to obtain $Q_n - \bar Q_n e \to q_*$ a.s., it only remains to show the convergence of the following synchronous update rule:
\begin{align}
    & F_{n+1}(s, a) = F_{n}(s, a) + \beta_n \nonumber \\
    & \left ( \tilde Q_n(s, a) + F_n(S'_n(s, a), \argmax_{a'} Q_n(S'_n(s, a), a')) - F_n(s, a) - \kappa \sum F_n \right), \nonumber\\
    & \forall s \in \calS, a \in \calA . \label{eq: Centered Q: sync single update}
\end{align}
Rewriting the above equation, we have
\begin{align*}
    & F_{n+1}(s, a) = F_{n}(s, a) + \beta_n \\
    & \Big( q_\infty(s, a) + c + \sum_{s'} p(s' \mid s, a) F_n(s', \pi_*(s')) - F_n(s, a) - \kappa \sum F_n \\
    & + F_n(S'_n(s, a), \pi_*(S'_n(s, a))) - \sum_{s'} p(s' \mid s, a)F_n(s', \pi_*(s')))\\
    & + Q_n(s, a) - q_\infty(s, a) + F_n(S'_n(s, a), \argmax_{a'} Q_n(S'_n(s, a), a'))\\
    & - F_n(S'_n(s, a), \pi_*(S'_n(s, a))) \Big),
\end{align*}
where $\pi_*$ is the unique greedy policy w.r.t. $q_\infty$, as there is only one optimal policy by our assumption. 

Now, let
\begin{align}
    & F_{n+1} = F_{n} + \beta_n (h(F_n) + M_{n+1} + \epsilon_n) \label{eq: Centered Q: F = h m epsilon},
\end{align}
where 
\begin{align*}
    h(F_n)(s, a) & \doteq q_\infty(s, a) + c + \sum_{s'} p(s' \mid s, a) F_n(s', \pi_*(s')) - F_n(s, a) - \kappa \sum F_n,\\
    M_{n + 1}(s, a) & \doteq F_n(S'_n(s, a), \pi_*(S'_n(s, a))) - \sum_{s'} p(s' \mid s, a)F_n(s', \pi_*(s'))),\\
    \epsilon_n(s, a) & \doteq Q_n(s, a) - q_\infty(s, a) + F_n(S'_n(s, a), \argmax_{a'} Q_n(S'_n(s, a), a')) \\
    & - F_n(S'_n(s, a), \pi_*(S'_n(s, a)).
\end{align*}

We will first show that without $\epsilon_n$, \eqref{eq: Centered Q: F = h m epsilon} converges a.s.. Then we will propose a variant of the last extension of the Section 2.2 of (Borkar 2009) and use that show the convergence of \eqref{eq: Centered Q: F = h m epsilon} with $\epsilon_n$.

Given a new table of estimates $F'_n(s, a)$, for all $s \in \calS, a \in \calA$, consider the following update rule
\begin{align*}
    F'_{n+1} = F'_{n} + \beta_n (h(F'_n) + M'_{n+1}),
\end{align*}
where $M'_{n + 1}(s, a) \doteq F'_n(S'_n(s, a), \pi_*(S'_n(s, a))) - \sum_{s'} p(s' \mid s, a)F'_n(s', \pi_*(s')))$, and $F'_0 \doteq F_0$.

The above update can be viewed as a special case of \eqref{eq: TD: sync h + M update form} with $\rho_n = 1$ for a new Markov Reward Process. The state space for this MRP is $\calS \times \calA$. The transition dynamics of the MRP is defined as $\tilde p((s', a') \mid (s, a)) \doteq \sum \pi_*(a \mid s) \sum_{s'} p(s' \mid s, a) \bbI(a' = \pi_*(s'))$ while the reward starting from $(s, a)$ is $\tilde r((s, a)) \doteq q_\infty(s, a)$. 

Therefore Lemma \ref{lemma: Synchronous General Differential Q} applies and we have that the update $F'_n$ converges to some point satisfying the state-value Bellman equation for this new MRP. By \eqref{eq: Q: R_n converges to r_*}, $\bar Q'_n = \kappa \sum F'_n - c$ converges to the reward rate for this new MRP, which is $\sum_{s, a} d_{\pi_*}(s, a) q_\infty(s, a)$, the offset in $q_\infty$ w.r.t. $q_*$ by Lemma \ref{lemma: equations that v_pi satisfies}.

Now, we propose a variant of the last extension of the Section 2.2 of Borkar (2009) and apply it to show that the additional noise $\epsilon_n$ does not affect the convergence and therefore $\lim_{n \to \infty} F_n = \lim_{n \to \infty} F'_n$ as $n \to \infty$ and $\bar Q_n$ also converges to $\sum_{s, a} d_{\pi_*}(s, a) q_\infty(s, a)$. The extension of the Section 2.2 of Borkar (2009) requires that $\epsilon_n$ is bounded and is $o(1)$. The variant we propose also requires that $\epsilon_n$ is $o(1)$, however instead of requiring $\epsilon_n$ being bounded, it requires a weaker condition 
\begin{align}
    \norm{\epsilon_n}_\infty \leq K (1 + \norm{F_n}_\infty) \label{eq: Centered Diff Q: weaker boundedness condition},
\end{align}
where $K$ is a positive constant. 

This can be shown with the following arguments. 1) If the boundedness of $F_n$ holds, then the conclusion of Lemma 1 of Section 2 of Borkar (2009) will not be affected and therefore the convergence of $F_n$ remains unchanged. 2) The boundeness of $F_n$ can be shown with the following three modifications of the proofs in Section 3 of Borkar (2009):

\begin{enumerate}
    \item It can be seen that the claim of Lemma 4 in Section 3.2 of Borkar (2009) remains unchanged with this additional noise $\epsilon_n$.
    \item A result similar to Lemma 5 in Section 3.2 of Borkar (2009) can be shown for this additional noise. That is, the sequence $\tilde \zeta'_n \doteq \sum_{k=0}^{n-1} a_k \tilde \epsilon_k, n \geq 1$ is a.s. convergent, where $a_k$ are the stepsizes, $\tilde \epsilon_k = \epsilon_k / r(n)$ for $m(n) \leq k < m(n+1)$ and $r(\cdot)$ and $m(\cdot)$ are defined in Section 3.2 of Borkar (2009). This is due to \Cref{assu: stepsize} and also $\epsilon_n$ being $o(1)$.
    \item Lemma 6 of Section 3.2 of Borkar (2009) holds with the additional $\epsilon_n$.
\end{enumerate}

Now let us verify if \eqref{eq: Centered Diff Q: weaker boundedness condition} holds and if $\epsilon_n$ is $o(1)$. It can be seen that \eqref{eq: Centered Diff Q: weaker boundedness condition} is satisfied because $Q_n$ is bounded as we showed in the proof of Lemma \ref{lemma: Synchronous General Differential Q}. In addition, because $Q_n \to q_\infty$ a.s., $\epsilon_n \to 0$ a.s. as $n \to \infty$ and thus $\epsilon_n$ is $o(1)$. Therefore a.s., $F_n$ converges and $\bar Q_n$ converges to $\sum_{s, a} d_{\pi_*}(s, a) q_\infty(s, a)$, the offset of $q_\infty$. Therefore $Q_n - \bar Q_n e \to q_*$ a.s..

$Q_n$ in the original update rule \eqref{eq: Centered Q: async single update} and $Q_n$ in the synchronous update rule \eqref{eq: Centered Q: sync single update} converge to the same point. Therefore using the original update rule, $Q_n - \bar Q_n e \to q_*$ a.s..
\end{proof}


\clearpage

\section{Additional Experiments and Experimental Details}
\label{app:experiments}

Here we provide the remaining details of all the experiments reported in this paper. We also present additional experiments that are pertinent to this paper. In particular, the following sections contain:
\begin{enumerate}\itemsep0mm
    \item Details of the control experiments in Section \ref{sec:control-exps}
    \item An experiment demonstrating the max reference function is not always the best choice of the reference function for RVI Q-learning 
    \item An experiment demonstrating RVI Q-learning diverges when the reference state is transient
    \item Details of the prediction experiments in Section \ref{sec:prediction-experiments} (both on- and off-policy)
    \item An empirical demonstration of the centering technique introduced in Section \ref{sec:centering}
\end{enumerate}

All the experiment code is available at \href{https://github.com/abhisheknaik96/average-reward-methods}{\texttt{https://github.com/abhisheknaik96/average-reward-methods}}.  

\subsection{Details of the Control Experiments on the Access-Control Queuing Task}

In this section, we provide the rest of the experimental details for the control experiments on the Access-Control Queuing Task in Section \ref{sec:control-exps} of the main text.

The task starts with all 10 servers free. With four types of customers, 11 possible number of free servers (0 to 10), and two actions, there are a total of 88 state--action pairs. The value function for both algorithms and the reward rate estimate for Differential Q-learning were initialized to zero. Both algorithms were run with step size $\alpha$ in the range $\{0.0015625, 0.00625, 0.025, 0.1, 0.4\}$. For Differential Q-learning, $\eta$ was chosen from $\{0.125, 0.25, 0.5, 1, 2\}$. The reference functions were chosen as mentioned in Section \ref{sec:control-exps}. Both algorithms used an $\epsilon$-greedy behavior policy with $\epsilon=0.1$ and no annealing.

The learning curve in Figure \ref{fig:results-accesscontrol-learningcurve} corresponds to the parameters for Differential Q-learning that resulted in the largest reward rate averaged over the training period of 80,000 steps: $\alpha=0.025$ and $\eta=0.125$. A point on the solid curves denotes the reward rate during training computed over a sliding window of previous 2000 rewards, and the shaded region denotes one standard error.

\subsection{Max is Not Always the Best Choice of Reference Function for RVI Q-learning}

In Section \ref{sec:prediction-experiments} we pointed out that RVI Q-learning performs well if the reference state (or state--action pair) occurs frequently under an optimal policy. This led to the speculation that perhaps the state--action pair with the highest action-value estimate might be the best choice of a reference state--action pair because under the RL paradigm, an agent seeks to visit the highly-rewarding states. We gave an example in the main text that it is not true in general that the state--action pair with the highest action-value estimate also occurs frequently under an optimal policy. In this section, we show this empirically; max is not always the best choice of the reference function for RVI Q-learning.

The two domains used are variants of the Two Loop MDP (from Section~\ref{sec:prediction-experiments}). The first variant has the same transitions as in the Two Loop MDP except there is a +10 reward when going from state $8$ to state $0$ instead of +2. 
The optimal policy is to take the action \texttt{right} in state $0$ and obtain a reward rate of +2 per step. The second variant builds on the first variant in that there is an additional state 9. Starting from any state other than state 9, no matter what action is taken, there is a 0.02 probability of moving to state 9 with 0 reward and a 0.98 probability of moving to a state with a reward just as in the first variant. From state 9, the action deterministically leads to state $0$ with a reward of +100; this makes state $9$ a high-value state. But it rarely occurs under the optimal policy, which is again to take the action \texttt{right} in state $0$. The optimal reward rate in this second domain is 3.84.

In addition to RVI Q-learning with the max reference function, we also ran Differential Q-learning as a baseline on these two domains. The value function for both algorithms and the reward rate estimate for Differential Q-learning were initialized to zero. Both algorithms were run with step size $\alpha$ in the range $\{0.003125, 0.00625, 0.0125, 0.025, 0.05, 0.1, 0.2, 0.4\}$. For Differential Q-learning, $\eta$ was chosen from $\{0.125, 0.25, 0.5, 1, 2\}$. Both algorithms used an $\epsilon$-greedy behavior policy with $\epsilon=0.1$ and no annealing. The experiments were run for 100000 steps and repeated 30 times. 

The sensitivity plots for both domains are shown in Figure~\ref{fig:results-control-twoloop}. The performance of RVI Q-learning was quite different qualitatively in both domains. In the first variant, the max reference function resulted in the best performance. In this case, the state--action pair with the highest action value (state $8$) occurs frequently under the optimal policy of taking the right loop. But in the second variant, the state--action pair corresponding to the highest action value (state $9$) occurs rarely under the optimal policy. As expected, the max reference function did not result in good performance in this case. In fact, the rate of learning for Differential Q-learning was better than or equal to that of RVI Q-learning with the max reference function for almost the whole range of parameters tested.

\setcounter{figure}{1}
\begin{figure*}[!t]
\centering
\begin{subfigure}{.5\textwidth}
    \centering
    \includegraphics[width=0.95\textwidth]{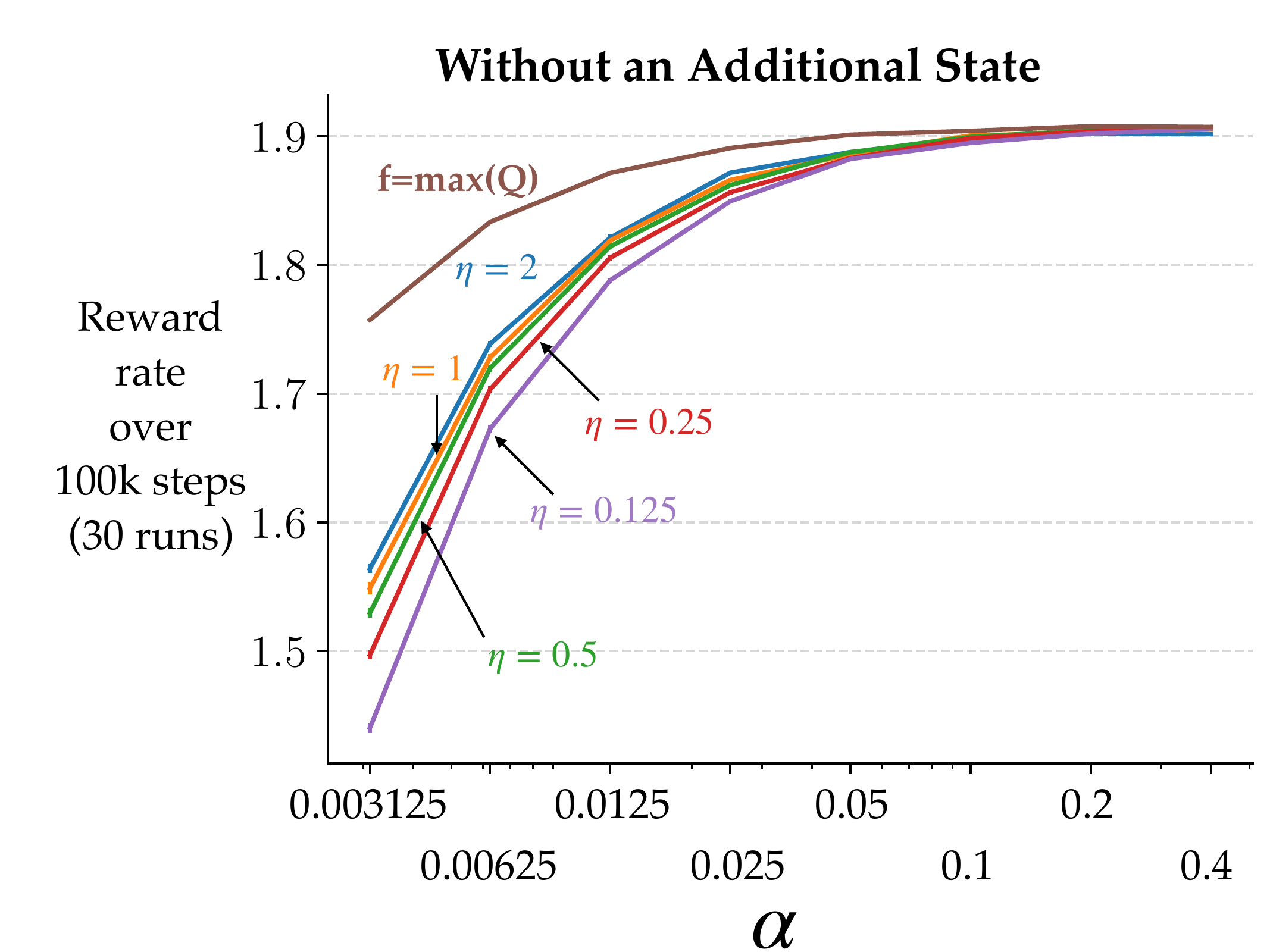}
\end{subfigure}%
\begin{subfigure}{.5\textwidth}
    \centering
    \includegraphics[width=0.95\textwidth]{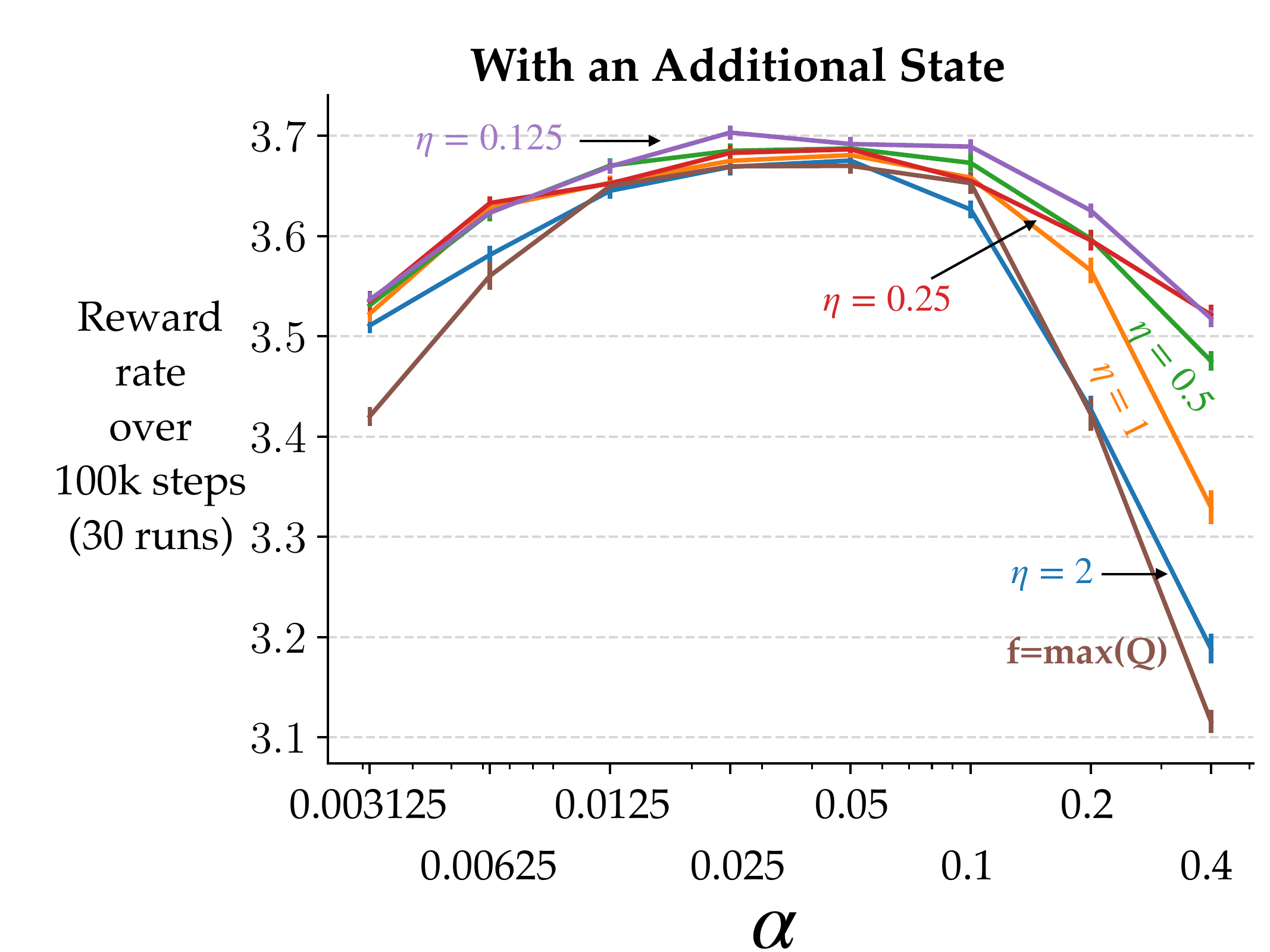}
\end{subfigure}
\vspace*{+3mm}
\caption{Parameter studies of RVI Q-learning using max as the reference function and Differential Q-learning with various values of $\eta$. \emph{Left:} In the first domain in which the state--action pair with the highest action value occurs frequently under the optimal policy, RVI Q-learning using a max reference function performed well across a wide range of step sizes. \emph{Right:} On the other hand, in the second domain, the state--action pair with the highest action value occurs rarely under the optimal policy, and RVI Q-learning using a max reference function performs well only for a relatively narrow range of step sizes. The domains are described in the text.}
\label{fig:results-control-twoloop}
\end{figure*}

These experiments show that the value of the state--action pair with the highest action value is not in general the best choice of the reference function for RVI Q-learning. 

\subsection{RVI Q-learning Diverges when the Reference State is Transient}
\label{app:rviq-divergence}

In this experiment, we show that RVI Q-learning diverges if the reference state is a \textit{transient} state of the MDP, that is, the state does not occur more than a finite number of times under any policy. Note that transient states are not allowed in communicating MDPs but are allowed in the unichain MDPs and the more general weakly-communicating MDPs. While our theory was developed for communicating MDPs, it can be extended with some modification to the more general weakly communicating MDP case (see the discussion right after Assumption~\ref{assu: asynchronous stepsize 2} for more details about this extension). The convergence results for RVI Q-learning (Abounadi et al.\ 2001) were developed for the unichain case, but we show via an experiment that RVI Q-learning can diverge under certain conditions. In particular, when the reference state is transient. 

The domain is a simple two-state MDP with the transition and reward dynamics shown in Figure \ref{fig:rviq-divergence} (left). State $0$ is transient under all stationary policies (including the optimal policy), meaning it only occurs a finite number of times before it is never seen again. 

The behavior policy is random. The value function for both algorithms and the reward rate estimate for Differential Q-learning was initialized to zero. The reference state--action pair was set to be action $a$ in state $0$. The step sizes were all set to a value of $0.01$ (an arbitrary choice; this effect can be observed for any positive step size). The starting state was state $0$, and the experiments were run for 1000 steps and repeated 50 times.

\begin{figure*}[h]
    \centering
    \begin{subfigure}{.3\textwidth}
    \centering
        \begin{tikzpicture}[scale=0.7,
        statenode/.style={circle, draw=green!60, fill=green!5, very thick, minimum size=7mm},
        actionnode/.style={circle, fill, minimum size=0.5mm, inner sep=-2},
        ]
            \node[statenode] at (0, 0)   (state0) {0};
            \node[statenode] at (2, 0)   (state1) {1};
            \node[actionnode] at (0,0.7)    (action0) {};
            \node at (0,1.3) {\small a};
        
            \draw[->] (state0) [out=330,in=210] to node[below]{\small\color{blue}-10}node[above]{\small b}(state1);
            \draw[-] (state0) [out=90,in=270] to (action0);
            \draw[->] (action0) [out=120,in=135] to node[left]{\small{\color{red}0.9},{\color{blue}+1}}(state0);
            \draw[->] (action0) [out=60,in=120] to node[above]{\small\color{red}0.1}(state1);
            \draw[->] (state1) [out=30,in=330,loop] to node[right]{\small\color{blue}+2}node[left]{\small a}(state1);
        \end{tikzpicture}
    \end{subfigure}%
    \hspace{6mm}
    \begin{subfigure}{.6\textwidth}
        \includegraphics[width=\textwidth]{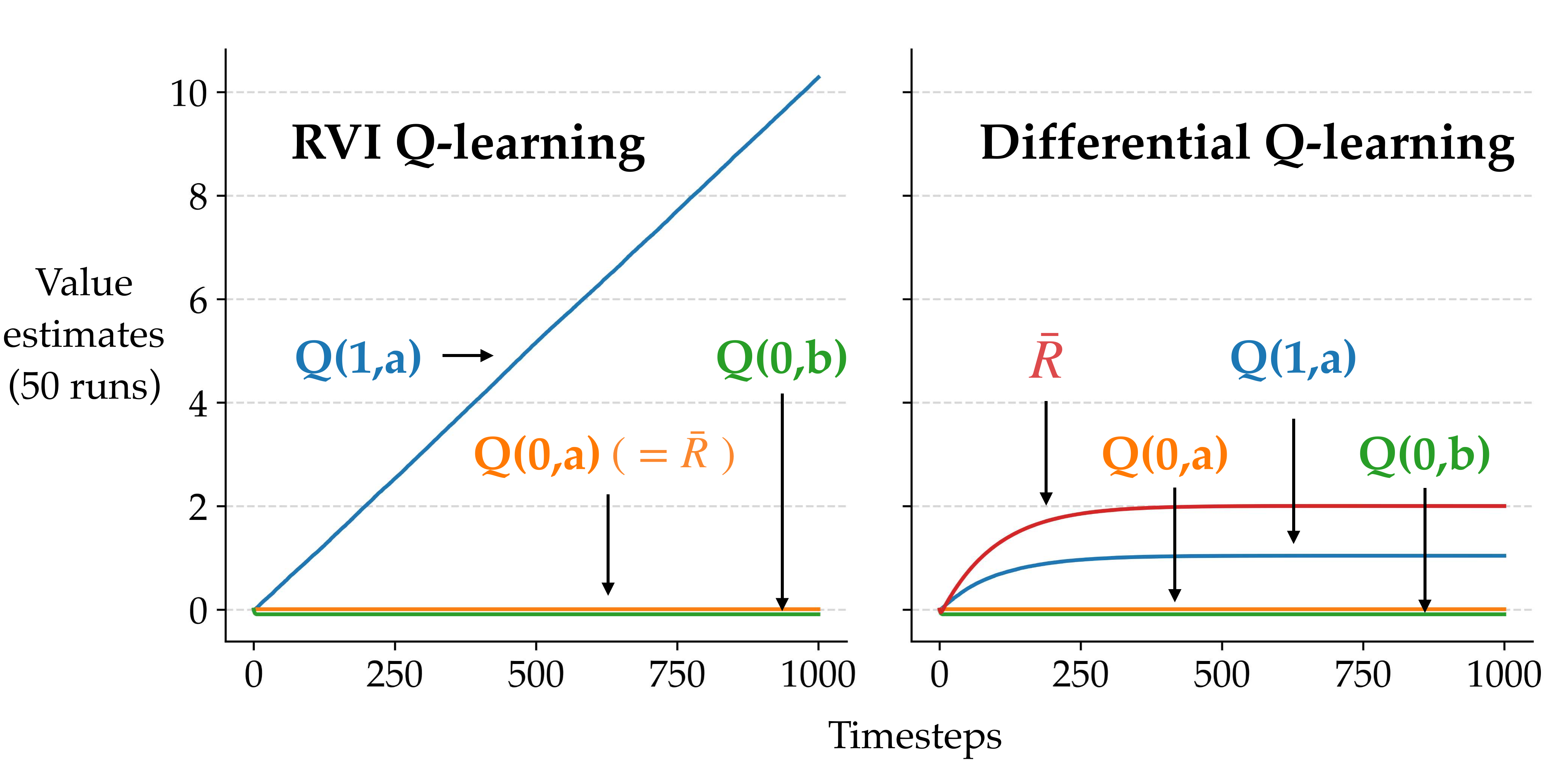}
    \end{subfigure}
    \caption{Demonstration of divergence in RVI Q-learning when the reference state is transient. \textit{Left:} The two-state MDP in which state $0$ is transient under all policies. \textit{Right:} Comparison of estimated values with RVI Q-learning and Differential Q-learning algorithms on the two-state MDP. The value of the recurrent state diverges in case of RVI Q-learning, whereas all the estimates converge in case of Differential Q-learning. The solid lines denote the mean, and one standard error is less than the width of the lines.}
    \label{fig:rviq-divergence}
\end{figure*}

Figure \ref{fig:rviq-divergence} (right) shows the evolution of the learned value estimates over time (the standard error is smaller than the width of the line representing the mean). The value of the reference state--action pair $Q(0,a)$ cannot reach the optimal reward rate of $2$
and hence the under-estimation leads to divergence in the estimate of the recurrent state which is updated as $Q_{t+1}(2, a) = Q_{t} (2, a) + \alpha \big(2 - Q(0,a) + Q_t(2, a) - Q_t(2, a)\big)$ (refer to Algorithm \ref{algo:rviQ}).

This simple experiment demonstrates that RVI Q-learning diverges when the reference state--action pair is transient.

\subsection{Details of the Prediction Experiments}
\label{app:exp-prediction}

This section presents the supplementary material for the experiments in Section \ref{sec:prediction-experiments}: the remaining experimental details, how the evaluation metric is computed, the sensitivity plots of the reward-rate error (RRE) for on-policy Differential TD-learning and Average Cost TD-learning, and the sensitivity plots for RMSVE (TVR) and RRE for off-policy Differential TD-learning.

The step size $\alpha$ and the parameter $\eta$ for all three algorithms (Average Cost TD-learning, on- and off-policy Differential TD-learning) were chosen from $\{0.025, 0.05, 0.1, 0.2, 0.4\}$ and $\{0.125, 0.25, 0.5, 1, 2\}$ respectively. The step sizes were decayed by a factor of 0.9995 at each step. The value estimates and the reward-rate estimate for all algorithms were initialized to zero. The learning curves for on-policy Differential TD-learning and Average Cost TD-learning (blue and orange) on the top-right of Figure \ref{fig:results-prediction} correspond to the parameters that minimized the average RMSVE (TVR) over the training period, which reflects their rate of learning: $\alpha=0.2$ and $\eta=0.25$ for Differential TD-learning, and $\alpha=0.1$ and $\eta=0.125$ for Average Cost TD-learning. The learning curve for off-policy Differential TD (green) in the same plot is plotted for the parameters that resulted in the minimum asymptotic RMSVE (TVR) computed over the last 5000 steps of training: $\alpha=0.2$ and $\eta=0.5$. In all the plots, the solid line represents the mean, and the error bars indicate one standard error (which in many cases was less than the width of the solid lines).

\begin{figure*}[ht!]
    \centering
    \begin{subfigure}{.5\textwidth}
    \centering
    \includegraphics[width=0.88\textwidth]{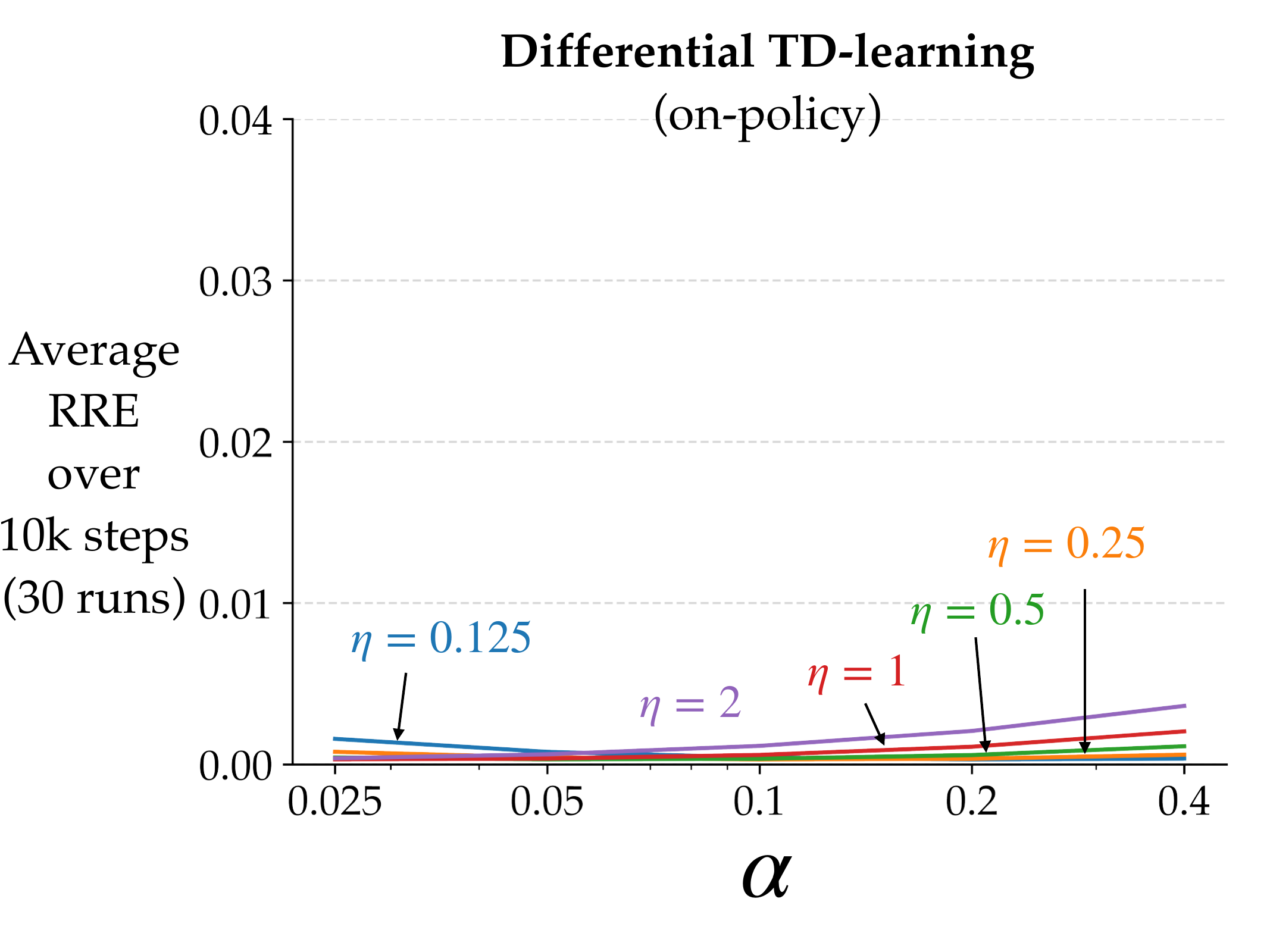}
    \end{subfigure}%
    \begin{subfigure}{.5\textwidth}
    \centering
    \includegraphics[width=0.88\textwidth]{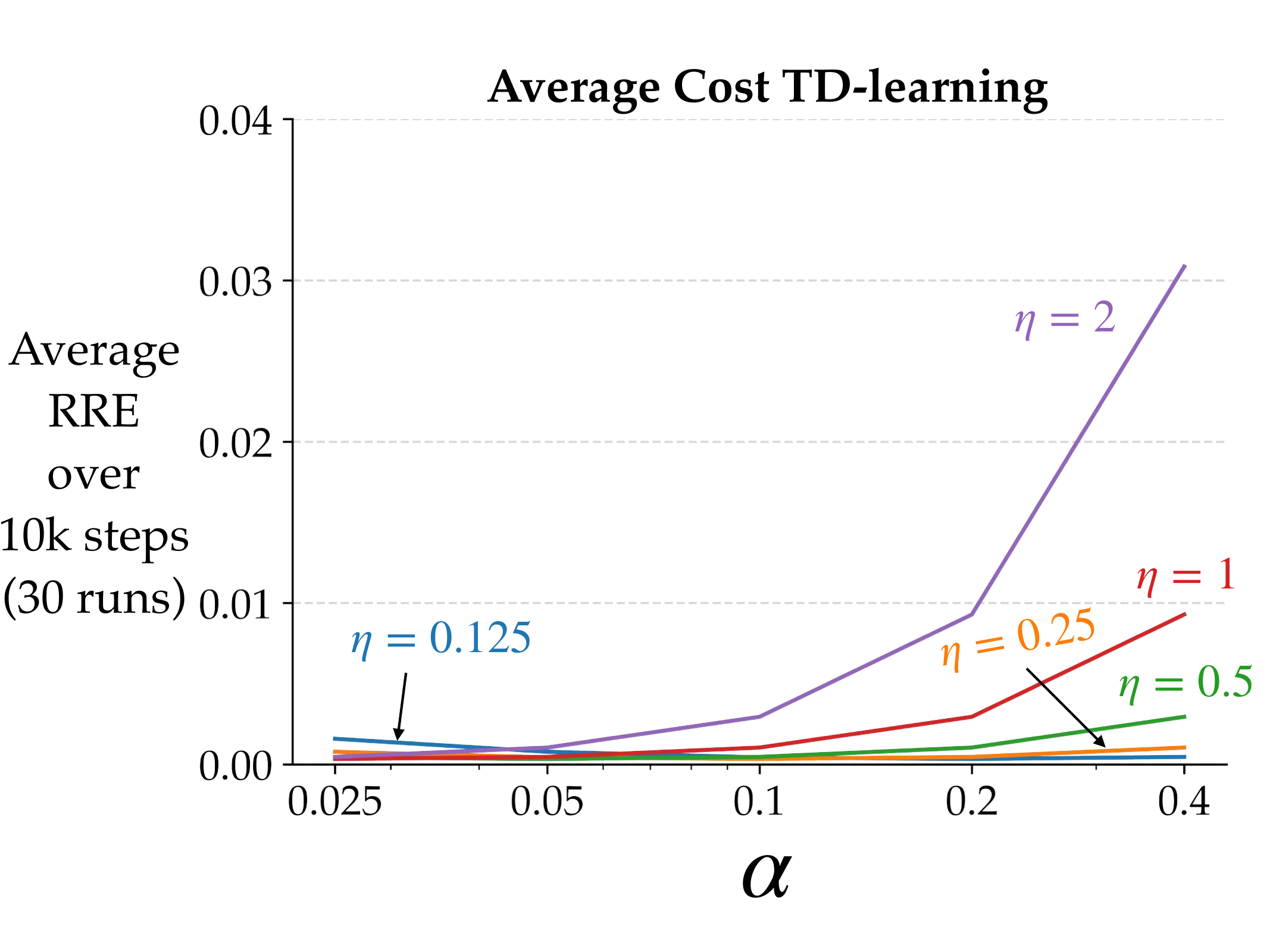}
    \end{subfigure}
    \begin{subfigure}{.5\textwidth}
    \centering
    \includegraphics[width=0.88\textwidth]{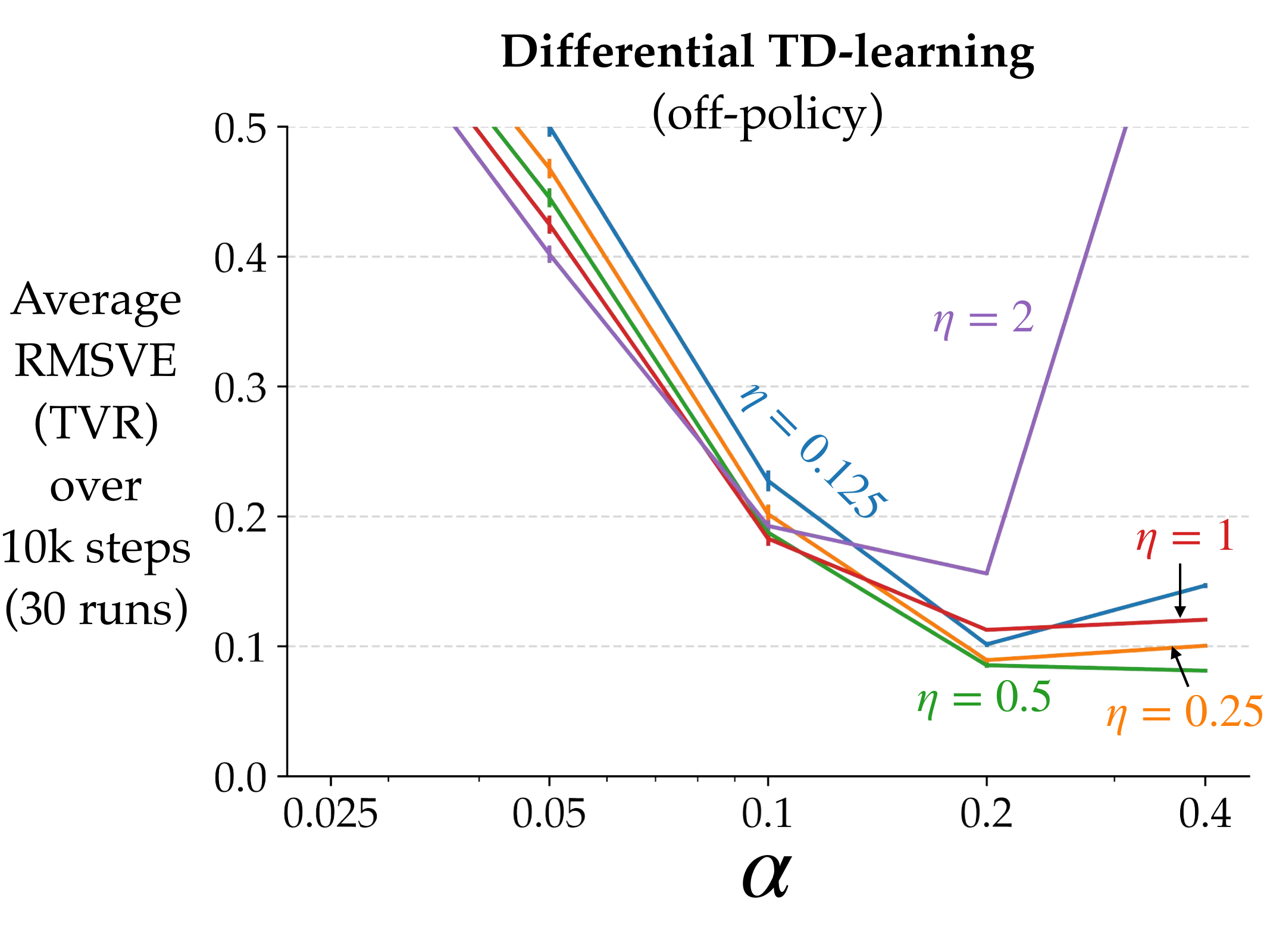}
    \end{subfigure}%
    \begin{subfigure}{.5\textwidth}
    \centering
    \includegraphics[width=0.88\textwidth]{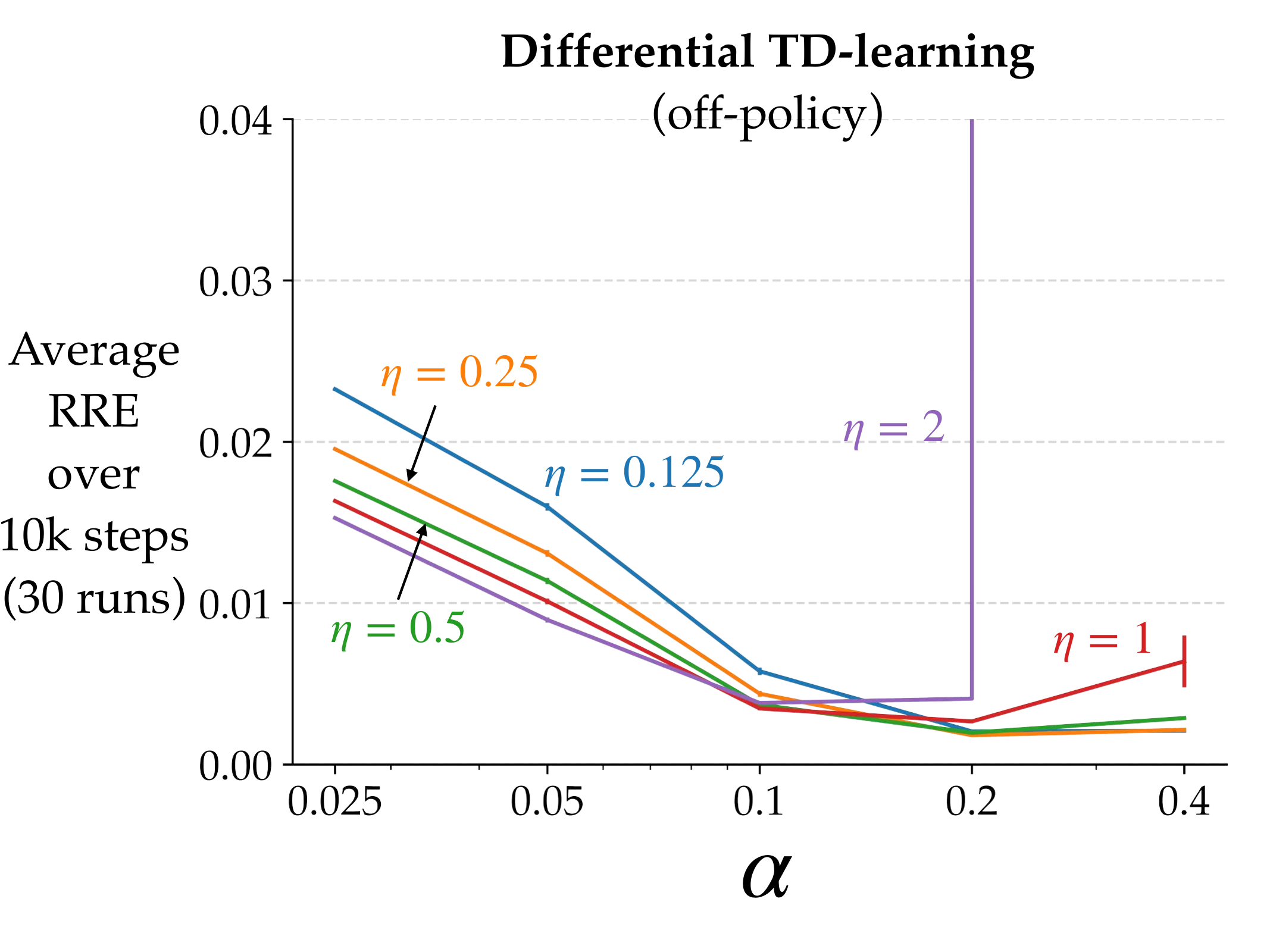}
    \end{subfigure}
    \caption{Parameter studies for the prediction experiments on the Two Loop task. The solid lines denote the mean, and one standard error is less than the width of the lines. \textit{Top:} On-policy Differential TD-learning achieves equal or lower average RRE than Average Cost TD-learning for a broad range of parameters. \textit{Bottom:} Sensitivity of off-policy Differential TD-learning's performance in terms of average RMSVE (TVR) and average RRE w.r.t.\ parameters $\alpha$ and $\eta$.}
    \label{fig:results-prediction-app}
\end{figure*}

We now describe the evaluation metric in more detail — the variant of RMSVE originally proposed by Tsitsiklis and Van Roy (1999) (hence the abbreviation `RMSVE (TVR)'). As noted in Section \ref{sec:prediction_background_algorithms}, there are multiple solutions to the Bellman equations for the differential value function of the form $v_\pi(s) + c$, where $c \in \mathbb{R}$. All algorithms converge to one of these solutions depending on design choices such as initializations and reference states. Therefore, computing the value error w.r.t.\ the actual value function $v_\pi$ does not say much about convergence. Tsitsiklis and Van Roy proposed computing the error w.r.t.\ the nearest valid solution to the Bellman equations — this error would be zero for any valid solution to the Bellman equations that an algorithm converges to. Mathematically, this error is given by:
\[
\inf_c \norm{ {v} - ({v}_\pi + c\,{e}) }_{{d}_\pi} = \norm{ \mathcal{P}{v} - {v}_\pi }_{{d}_\pi},
\]
where $\mathcal{P}$ is a projection operator and ${d}_\pi$ is the stationary state distribution corresponding to the policy $\pi$. Algorithmically, this translates to computing the offset of the learned value function, subtracting it, and then computing the RMSVE w.r.t. the actual value function ${v}_\pi$. The offset can be computed by simply taking a dot product of the learned value function and ${d}_\pi$:
${d}_\pi^T ({v}_\pi + c\,{e}) = {d}_\pi^T {v}_\pi + c\,{d}_\pi^T {e} = 0 + c = c$,
where ${d}_\pi^T {v}_\pi = 0$ (from Lemma \ref{lemma: equations that v_pi satisfies}). 

For the target policy $\pi$ that uniformly randomly picks one of the two actions in state $0$, ${d}_\pi = [0.2, 0.1, 0.1, 0.1, 0.1, 0.1, 0.1, 0.1, 0.1]^T$, and ${v}_\pi = [-0.2, -1.4, -1.1, -0.8, -0.5, 0.6, 0.9, 1.2, 1.5]^T$, which can be obtained by solving the Bellman equations along with the constraint ${d}_\pi^T {v}_\pi = 0$. 

The top row in Figure \ref{fig:results-prediction-app} shows the sensitivity of RRE of on-policy Differential TD-learning and Average Cost TD-learning averaged over the training period w.r.t.\ its parameters for the experiment in Section \ref{sec:prediction-experiments}. On-policy Differential TD-learning was less sensitive to both $\alpha$ and $\eta$, and converged to a lower RRE across a large range of its parameters as compared to Average Cost TD-learning. A similar trend was observed for RMSVE (TVR) and discussed in Section \ref{sec:prediction-experiments}. 

The bottom row of Figure \ref{fig:results-prediction-app} shows the sensitivity of both RMSVE (TVR) and RRE of off-policy Differential TD-learning w.r.t.\ its parameters $\alpha$ and $\eta$. The rate of convergence is affected if the parameters are too high or too low, but otherwise it is relatively insensitive to different choices of $\eta$. 


\subsection{Estimating the Actual Differential Value Function}
\label{app:exp-centering}

\begin{table*}[t]
    \centering
    \caption{The stationary state--action distribution and the action-value function for the optimal policy of choosing action \texttt{right} in state $0$ in the Two Loop task.}
    \begin{tabular}{@{}ccccccccccc@{}}
    \toprule
    \multicolumn{1}{l}{} & \multicolumn{10}{c}{state--action pair}                                               \\ \midrule
                         & (0,left) & (0,right) & (1,a) & (2,a) & (3,a) & (4,a) & (5,a) & (6,a) & (7,a) & (8,a) \\ \midrule
    $d_\pi$              & 0        & 0.2       & 0     & 0     & 0     & 0     & 0.2   & 0.2   & 0.2   & 0.2   \\
    $q_\pi$              & -1.8     & -0.8      & -2.4  & -2.0  & -1.6  & -1.2  & -0.4  & 0.0   & 0.4   & 0.8   \\ \bottomrule
    \end{tabular}
    \label{tab:centering-twoloop-opt}
\end{table*}

\begin{figure*}[b!]
    \centering
    \begin{subfigure}{.5\textwidth}
        \includegraphics[width=0.88\textwidth]{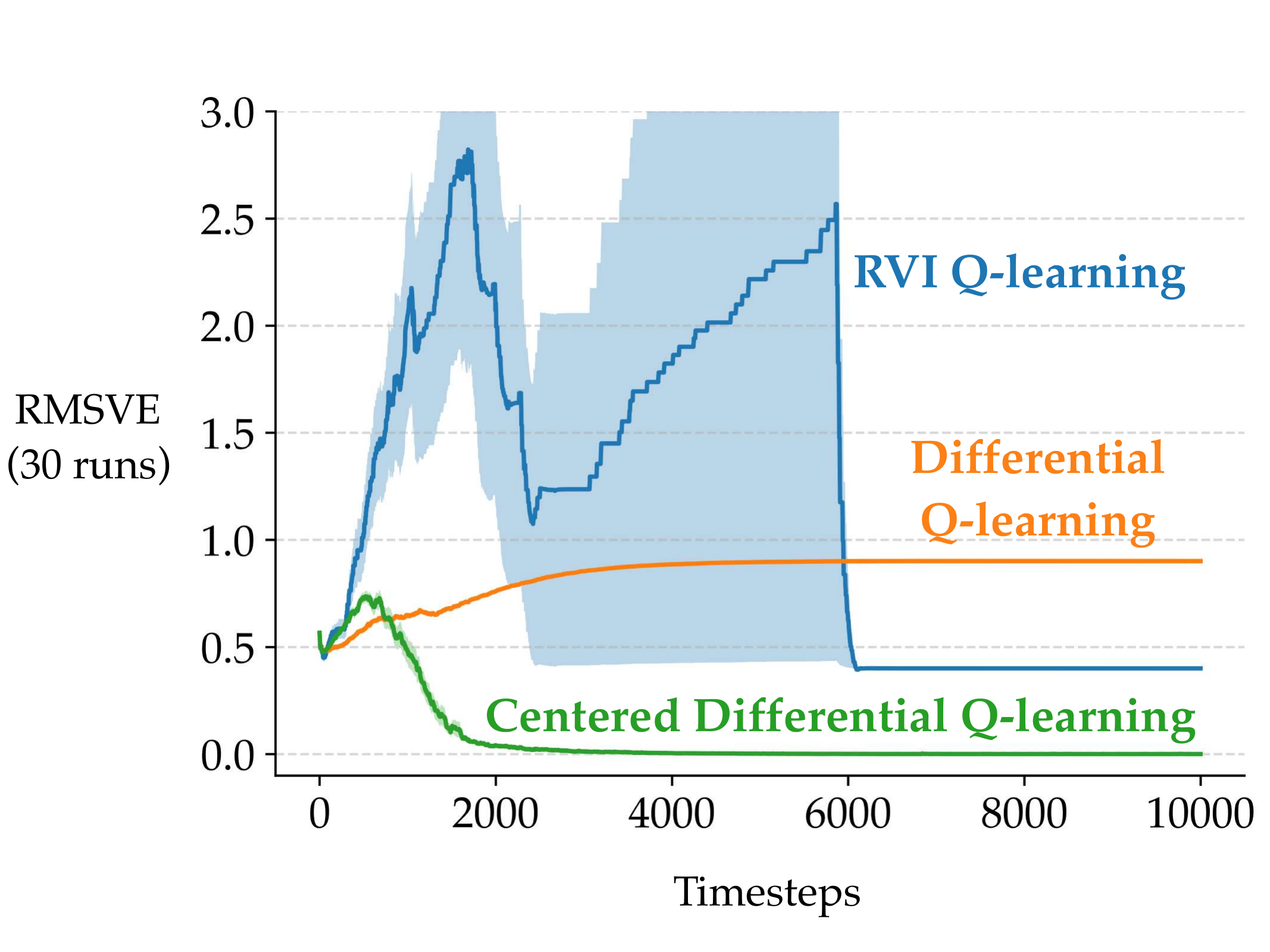}
    \end{subfigure}%
    \begin{subfigure}{.5\textwidth}
        \includegraphics[width=0.88\textwidth]{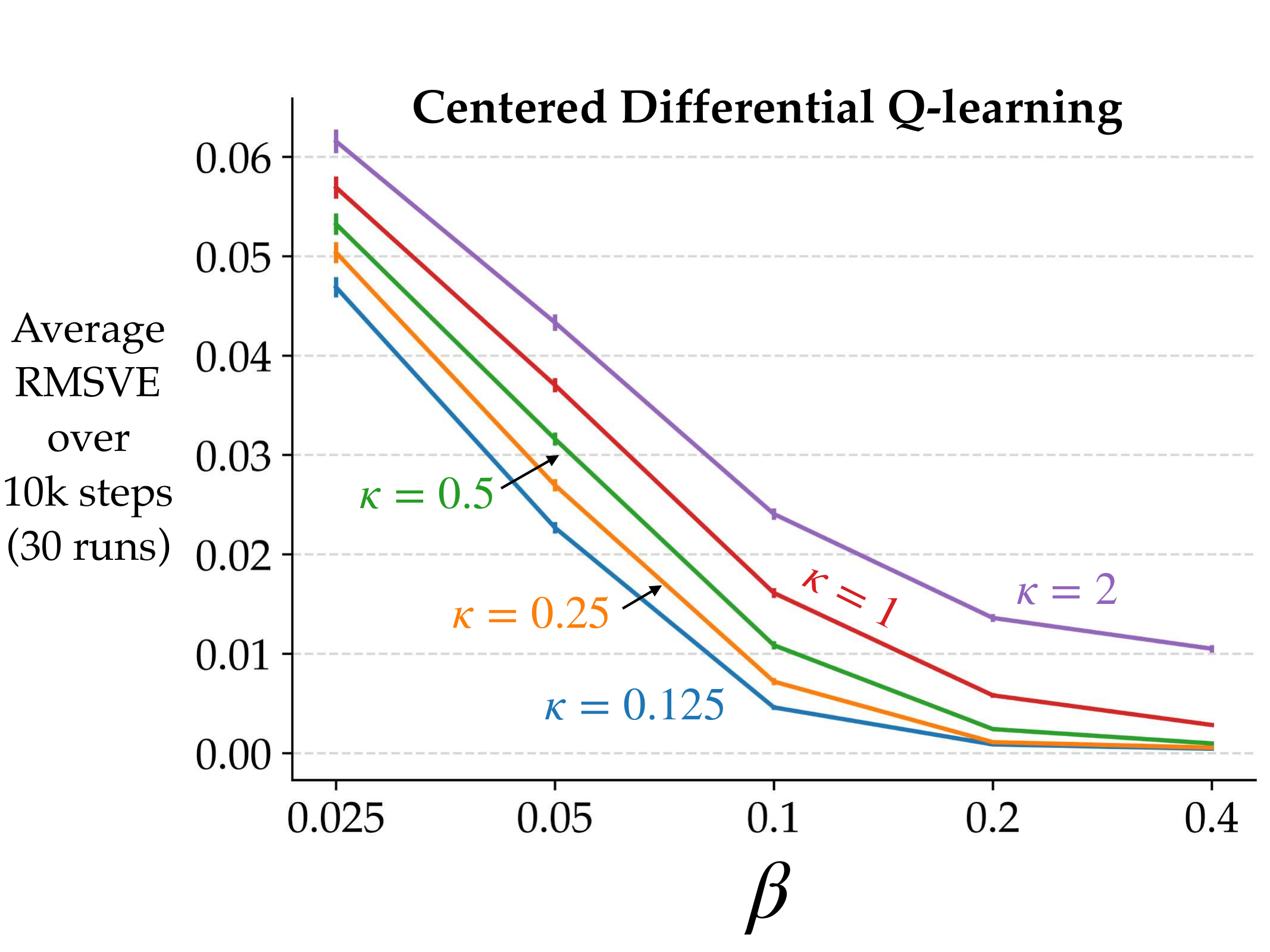}
    \end{subfigure}
    \caption{\textit{Left:} On the Two Loop task, Centered Differential Q-learning learned the centered differential value function corresponding to the optimal policy, while RVI Q-learning and (uncentered) Differential Q-learning converged to some offset versions of the centered differential value function. \textit{Right:} Parameter study showing the centering technique was relatively robust to the parameter $\kappa$ and resulted in a low average RMSVE for a broad range of its parameters. All solid curves denote the mean, and the shaded region or error bars denote one standard error.}
    \label{fig:centering-diffq}
\end{figure*}

In this section, we demonstrate that the technique introduced in Section \ref{sec:centering} of the main text results in the estimation of the actual (centered) differential value function. As mentioned earlier, the technique is general and can be used with any average-reward algorithm proposed in this paper. For illustration, we use this technique with Differential Q-learning on the Two Loop task described in Section \ref{sec:prediction-experiments} of the main text.

We first applied RVI Q-learning and (uncentered Differential Q-learning on this problem). RVI Q-learning was run with
the reference function which is the value of a single reference state--action pair, $f=Q(s_0,a_0)$, for which we considered all possible state--action pairs as $s_0$ and $a_0$: state 0 with actions \texttt{left} and \texttt{right}, and states 1–8 with the action $a$. Both algorithms were run with a range of step sizes $\alpha \in \{0.025, 0.5, 0.1, 0.2, 0.4\}$, each for 30 runs and 10,000 steps. Differential Q-learning was run with a range of $\eta$ values: $\{0.125, 0.25, 0.5, 1, 2\}$. Both algorithms had an $\epsilon$-greedy behavior policy with $\epsilon = 0.1$ and no annealing. The value function for both the algorithms and the reward rate estimate for Differential Q-learning was initialized to zero.
We then applied Centered Differential Q-learning (refer to Algorithm \ref{algo:diffQ-centered} for the pseudocode) on the Two Loop MDP, for 30 runs and 10,000 steps with the parameters $\beta$ and $\kappa$ chosen from $\{0.025, 0.05, 0.1, 0.2, 0.4\}$ and $\{0.125, 0.25, 0.5, 1, 2\}$ respectively. $\alpha$ and $\eta$ were chosen as ones that achieved performed well for (uncentered) Differential Q-learning ($\alpha=0.4, \eta=0.5$) because the centering technique in its current form learns the offset separately and does not affect the value estimates during learning. The learned offset was subtracted from the value estimates every time the RMSVE\footnote{Note that we no longer need Tsitsiklis and Van Roy's (1999) variant of the RSMVE (which we earlier used and denoted as `RMSVE (TVR)') because now we want to compute the error w.r.t.\ the actual differential value function.} was computed. The stationary state--action distribution and the action-value function for the optimal policy is shown in Table \ref{tab:centering-twoloop-opt}, which were again obtained by solving the Bellman equations with the constraint that $d_\pi^T v_\pi = 0$. The estimates of the secondary estimator were also initialized to zero in every run.

A learning curve for each of the algorithms is shown in the left side of Figure \ref{fig:centering-diffq}. For RVI Q-learning and (uncentered) Differential Q-learning, these are corresponding to the parameter settings that resulted in the largest reward rate averaged over the training period (reference state 8, action $a$ with $\alpha=0.4$ for RVI Q-learning and $\alpha=0.4$ with $\eta=0.5$ for Differential Q-learning). 
For Centered Differential Q-learning, the learning curve is plotted for the parameters that resulted in the lowest RMSVE averaged over the course of training ($\beta=0.4, \kappa=0.125$). 

We saw that Centered Differential Q-learning converged to a differential value function with zero RMSVE, in other words, the centering technique proposed in Section \ref{sec:centering} of the main text succeeds in estimating the offset correctly, which is subtracted from the value estimates to result in the centered differential value function. RVI Q-learning and (uncentered) Differential Q-learning also converged to some particular value functions with some offset from the centered differential value function. Note that there was a lot of variance in the values estimated by RVI Q-learning till it converged after about 6000 steps. The source of this variance needs further investigation. 
Also shown on the right of Figure \ref{fig:centering-diffq} is the sensitivity of the performance of the centering technique to its two parameters $\beta$ and $\kappa$. We saw that in this task where the transitions were mostly deterministic, larger step sizes $\beta$ could be used, and the value of $\kappa$ only had a small effect.

This experiment shows that the technique introduced in Section \ref{sec:centering} of the main text can learn the centered differential value function. We demonstrated this with Differential Q-learning, an off-policy control learning algorithm, and we expect this technique to work with other combinations of settings as well: on-policy and off-policy, prediction and control, learning and planning. 


\clearpage

\section{Additional Discussion}
\label{app:additional}

\subsection{Yang et al.'s (2016) convergence results are incorrect} 
\label{app: Mistakes in CSV-learning}

We show that the proofs of two lemmas leading up to the convergence theorem of CSV-learning are not valid. The proofs would be valid if both the transition and reward dynamics of the MDP are deterministic. However the assumption of the MDP being deterministic is not made in the paper. We begin by presenting the relevant assumptions, definitions, and lemmas from the paper (Yang et al.~2016).

\textbf{Assumption 1}: \textit{The MDP is irreducible, aperiodic, and ergodic; that is, under any stationary policy, the generated Markov chain is communicating and has a recurrent state.}

\textbf{Definition 1}: $d_t(s',s) \doteq \max_{a'} Q_t(s',a') - \max_a Q_t(s,a)$

\textbf{Lemma 1}: \textit{In the Markov chain $M_t$ under $\pi_t$, $d_t(s',s) = \rho_{\pi_t} - r(s,\pi_t(s))$.}

Let us see why this lemma is incorrect:
\begin{align*}
    d_t(s',s) &\doteq \max_{a'} Q_t(s',a') - \max_a Q_t(s,a) \\
    &= \max_{a'} Q_t(s',a') - Q_t(s,\pi_t(a)) \\
    &= \max_{a'} Q_t(s',a') - \sum_{s^n,r} p(s^n,r \mid s,\pi_t(s)) \big[ r - \rho_{\pi_t} + \max_{a^n} Q(s^n,a^n) \big] \\
    &= \max_{a'} Q_t(s',a') - \big[ r(s,\pi_t(s)) - \rho_{\pi_t} + \sum_{s^n,r} p(s^n,r \mid s,\pi_t(s)) \max_{a^n} Q(s^n,a^n) \big] \\
    &= \rho_{\pi_t} - r(s,\pi_t(s)) + \max_{a'} Q_t(s',a') - \sum_{s^n,r} p(s^n,r \mid s,\pi_t(s)) \max_{a^n} Q(s^n,a^n) \\
    &\neq \rho_{\pi_t} - r(s,\pi_t(s))
\end{align*}
The equality does not hold in general, only if there is a deterministic transition from state $s$ to $s'$. 

\textbf{Lemma 2}: \textit{If $\pi_t$ is stable, then $\rho_{\pi_t} \geq \hat{\rho}$.} ($\hat{\rho}$ is the constant or fixed estimate of the reward rate used by the CSV-learning algorithm)

This is also incorrect because:
\begin{align*}
    \Delta Q_t(s,\pi_t(s)) &= \alpha \big( r - \hat{\rho} + \max_{a'} Q_t(s',a') - Q_t(s,\pi_t(s)) \big) \\
    &= \alpha \big( r - \hat{\rho} + \max_{a'} Q_t(s',a') - \max_a Q_t(s,a) \big) \\
    &= \alpha \big( r - \hat{\rho} + d_t(s',s) \big) \ \ \,\qquad\qquad \text{(Definition 1)} \\
    &= \alpha \big( r - \hat{\rho} + \rho_{\pi_t} - r(s,\pi_t(s))\big) \quad \text{(Assuming a deterministic transition from $s$ to $s'$)} \\
    &\neq \alpha \big( \rho_{\pi_t} - \hat{\rho} \big)
\end{align*}
Again, the equality does not hold in general, but only if the rewards are deterministic as well. In other words, the expected reward from a given state after taking an action according to the policy ($r(s,\pi_t(s)$) is equal to the immediate reward ($r$). Note we also had to assume there is a deterministic transition from $s$ to $s'$.

Thus, Theorem 1 only holds if the both the transition and reward dynamics of the MDP are deterministic. The paper does not state this assumption\footnote{In any case, assuming MDPs are completely deterministic is a significantly restrictive and unrealistic assumption.}
, which invalidates the proof.


\subsection{Average Cost TD-learning cannot be extended to the off-policy case by adding an importance-sampling ratio}

Average Cost TD-learning cannot be extended to the off-policy setting by simply adding an importance-sampling (IS) ratio as it only corrects the mismatch in targets due to misalignment between actions taken by the target and behavior policies. The IS ratio does not correct the mismatch in distribution of updates, which is also required. However, using the TD error instead of the conventional error as in Differential TD-learning to update the reward-rate estimate only requires correction to the mismatch in targets and not to the mismatch in distribution of updates, and hence the addition of the IS ratio suffices. Both of these claims are substantiated below. 

Consider the update made by the Average Cost TD-learning algorithms to the reward-rate estimate in the on-policy setting:
\begin{align}
    \bar{R}_{t+1} &= \bar{R}_t + \eta\alpha_t \big( R_{t+1} - \bar{R}_{t} \big) \label{eq:avgcosttd_barr}
\end{align}

At convergence the expected update is zero:
\begin{align*}
    0 &= \mathbb{E}[R_{t+1} - \bar{R}_t]\\
    0 &= \sum_s d_\pi(s) \sum_a \pi(a|s) \sum_{s',r} p(s',r \mid s,a) (r - \bar{R}_\infty) \\
    0 &= \sum_s d_\pi(s) \sum_a \pi(a|s) \sum_{s',r} p(s',r \mid s,a)\ r - \bar{R}_\infty\\
    0 &= r(\pi) - \bar{R}_\infty  \qquad \text{(by definition)} \\
    \implies\ \bar{R}_\infty &= r(\pi)
\end{align*}
where $d_\pi(s)$ is the steady-state distribution over states when following policy $\pi$, and $\bar{R}_\infty$ is the point where $\bar{R}_t$ converges, and $r(\pi)$ is the true reward rate of the policy $\pi$. The above equations show that in the on-policy setting, the reward-rate estimate of Average Cost TD-learning converges to the true reward rate of the target policy.

Adding an importance-sampling ratio to Average Cost TD-learning to extend it to the off-policy setting does not work because the point of convergence is no longer the reward rate of the target policy $\pi$ when following behavior policy $b$.

Proposed off-policy reward-rate update: $\bar{R}_{t+1} = \bar{R}_t + \eta\alpha_t \rho_t \big( R_{t+1} - \bar{R}_{t} \big)$.

Following a similar analysis at the point of convergence:
\begin{align*}
    0 &= \mathbb{E}\big[\rho_t(R_{t+1} - \bar{R}_t)\big]\\
    0 &= \sum_s d_b(s) \sum_a b(a|s) \sum_{s',r} p(s',r \mid s,a) \frac{\pi(a|s)}{b(a|s)} (r - \bar{R}_\infty) \\
    0 &= \sum_s d_b(s) \sum_a \pi(a|s) \sum_{s',r} p(s',r \mid s,a) (r - \bar{R}_\infty) \\
    0 &= \sum_s d_b(s) \sum_a \pi(a|s) \sum_{s',r} p(s',r \mid s,a)\ r - \bar{R}_\infty \\
    \implies\ \bar{R}_\infty &= \sum_s d_b(s) \sum_a \pi(a|s) \sum_{s',r} p(s',r \mid s,a)\ r \\
    &\neq r(\pi)
\end{align*}

With the proposed off-policy reward-rate update for Average Cost TD-learning, the point of convergence for the reward-rate estimate is no longer the reward rate of the target policy $\pi$ because the IS ratio only corrects for the mismatch in targets and not the mismatch in the distribution of updates.

Now consider the update to the reward-rate estimate in case of off-policy Differential TD-learning:
\begin{align}
    \bar{R}_{t+1} = \bar{R}_t + \eta\alpha_t \rho_t \big( R_{t+1} - \bar{R}_{t} + V(S_{t+1}) - V(S_t) \big) \label{eq:difftd-on_barr}
\end{align}
At convergence, 
\begin{align*}
    0 &= \mathbb{E}\Big[\rho_t\big(R_{t+1} - \bar{R}_t + V_t(S_{t+1} - V_t(S_t)\big)\Big]\\
    0 &= \sum_s d_b(s) \sum_a b(a|s) \sum_{s',r} p(s',r \mid s,a) \frac{\pi(a|s)}{b(a|s)} \big(r - \bar{R}_\infty + v_\infty(s') - v_\infty(s)\big) \\
    0 &= \sum_s d_b(s) \sum_a \pi(a|s) \sum_{s',r} p(s',r \mid s,a) \big(r - \bar{R}_\infty + v_\infty(s') - v_\infty(s)\big) \\
    0 &= \sum_s d_b(s) \Big[ \sum_a \pi(a|s) \sum_{s',r} p(s',r \mid s,a) \big(r - \bar{R}_\infty\big) + \sum_a \pi(a|s) \sum_{s',r} p(s',r \mid s,a) \big(v_\infty(s') - v_\infty(s)\big) \Big] \\
    0 &= \sum_s d_b(s) \Big[ \sum_a \pi(a|s) \sum_{s',r} p(s',r \mid s,a) \big(r - \bar{R}_\infty\big) + \sum_a \pi(a|s) \sum_{s',r} p(s',r \mid s,a) \big(r - r(\pi)\big) \Big] \\
    0 &= \sum_s d_b(s) \sum_a \pi(a|s) \sum_{s',r} p(s',r \mid s,a)\big(r - \bar{R}_\infty - r + r(\pi) \big)\\
    0 &= \sum_s d_b(s) \sum_a \pi(a|s) \sum_{s',r} p(s',r \mid s,a)\big(r(\pi) - \bar{R}_\infty \big)\\
    0 &= r(\pi) - \bar{R}_\infty \\
    \implies\ \bar{R}_\infty &= r(\pi)
\end{align*}

The above holds due to the $v_\infty$ being a solution to the Bellman equation by Theorem \ref{Differential TD-learning}: $v_\infty(s) = \sum_a \pi(a|s) \sum_{s',r} p(s',r \mid s,a) \big(r - r(\pi) + v_\infty(s')\big)$.

Hence, Average Cost TD-learning is inherently an on-policy algorithm. It cannot be extended to the off-policy case simply by using an importance-sampling ratio in its updates. On the other hand, the usage of the TD-error to update the reward-rate estimate instead of the conventional error enables Differential TD-learning to converge to the right solution with the usage of the importance-sampling ratio in the off-policy case.


\subsection{Discussion about other off-policy prediction methods in the literature}

In this section, we shed some light on why existing off-policy prediction methods for average-reward MDPs (Liu et al.\ 2018, Tang et al.\ 2019, Mousavi et al.\ 2020, Zhang et al.\ 2020a,b) are not guaranteed to converge to the reward rate of the target policy. The primary characteristic of these methods is that they estimate only the reward rate and not the differential value function. To estimate the reward rate, they first estimate the ratio of the stationary distribution under the target policy and under the behavior policy, and then use that to estimate the reward rate using. The first step is difficult, after which the second step is straightforward. These methods use different ways to estimate the ratio using a batch of data. 

While these methods are developed for the function approximation setting, none of them can guarantee convergence to the reward rate even in the tabular setting with an infinite-sized batch of data. Liu et al.\ (2018), Tang et al.\ (2019), and Mousavi et al.\ (2020) used a self-normalization trick in their methods, which typically leads to a biased solution (Zhang et al.\ 2020a). 
Zhang et al.\ (2020a) used a primal-dual approach and their algorithm optimizes a min-max objective; however, this objective may not be convex-concave even in the tabular setting (Zhang et al.\ 2020b). Therefore there could be multiple solutions to that objective and their algorithm will not in general obtain the one corresponding to the true reward rate. Finally, Zhang et al.\ (2020b) proposed optimizing a convex-concave saddle-point problem that has a unique solution. However, because of a regularization term in the objective, the unique solution of the objective in general does not yield the true reward rate---only a biased one---even in the tabular case.


\clearpage

\section{Extensions of our Differential Algorithms to the Setting of Linear Function Approximation}
\label{app:extensions_lfa}

We consider the setting in which at each timestep $t$, the agent observes a feature vector $\mbf{x}_t$ representing the state of the environment, takes a discrete action $A_t$, observes the next feature vector $\mbf{x}'_t$ and the scalar reward signal $R_{t+1}$. The agent approximates the action-value function at each timestep $t$ as a linear function of the feature vector: $\hat{q}_t \doteq \mbf{w}_{A_t}^\top \mbf{x}_t$ (for feature vectors of dimension $d$, the action-value function is parameterized with $|\cal A|$ number of $d$-dimensional weight vectors).

\begin{algorithm}[h]
\SetAlgoLined
\SetKwInput{AP}{Algorithm parameters}
\AP{step-size parameters $\alpha$, $\eta$}
Initialize $\mbf{w}_a \in \mathbb{R}^d\ \forall a$ and $\bar{R} \in \mathbb{R}$ arbitrarily (e.g., to zero) \\
 Obtain initial observation vector $\mbf{x}$ \\
 \For{each timestep}
 {
    Take action $A$ (using, say, an $\epsilon$-greedy policy w.r.t.\ $\hat{q}$), obtain $R,\mbf{x}'$ \\
    $\delta = R - \bar{R} + \max_a \mbf{w}_a^T\mbf{x'} - \mbf{w}_A^T\mbf{x}$ \\
    $\mbf{w}_A \leftarrow \mbf{w}_A + \alpha\,\delta\,\mbf{x}$ \\
    $\bar{R} \leftarrow \bar{R} + \eta\,\alpha\,\delta$ \\
    $\mbf{x} = \mbf{x}'$ \\
 } 
 \caption{Differential Q-learning with linear function approximation}
 \label{algo:diffQ_LFA}
\end{algorithm}

This algorithm is a straightforward extension of the tabular version of Differential Q-learning (Algorithm \ref{algo:diffQ-uncentered}). Similar extensions exist for Differential Q-planning, Differential TD-learning, Differential TD-planning, and the centered versions of these algorithms.

\subsection{Extension of the notion of reference functions to the function approximation setting is not as straightforward}

Compared to Differential Q-learning, extensions of the tabular version of RVI Q-learning (Abounadi et al.\ 2001) to the case of (linear) function approximation are not as straightforward. RVI Q-learning requires the value of a reference function $f$ to be computed at every timestep $t$, where $f$ is a function over the current estimates of the value estimates $\hat{q}_t$. 
Some difficulties that arise with the first attempts of extending the reference functions suggested by Abounadi et al.\ to the function approximation setting:

\begin{itemize}\itemsep0mm
    \item Reference function is the mean of all action-value estimates: $f(\hat{q}_t) \doteq \frac{1}{|\calS||\calA|} \sum_{s,a} \hat{q}_t(s,a)$ \\
    It is easy to see why the computation of this quantity is problematic in the function approximation setting: unlike the tabular setting, the agent does not have access to the underlying states.\footnote{An alternative problem setting when function approximation is used is when the agent does have access to the underlying states, but they are too many to enumerate (e.g., in a table). In this case $|\calS|$ is either unknown, or too large (making $\frac{1}{|\calS|}$ too small).}
    
    \item Reference function is the max of all action-value estimates: $f(\hat{q}_t) \doteq \max_{s,a} \hat{q}_t(s,a)$ \\
    In the tabular setting, it is straightforward to compute the max of the action-value function over all state--action pairs. In the function approximation setting, computing this quantity would again require access to all the underlying states, which the agent does not have.
    
    \item Reference function is the action-value estimate of  a single reference state--action pair ($s_0,a_0$): $f(\hat{q}_t) \doteq \hat{q}_t(s_0,a_0)$ \\
    Again, the agent does not have access to any underlying state in the function approximation setting. Instead, one might consider using a value of a \textit{reference feature vector} with an action as the reference function. The question then becomes what the reference feature vector should be, among the infinite choices in $\mathbb{R}^d$. \\
    Based on our observations in the tabular setting, we hypothesize that the performance of RVI Q-learning with a reference feature vector would depend on the frequency with feature vectors similar to the reference feature vector occur under the optimal policy for the given problem.
\end{itemize}

Based on the above discussion, we can attempt to create a couple of reference functions for the function approximation setting, for instance, the action-value estimate corresponding to the \textit{first} feature vector the agent observes along the action of moving left. There is no way to compute the max exactly, but perhaps we can try using 
the maximum of the set of estimated action values corresponding to the feature vectors when they are observed. These are just some first attempts; further research is required to develop theoretically-grounded reference functions for the function approximation setting. 

If we have good ways of computing such reference functions at each timestep, the linear function approximation version of RVI Q-learning would look similar to Algorithm \ref{algo:diffQ_LFA}, except the TD-error term in line 5 would be: $\delta = R - f(\hat{q}) + \max_a \mbf{w}_a^T\mbf{x'} - \mbf{w}_A^T\mbf{x}$, and there would be no update to a reward-rate estimate like in line 7.


\subsection{Preliminary experimental results in the linear function approximation setting}

\begin{figure*}[b!]
    \centering
    \begin{subfigure}{.5\textwidth}
    \centering
    \includegraphics[width=0.7\textwidth]{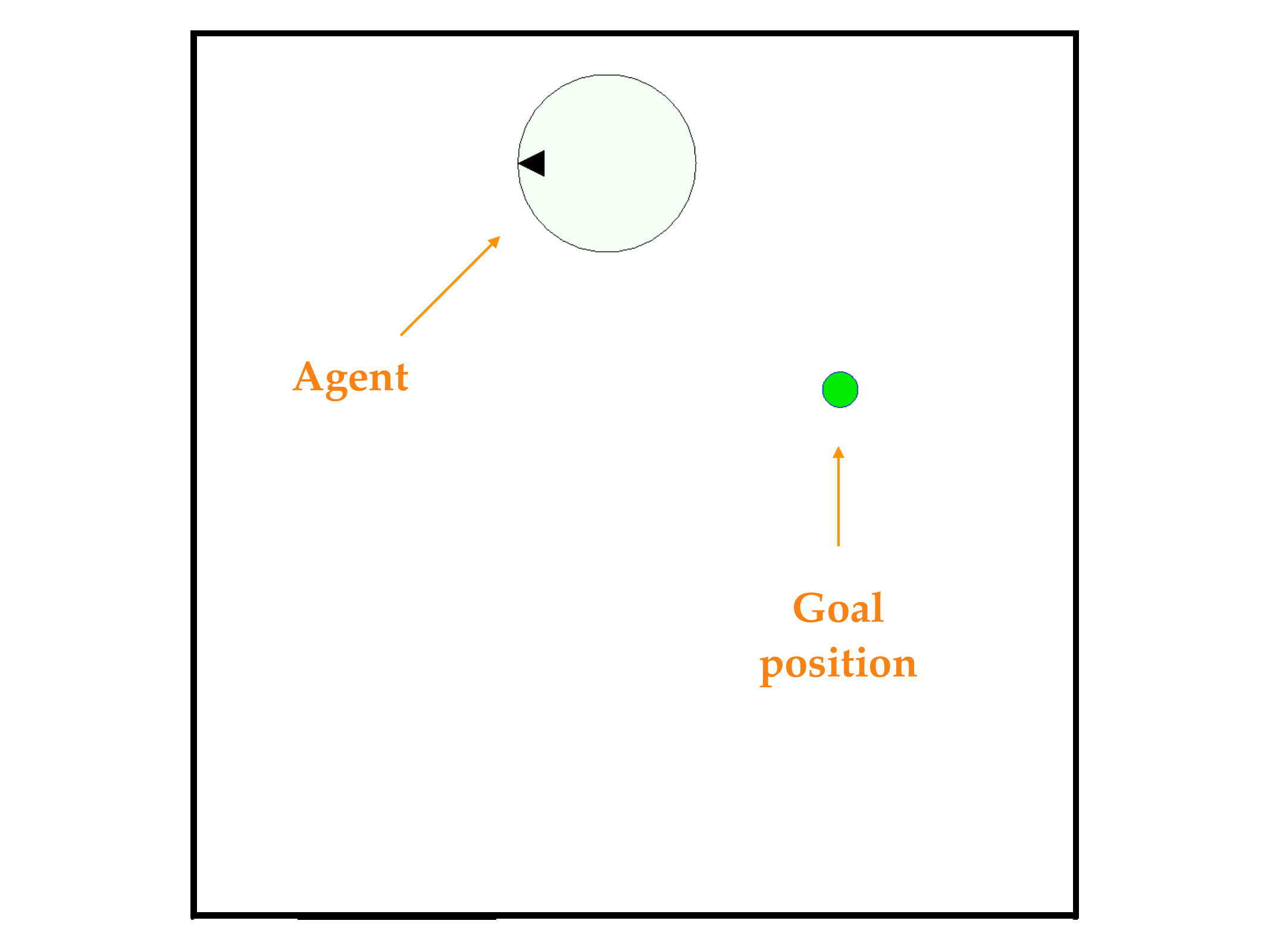}
    \end{subfigure}%
    \begin{subfigure}{.5\textwidth}
    \centering
    \includegraphics[width=0.87\textwidth]{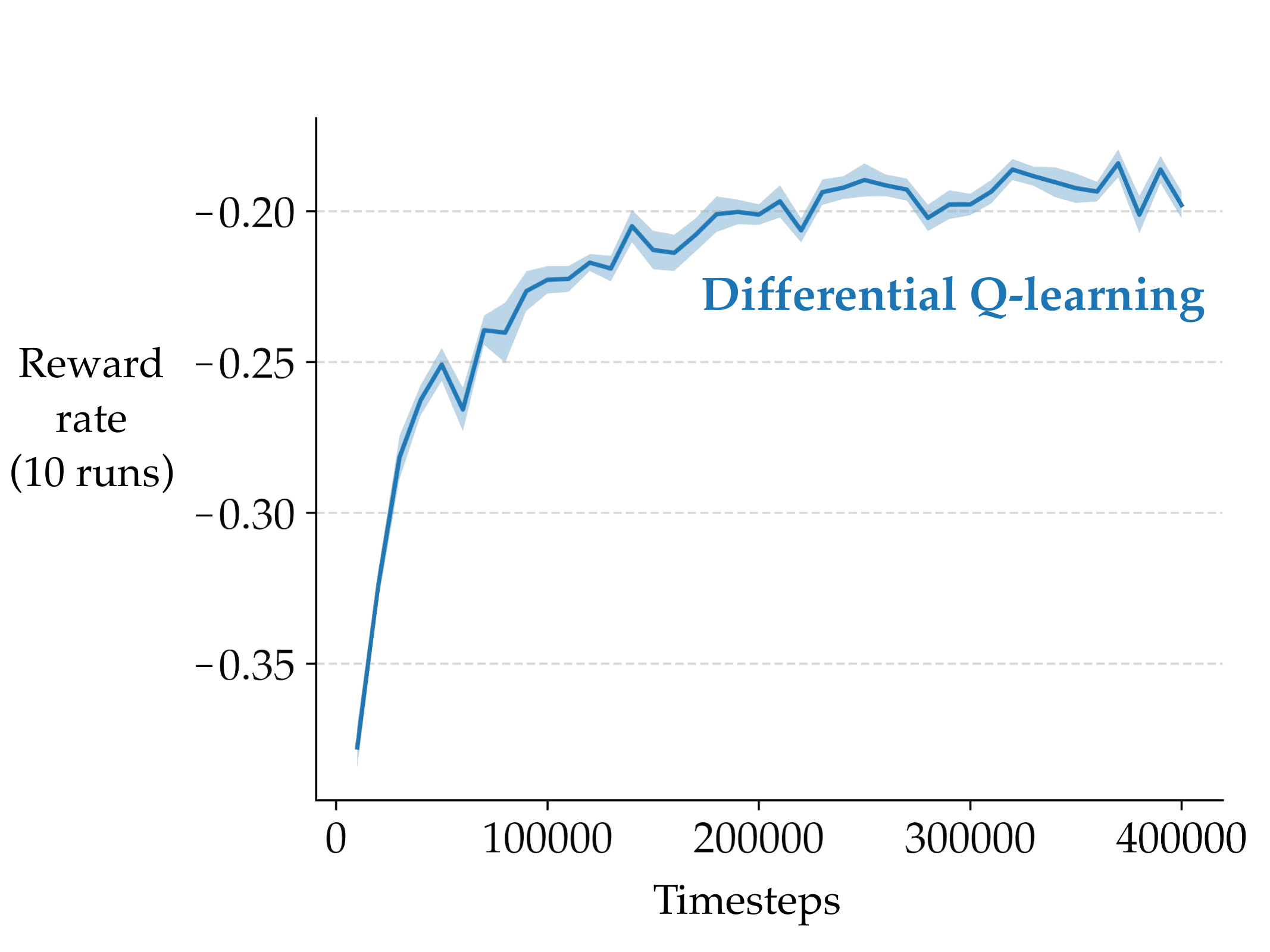}
    \end{subfigure}
    \begin{subfigure}{.5\textwidth}
    \centering
    \includegraphics[width=0.87\textwidth]{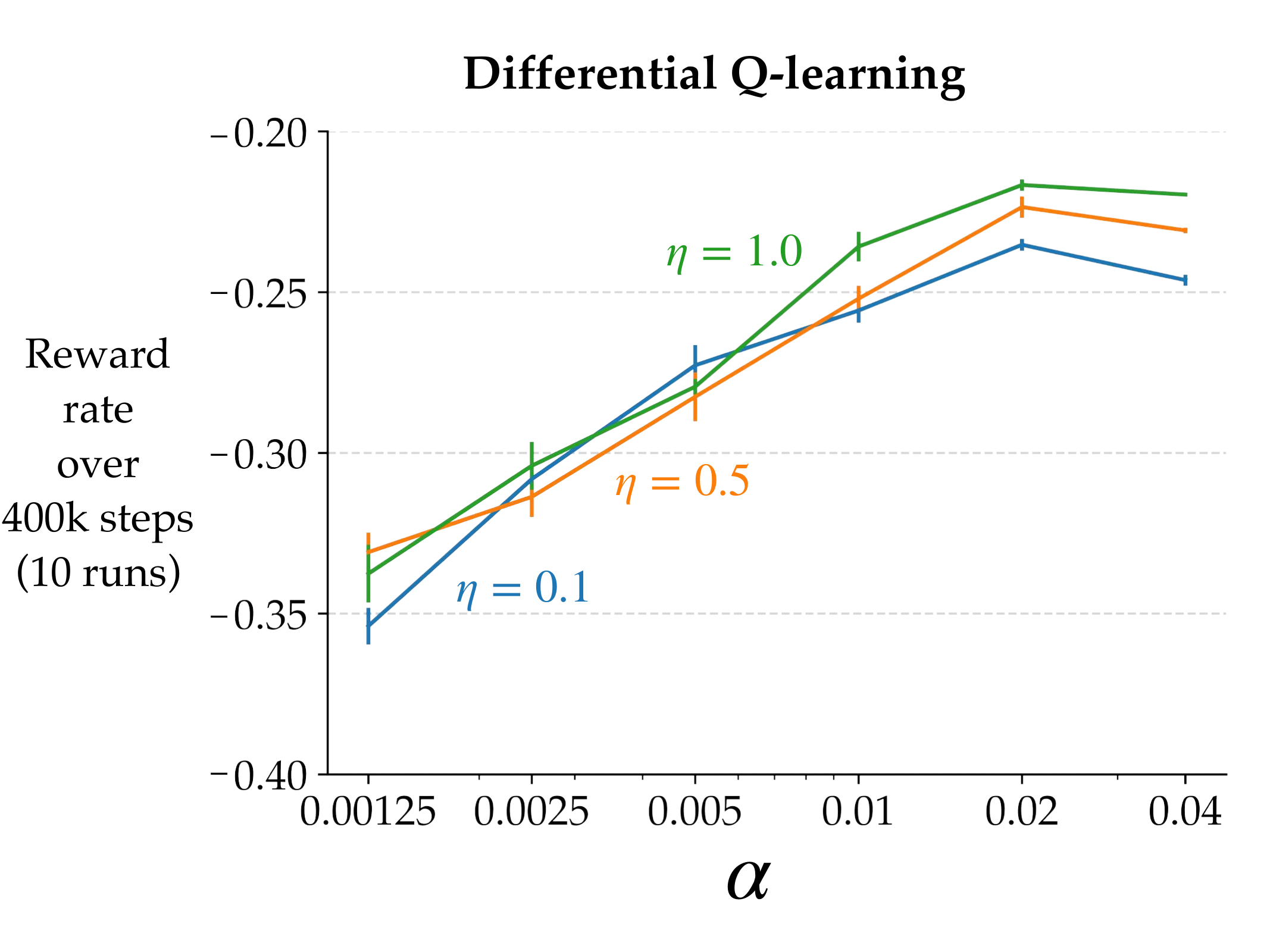}
    \end{subfigure}%
    \begin{subfigure}{.5\textwidth}
    \centering
    \includegraphics[width=0.87\textwidth]{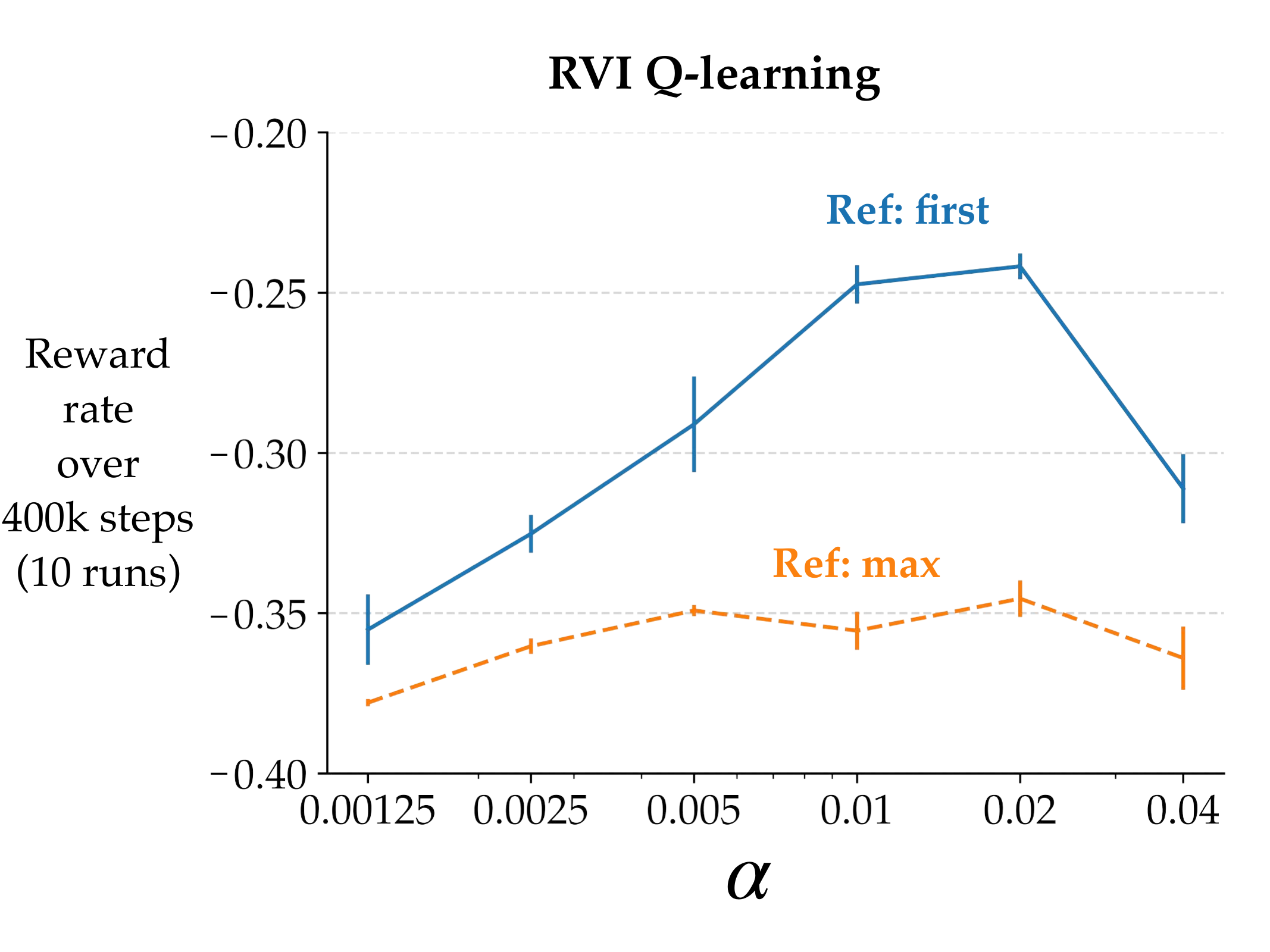}
    \end{subfigure}
    \caption{A learning curve and parameter studies for the linear function approximation versions of Differential Q-learning and RVI Q-learning on the PuckWorld problem. The shaded region and the error bars in the plots represent one standard error. \textit{Top-left:} A still of the PuckWorld domain showing the agent and the goal position. \textit{Top-right:} A typical learning curve started roughly at a reward rate of -0.4 and rose to about -0.19. \textit{Bottom-left}: Parameter studies showing the performance of Differential Q-learning in terms of average reward rate was not very sensitive to the choice of parameters. \textit{Bottom-right}: Parameter studies showing the performance of RVI Q-learning is relatively good when the reference function is the first observed feature vector, and relatively worse for the other reference function for a broad range of step sizes.}
    \label{fig:results-control-puckworld}
\end{figure*}

We performed a couple of preliminary experiments in the linear function approximation setting using the PuckWorld and Catcher domains from the PyGame Learning Environment\footnote{https://github.com/ntasfi/PyGame-Learning-Environment}. 

In the PuckWorld problem, the agent needs to reach the position marked by the green circle, which moves to a different location after every few timesteps. A still of the environment is shown in the top-left panel of Figure \ref{fig:results-control-puckworld}. At every timestep, the agent can take one of four actions — left, right, up, down — which move the agent in that direction by a small amount. Repeated actions in the same direction build some velocity in that direction, which decays at an exponential rate at every timestep. At every timestep, the agent gets a reward proportional to its distance to the goal position. This reward is typically negative and becomes zero when the agent reaches the goal position. At every timestep, the agent observes a six-dimensional feature vector of its horizontal position, vertical position, horizontal velocity, vertical velocity, target's horizontal position, target's vertical position. The positions and velocities are scaled to lie in $[0,1]$ and $-[1,1]$ respectively. After a regular interval of timesteps, the goal position is uniform-randomly initialized in the two-dimensional space. 

We applied the linear function approximation versions of Differential Q-learning and RVI Q-learning on this problem. RVI Q-learning used the two reference functions discussed earlier in this section: (1) the action-value estimate corresponding to the \textit{first} feature vector the agent observed when moving left, and (2) the maximum of the set of estimated action values corresponding to the feature vectors when they are observed. 
\footnote{the max was tracked online without storing all the previously observed feature vectors}. Both algorithms used tile coding (Sutton \& Barto 2018: Section 9.5.4) with 16 symmetric tilings of $2\times2\times2\times2\times2\times2$ tiles each. The weight vectors of both algorithms and the reward-rate estimate of Differential Q-learning was initialized to zero. The step-size parameter $\alpha$ was varied for both algorithms in the range $\{0.00125, 0.0025, 0.005, 0.01, 0.02, 0.04\}$. The parameter $\eta$ for Differential Q-learning was varied in $\{0.1, 0.5, 1.0\}$. Each instance of parameters was applied for 10 runs of 400,000 timesteps each. Both algorithms used an $\epsilon$-greedy policy with $\epsilon=0.1$ and no annealing. 

The top-right panel of Figure \ref{fig:results-control-puckworld} shows a typical learning curve on an instance of this problem where the goal positions is changed after every 100 timesteps. Using an $\epsilon$-greedy policy with $\epsilon=0.1$, the agent learns a policy that obtains a reward rate (computed over the last 10k steps) of about -0.19. The reward rate of a random policy is around -0.4. This learning curve corresponds to Differential Q-learning with $\alpha=0.02, \eta=1.0$. The learned policy was visualized and seen to be good everywhere except at the very edges of the two-dimensional space, which was probably an artifact of tile-coding.

We evaluated the performance of the agents across all the different parameter settings in terms of the average reward rate across the entire 400k timesteps of interaction. This is an indicator of the rate of learning. We observed that Differential Q-learning's rate of learning was quite robust to the parameter $\eta$. Its two parameters did not interact strongly; the best value of $\alpha$ was independent of the choice of $\eta$. Moreover, the best performance for different $\eta$ values was roughly the same. These observations were similar to those in the tabular case (see Section \ref{sec:control-exps} in the main text). 

RVI Q-learning also performed well on this problem for one choice of the reference function---the value estimate corresponding to the first feature vector the agent observes (with the `left' action). The performance corresponding to the other reference function tested---tracking the maximum value of the observed feature vectors online---was not as good. This might be because unlike the tabular setting, updating the weights corresponding to one feature vector also modifies the estimate for other feature vectors, making the max hard to track. The best rate of learning corresponding to the better-performing reference function was slightly lower than that with Differential Q-learning.

We now move on to the second experiment in the linear function approximation setting. In the Catcher problem, the agent needs to catch as many falling fruits as possible. A still of the environment is shown in the top-left panel of Figure \ref{fig:results-control-catcher}. `Fruits' fall vertically down from a uniformly-random horizontal position starting at the top of the frame. The agent can control the position of a `crate' at the bottom of the frame using two actions — left and right — which move the crate in that direction by a small amount. If the fruit falls on/in the crate, the agent gets a reward of +40; if the fruit falls anywhere outside at the bottom of the frame, the agent gets -40. The next fruit starts falling only after the previous fruit has reached the bottom of the frame. A fruit takes roughly 40 timesteps to reach the the bottom starting from the top. Hence, the maximal reward rate on this problem is 1. At every timestep, the agent observes a four-dimensional feature vector of the crate's horizontal position, the crate's horizontal velocity, the fruit's horizontal position, and the fruit's vertical position. The positions and velocity are scaled to lie roughly in $[0,1]$ and in $[-1,1]$ respectively. 

All the experimental details are the same as for PuckWorld, the only difference being that both algorithms used tile coding with 8 symmetric tilings of $4\times4\times4\times4$ tiles each. 

The top-right panel of Figure \ref{fig:results-control-catcher} shows a typical learning curve on this problem. Using an $\epsilon$-greedy policy with $\epsilon=0.1$, the agent learns a policy that obtains a reward rate of about 0.85, which is close to the optimal reward rate of 1. The reward rate of a random policy is around -0.3. The learning curve shown corresponds to Differential Q-learning with $\alpha=0.02, \eta=1.0$.

\begin{figure*}[t!]
    \centering
    \begin{subfigure}{.5\textwidth}
    \centering
    \includegraphics[width=0.75\textwidth]{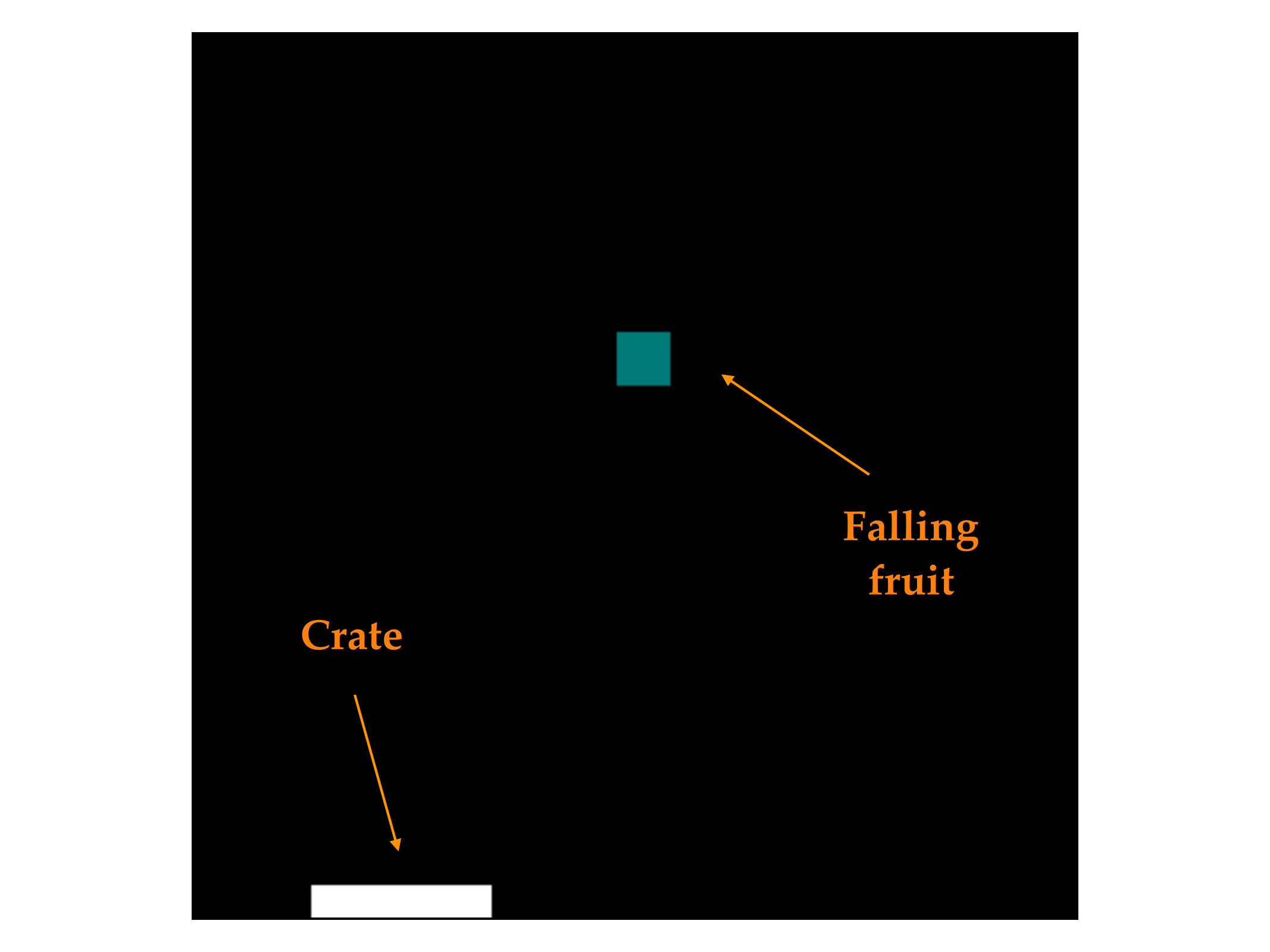}
    \end{subfigure}%
    \begin{subfigure}{.5\textwidth}
    \centering
    \includegraphics[width=0.95\textwidth]{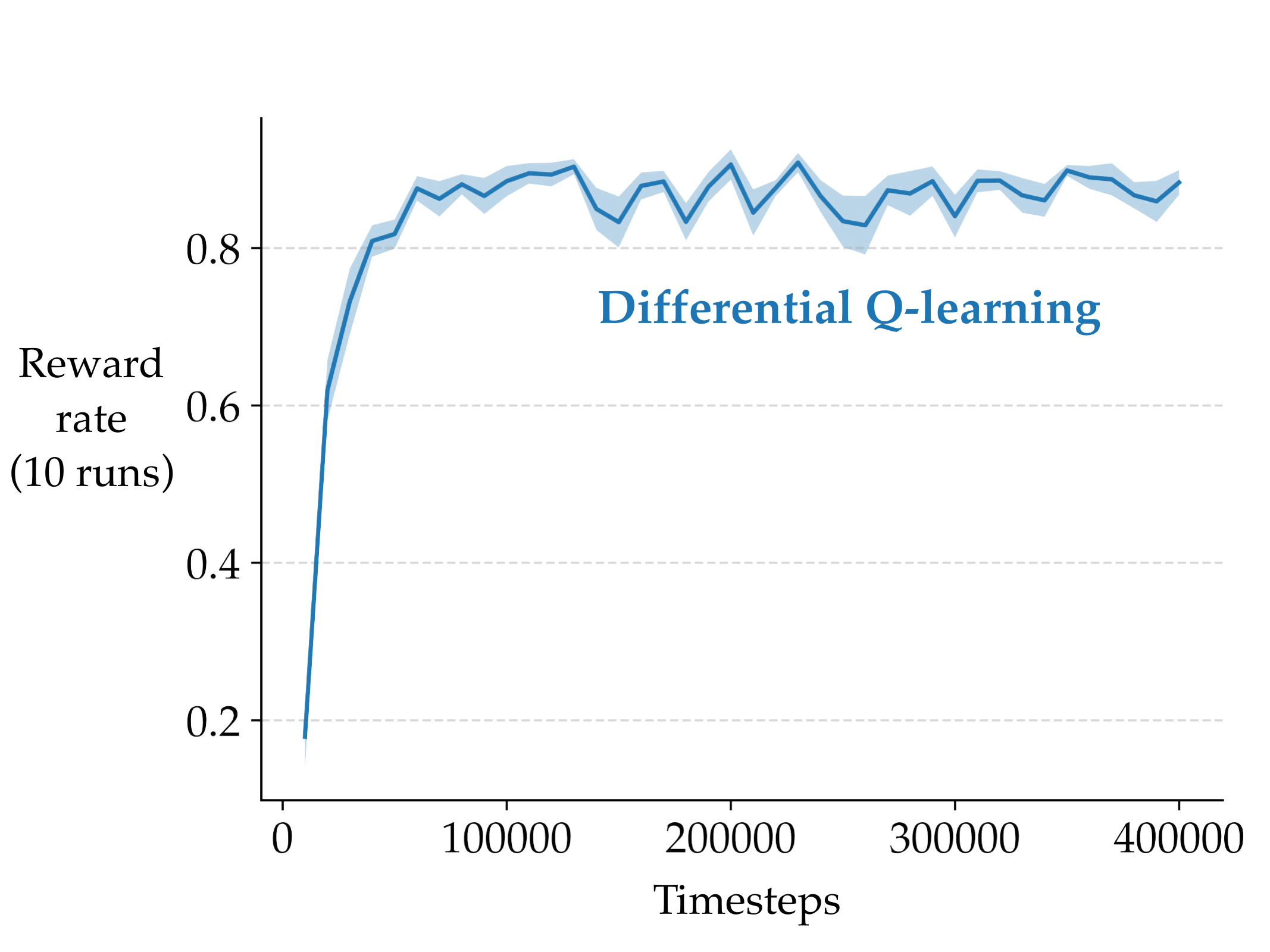}
    \end{subfigure}
    \begin{subfigure}{.5\textwidth}
    \centering
    \includegraphics[width=0.95\textwidth]{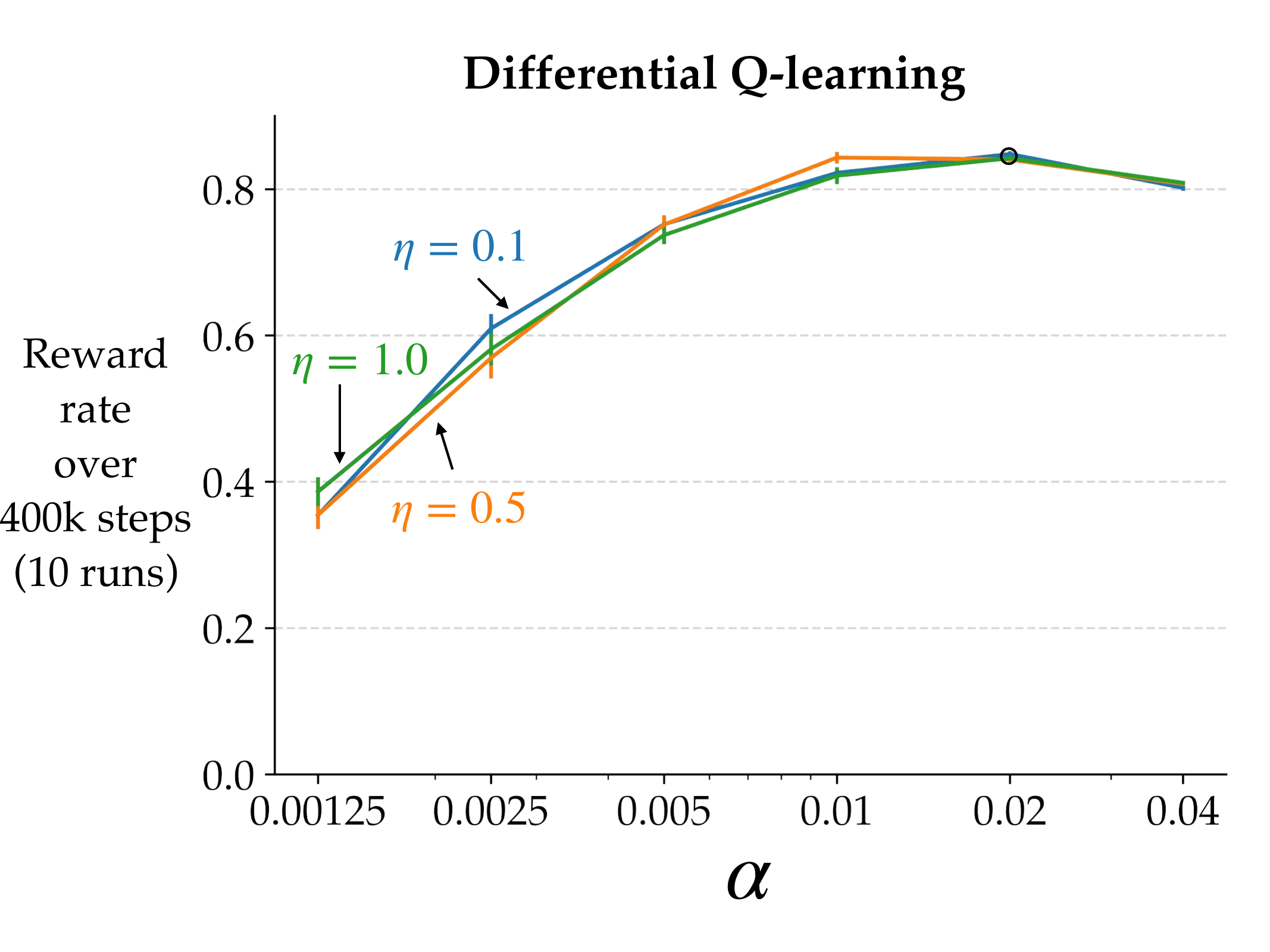}
    \end{subfigure}%
    \begin{subfigure}{.5\textwidth}
    \centering
    \includegraphics[width=0.95\textwidth]{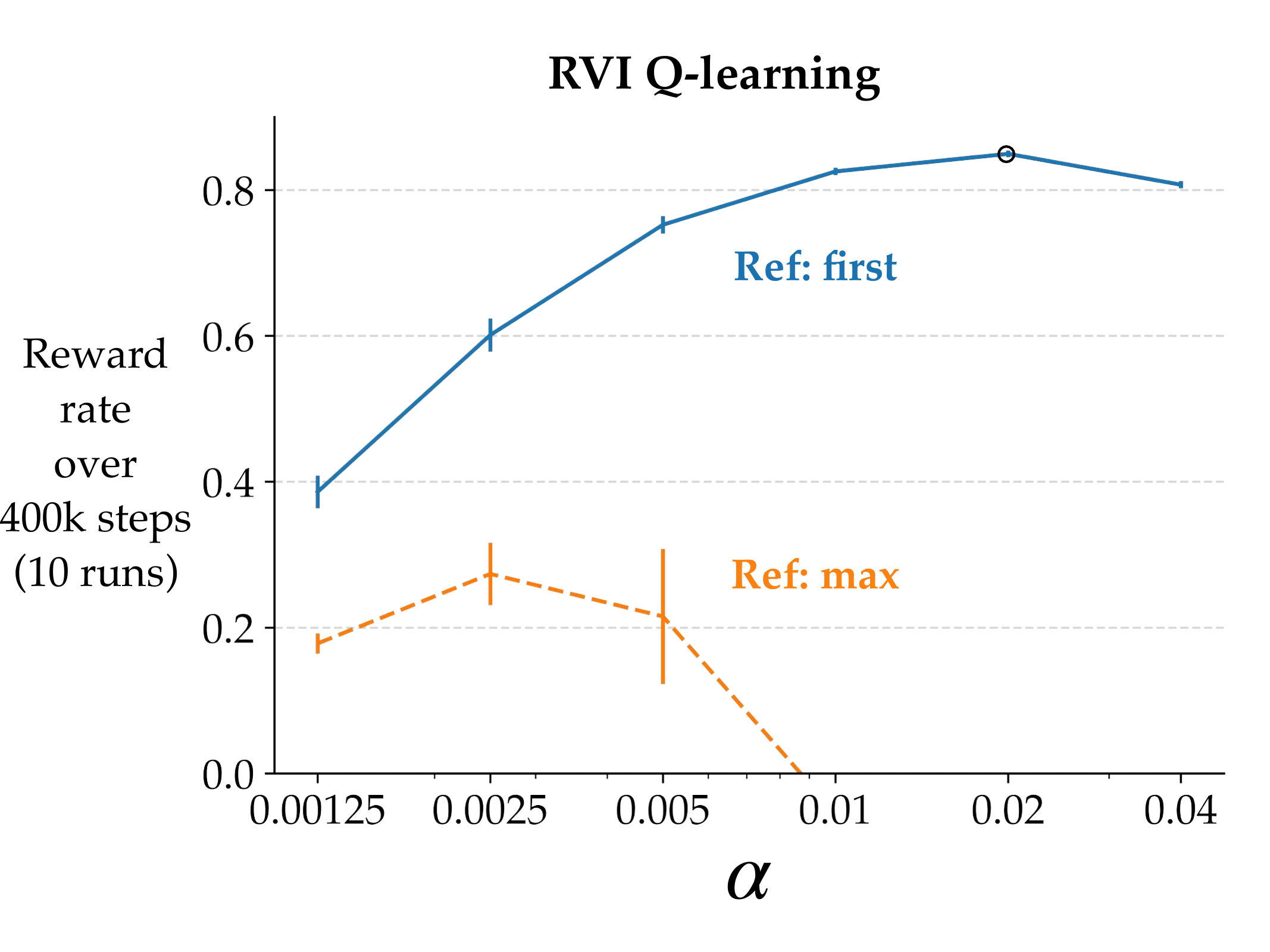}
    \end{subfigure}
    \caption{A learning curve and parameter studies for the linear function approximation versions of Differential Q-learning and RVI Q-learning on the Catcher problem. The shaded region and the error bars in the plots represent one standard error. \textit{Top-left:} A still of the Catcher domain showing a falling fruit and the crate that the agent controls along the horizontal dimension at the bottom. \textit{Top-right:} A typical learning curve started close to a reward rate of 0 and rose to about 0.9. \textit{Bottom-left}: Parameter studies showing the performance of Differential Q-learning in terms of average reward rate was not very sensitive to the choice of parameters. \textit{Bottom-right}: Parameter studies showing the performance of RVI Q-learning is relatively good when the reference function is the first observed feature vector, and relatively worse for the other reference function for a broad range of step sizes.
    }
    \label{fig:results-control-catcher}
\end{figure*}

Again, we evaluated the rate of the learning of the agents across different parameter settings. We again observed that Differential Q-learning's rate of learning did not vary much across a broad range of its parameter values. It was also especially robust to $\eta$. The linear function approximation version of RVI Q-learning also performed well for one choice of the reference function, not as much with the other. The learned policies corresponding to good parameter values for both algorithms successfully catch almost every fruit. 

For RVI Q-learning, using the estimate of the first observed feature vector as a reference value worked better in Catcher than in PuckWorld. This might be because the agent might be observing feature vectors similar to the first one quite frequently, given that the crate has to move across the whole one-dimensional horizontal plane under any optimal policy. On the other hand, the agent moves in a relatively larger two-dimensional space in PuckWorld. In a finite number of agent-environmental interactions, the agent might not visit its starting location that frequently. This suggests that the choice of the reference feature vector can affect the performance of RVI Q-learning differently in different problems. Additionally, in both cases, the other reference function did not result in good performance; this was probably because tracking the maximum action value in the function approximation setting is a poor approximation to the maximum action value across all state--action pairs.

The two experiments showed that the simple extension of the tabular Differential Q-learning to the linear function approximation setting can work rather well in terms of the final performance as well as robustness to different parameter values. The extension of the notion of reference functions to the linear function approximation setting is not as straightforward.


\end{document}